\documentclass[twoside]{article}

\usepackage[margin=1.2in]{geometry}
\usepackage[usenames,dvipsnames]{xcolor}
\usepackage{microtype}

\usepackage{graphicx}
\usepackage{subcaption}  
\usepackage{tikz}
\graphicspath{{figures/}}

\usepackage{booktabs}
\usepackage{nicematrix}
\usepackage{enumitem}

\usepackage{algorithm}
\usepackage{algorithmic}
\usepackage{stfloats}

\usepackage[authoryear,round]{natbib}

\usepackage{framed}
\usepackage[most]{tcolorbox}

\usepackage{amsmath}
\usepackage{amssymb}
\usepackage{amsthm}
\usepackage{mathtools}
\usepackage{nicefrac}       

\usepackage{wrapfig}
\usepackage{url}            
\usepackage[textsize=tiny]{todonotes}

\usepackage[hidelinks]{hyperref}

\usepackage[capitalize,noabbrev,nameinlink]{cleveref}

\usepackage{authblk}

\usepackage{tcolorbox}
\usepackage{bbding}    

\newcommand{\perr}{p_\text{err}}

\newcommand{\SCOUT}{\texttt{SCOUT} }

\newcommand{\norm}[1]{\left\lVert#1\right\rVert}
\newcommand{\dotp}[2]{\langle #1, #2 \rangle}
\newcommand{\abs}[1]{| #1 |}
\newcommand{\indctr}[1]{\mathds{1}\{#1\}}
\newcommand{\identity}[1]{\mathbf{I}_{#1}}

\usepackage{amsmath}
\usepackage{amssymb}
\usepackage{mathtools}
\usepackage{amsthm}

\usepackage{amsfonts,dsfont, amssymb,graphicx,xcolor,semantic}

\newcommand{\ignore}[1]{}

\definecolor{forestgreen}{rgb}{0.0, 0.27, 0.13}

\newcommand\numberthis{\addtocounter{equation}{1}\tag{\theequation}}
\newcommand{\opt}{^\star}

\newtheorem{defn}{Definition}

\newtheorem{lemma}{Lemma}
\newtheorem{assum}{Assumption}

\mathlig{==}{\equiv}
\mathlig{=.}{\doteq}
\mathlig{:=}{\triangleq}
\mathlig{<<}{\ll}
\mathlig{>>}{\gg}
\mathlig{<>}{\neq}
\mathlig{<=}{\leq}
\mathlig{>=}{\geq}
\mathlig{<==}{\Leftarrow}
\mathlig{==>}{\Rightarrow}
\mathlig{<=>}{\Leftrightarrow}
\mathlig{<==>}{\iff}
\mathlig{<--}{\leftarrow}
\mathlig{-->}{\rightarrow}
\mathlig{<->}{\leftrightarrow}
\mathlig{+-}{\pm}
\mathlig{-+}{\mp}
\mathlig{...}{\dots}
\mathlig{!=}{\stackrel{!}{=}}

\DeclareMathOperator*{\argmax}{\arg\!\max}
\DeclareMathOperator*{\argmin}{\arg\!\min}

\newif\ifshowanswer    

\newcommand{\isitthree}[1]
{
  \ifnum#1=3
    number #1 is 3
  \else
    number #1 is not 3
  \fi
}

\newcommand{\be}{\begin{equation}}
\newcommand{\ee}{\end{equation}}

\newcommand\R{{\mathbb{R}}}

\renewcommand\P{{\mathds{P}}}
\newcommand\E{{\mathds{E}}}

\newcommand\eps{{\varepsilon}}

\newcommand\Bc{{\mathbf c}}

\newcommand\Bu{{\mathbf u}}

\newcommand\Bx{{\mathbf x}}

\newcommand\BI{{\mathbf I}}

\newcommand\CB{{\mathcal B}}
\newcommand\CC{{\mathcal C}}

\newcommand\CL{{\mathcal L}}

\newcommand\CN{{\mathcal N}}
\newcommand\CO{{\mathcal O}}

\newcommand\CQ{{\mathcal Q}}
\newcommand\CS{{\mathcal S}}
\newcommand\CT{{\mathcal T}}

\newcommand\CX{{\mathcal X}}

\newcommand\N{{\mathbb N}}

\newtheorem{proposition}{Proposition}
\crefname{proposition}{Proposition}{Propositions}
\Crefname{proposition}{Proposition}{Propositions}

\crefname{assum}{Assumption}{Assumptions}
\Crefname{assum}{Assumption}{Assumptions}

\crefname{thm}{Theorem}{Theorems}
\Crefname{thm}{Theorem}{Theorems}

\crefname{defn}{Definition}{Definitions}
\Crefname{defn}{Definition}{Definitions}

\newtheorem*{lemma*}{Lemma}  

\newenvironment{nscenter}
 {\parskip=0pt\par\nopagebreak\centering}
 {\par\noindent\ignorespacesafterend}

\usepackage{thm-restate} 
\usepackage{adjustbox}

\newcounter{constant} 
\newcommand{\newconstant}[1]{\refstepcounter{constant}\label{#1}} 
\newcommand{\useconstant}[1]{C_{\ref{#1}}}

\usepackage{pgfplots}
\pgfplotsset{compat=1.18}


\usepackage{balance}

\usepackage[preprint]{aistats2026} 
\makeatletter
\newcommand{\blfootnote}[1]{%
  \begingroup
  \renewcommand\thefootnote{}%
  \footnotetext{#1}%
  \addtocounter{footnote}{-1}%
  \endgroup
}
\makeatother

\begin{document}

%

%

\twocolumn[

\aistatstitle{The Good, the Bad, and the Sampled: a No-Regret Approach to Safe Online Classification}

\aistatsauthor{
Tavor Z. Baharav${}^{1,2,*,\dagger}$ \And
Spyros Dragazis${}^{3,*}$ \And
Aldo Pacchiano${}^{1,4,\dagger}$
}

\aistatsaddress{
$^1$ Eric and Wendy Schmidt Center,\\
Broad Institute\\
$^2$ Department of Data Science,\\
Dana Farber Cancer Institute\\
\And
$^3$ Department of Computer Science,\\
Boston University\\
\And
$^1$ Eric and Wendy Schmidt Center,\\
Broad Institute\\
$^4$ Faculty of Computing \& Data Sciences,\\
Boston University
}

]

\blfootnote{\small \textsuperscript{*}Equal contribution \textsuperscript{$\dagger$}Equal senior contribution.}

\begin{abstract}
We study sequential testing for a binary disease outcome when risk follows an unknown logistic model. At each round, the decision maker may either pay for a test revealing the true label or predict the outcome based on patient features and past data. The goal is to minimize costly tests while ensuring the misclassification rate stays below $\alpha$ with probability at least $1-\delta$. We propose a method that jointly estimates the logistic parameter $\theta\opt$ and the feature distribution, using a conservative threshold on the logistic score to decide when to test. We prove our procedure achieves the target error with high probability and requires only $\widetilde O(\sqrt{T})$ more tests than an oracle with full knowledge. This is the first no-regret guarantee for error-constrained logistic testing, with direct applications to medical screening. Simulations corroborate our theoretical results, showing safe classification of patients and efficient estimation of $\theta \opt$ with few excess tests.
\end{abstract}

\section{INTRODUCTION}
Modern machine learning has recently provided solutions to real-world automated decision-making systems in various fields such as drug discovery \citet{vamathevan2019applications,dara2022machine}, recommendation systems \citet{afsar2022reinforcement,zhu2023scalable}, online ad-allocation \citet{slivkins2013dynamic}, and portfolio selection \citet{pinelis2022machine}.
Bandit algorithms \citet{lattimore2020bandit} and reinforcement learning \citet{sutton1999reinforcement} play a significant role in building interactive decision-making systems that collect feedback from users and improve their performance with each interaction. 
Two primary challenges exist in the aforementioned applications: the first is the learning challenge, estimating the problem parameters which are vital for decision-making; the second is the decision-making challenge, where effective performance is required concurrently with learning.

Although machine learning systems perform well in practice, safety is crucial in human-centric applications~\citet{gu2022review,giudici2024safe}. 
In sequential decision making, safety has been formalized in several ways, including (i) constraining actions to a safe set defined by a cost signal~\citet{pacchiano2021stochastic,yao2021power,gangrade2024safe,camilleri2022active}, (ii) conservative bandits, which require performance comparable to a baseline~\citet{kazerouni2017conservative}, and (iii) enforcing explicit performance threshold constraints on the learning process itself \citet{bar2025juggler}.
Satisfying such criteria generally complicates reward acquisition, making the central challenge one of optimally balancing safety and learning.
Motivated by the COVID-19 pandemic, we study an online learning problem with a novel safety constraint. Patients arrive sequentially, each presenting with a feature vector (e.g., fever, fatigue, blood oxygen) and an unobserved disease state. The hospital aims to minimize costly tests while ensuring sick patients are properly quarantined. We assume an unknown logistic model links features to disease status, allowing the hospital to learn, for example, that low oxygen and high fever imply high risk without testing. Thus, the task is to simultaneously learn (a) the patient distribution, (b) the logistic model parameters, and (c) the testing threshold. \looseness=-1

Related problems have been studied in the active learning and selective sampling literature~\citet{settles2009active,hanneke2021toward,orabona2011better,cesa2006worst,freund1997selective,seung1992query,dekel2012selective}, which study a similar observation model and generalization error (regret) metric but without a safety constraint. 
These study settings where context information may be abundant but the labels are hard to come by~\citet{duan2023towards}. 

Focusing on the classification task and changing the objective from minimizing the generalization error to minimizing the cumulative pseudo regret (with respect to the optimal labeling policy), various algorithms have been developed in the online selective sampling literature, such as \citet{orabona2011better,sekhari2023selective}, considering both stochastic and adversarial contexts.
The objective in these works is to achieve sublinear regret while minimizing the expected number of queries made.
Similarly, online selective classification ~\citet{gangrade2021online,gangrade2021selective,goel2023adversarial} allows the learner to abstain from predicting.
The objective is to simultaneously minimize the expected number of abstentions and mistakes.

However, in real-world scenarios \citep{bastani2022interpretable}, it makes sense to require that the prediction error rate remain under a safety threshold with high probability while minimizing the number of queries.
For example in the streaming patient scenario we described above, where patients arrive one by one and the medical provider needs to classify them as sick or not.
In this problem, due to the sensitive nature of making misclassification mistakes, the selective testing procedure must guarantee that the misclassification rate remains below a safety threshold $\alpha \in [0, 1]$.
Testing every patient clearly attains this safety threshold, but can be prohibitively expensive.
Our question is thus:

\vspace{.1cm}
\begin{nscenter}
\textit{Can we design an adaptive algorithm that minimizes the expected number of tests while maintaining a misclassification rate below a given safety threshold?}
\vspace{.2cm}
\end{nscenter}
We define a baseline testing policy, that is optimal when the $\alpha$ error rate is only required to hold in expectation, which tests $p\opt := p\opt(\alpha)$ fraction of the time.
We develop an adaptive algorithm to ensure this $\alpha$ error rate with probability at least $1-\delta$, which requires only a sublinear number of excess tests: 
$\CO\left(\sqrt{\frac{dT}{p\opt(\alpha)\lambda_0}\log(T/\delta)}\right)$, where $\lambda_0$ is the minimum eigenvalue of the covariance matrix of the contexts observed under the baseline policy. 
In \Cref{lem:min_eigenvalue}, we derive a lower bound on $\lambda_0$ in terms of the problem parameters (additional discussion in \Cref{app:min_eigenval_discussion}).
In the special case where the context distribution is uniform over the $d$-dimensional sphere, $\lambda_0 = \Theta(1/d)$, recovering the standard $\tilde{\CO}(d\sqrt{T})$ regret rate for linear bandits \citep{lattimore2020bandit}.
We corroborate our theoretical results through comprehensive synthetic experiments. \looseness=-1

\section{RELATED WORK}\label{sec:related_work} \vspace{-.2cm}
The most related line of work to ours is \emph{online learning with abstention} \citet{cortes2018online,cortes2016learning}, where the learner may abstain from labeling a context $x \in \CX$ by incurring an abstention cost $c$.
In particular, \citet{cortes2018online} studies both the stochastic and adversarial settings, with the goal of minimizing regret; the resulting guarantees depend, in addition to the standard problem parameters such as the horizon $T$, on the abstention cost $c$.
The key difference from our work is that, while we also aim to reduce the cost induced by abstention, we simultaneously control the misclassification rate so that it remains below a prescribed target level throughout the entire process.

Another relevant line of work is \emph{online selective classification} \citet{gangrade2021selective,gangrade2021online,sekhari2023selective}, where the learner may abstain from releasing a prediction and instead observe the true outcome.
The goal in this setting is to jointly minimize the number of mistakes and the number of label queries.
Although this framework is the closest to ours, existing guarantees do not enforce uniform control of the cumulative misclassification rate throughout the learning process.
This distinction is especially important in applications such as medicine, where maintaining the error rate below a prescribed threshold at all times is often essential.

Moving away from regret minimization, another related direction in the bandit literature is best-arm identification under selective sampling \citet{camilleri2021selective,camilleri2022active}.
In \citet{camilleri2021selective}, the authors precisely characterize the trade-off between the number of observed labels and the stopping time, and provide a geometric interpretation of the optimal decision rule.
In \citet{camilleri2022active}, they study the problem of actively selecting the covariates or contexts whose labels are to be observed, while ensuring that the selected covariates lie within a safe set.

Finally, the recent literature on \textit{PAC}-labeling \citet{candes2025probably} addresses a closely related problem from a different perspective.
In that setting, one assumes access to an AI model that predicts the labels of an unlabeled dataset.
For each prediction $Y_i$, the ``expert'' model also provides an uncertainty score $U_i$.
The algorithmic challenge is then to leverage these uncertainty scores to produce ``PAC labels'', or, in our terminology, to satisfy $(\alpha,\delta)$-safety.

Our work is also related to the literature on learning halfspaces under label noise. In that line of work, the objective is to learn a linear separator with small classification error from noisy labeled data, either in the batch or online setting. Recent results study efficient learning under Tsybakov noise \citet{tsybakov2004optimal,diakonikolas2021efficiently} and online learning under Massart noise \citet{diakonikolas2024online}. These works focus on statistical and computational guarantees for recovering an accurate classifier; by contrast, our goal is different. We study a sequential decision problem in which the learner must decide, for each context, whether to query the label or to predict, while maintaining a uniform high-probability bound on the cumulative misclassification rate and simultaneously minimizing the number of queries. \looseness=-1

A common assumption in the halfspace literature is the Tsybakov noise condition. Writing $\eta(x) = \P(Y=1 \mid X=x)$, this condition controls the probability mass of examples whose conditional label probability lies close to the ambiguous value $1/2$. Intuitively, it quantifies how much probability mass concentrates near the Bayes decision boundary, and therefore how hard it is to identify an accurate classifier. Under this perspective, the difficulty of learning is governed by the Tsybakov parameters. In our setting, however, points near the decision boundary are precisely the ones that must be tested in order to guarantee safety. As a result, the relevant notion of difficulty is not how sharply the distribution separates around the boundary, but rather the baseline testing rate $p^\star$, namely the mass of contexts on which even the oracle safe policy must abstain from autonomous prediction.

Another common assumption in active learning for halfspaces is that the unlabeled contexts follow a particularly structured distribution, such as the uniform distribution on the unit sphere \citet{dasgupta2005analysis}. In contrast, our analysis only requires the context distribution to satisfy the regularity condition of \Cref{assum:density}. This allows a broader family of distributions, including smooth radial densities of the form $f(x)=g(\|x\|)$ and truncated Gaussian distributions. These assumptions are used to guarantee stability of the threshold policy and positivity of the covariance matrix of the queried contexts, rather than to enable efficient recovery of a separator from label noise alone.

\section{PRELIMINARIES}
\textbf{Notation}
We adopt the following notation throughout the paper.
The inner product between two vectors $x,y \in \R^d$ will be denoted either as $x^\top y$ or as $\dotp{x}{y}$.
We denote the $\ell_2$ norm of a vector $x \in \R^d$ as $\norm{x}_2 = \sqrt{\dotp{x}{x}}$ and $\norm{x}_A = \sqrt{x^\top A x}$ for any positive semi-definite matrix $A$.
The minimum eigenvalue of a matrix $A$ will be denoted as $\lambda_{\min}(A)$.
The set $\{1,2,\dots,n\}$ is denoted as $[n]$.
The logistic function is denoted as $\mu(z) = \frac{1}{1 + \exp(-z)}$ and $\mathds{1}(E)$ denotes the indicator function of an event $E$.
For two functions $f,g$ we say that $f(x) \preccurlyeq g(x)$ when there exists an absolute constant $c>0$ such that $f(x) \le c g(x)$ for all $x>0$. 
We use upper case letters for random variables and lower case for scalars.
For any measurable set $A$ we denote the set of all distributions on $A$ as $\Delta(A)$. 
An $\CL_2$ ball centered at $\Bc \in \R^d$ with radius $r>0$ is symbolized as $\CB(\Bc,r)$. \looseness=-1

\subsection{Problem Definition}
We consider the following repeated interaction between a learner and the environment.
At every round $t \in [T]$, the environment generates a context $X_t \in \R^d$ in the unit ball.
These contexts are identically distributed, and are drawn independently from an unknown distribution with density $P$. 
Every patient-context has an unseen random label $Y_t \in \{0,1\}$ that represents their disease status.
We assume that $Y_t \sim \text{Ber}(\mu(X_t^\top\theta\opt))$, independent from all other $X_{t'}$ and $Y_{t'}$.
Here, $\theta\opt\in \R^d$ is some fixed parameter vector unknown to the learner, with $\norm{\theta\opt}_2 = 1$.

At each round, the learner observes the patient's context $X_t$ and must decide whether or not to test the patient, denoted by $Z_t \in \{0,1\}$.
Then, the learner must predict whether the patient is healthy or sick, denoted by $\hat{Y}_t \in \{0,1\}$.
If $Z_t = 1$, the patient is tested, and the learner observes the true label $Y_t$, and so can predict $\hat{Y}_t = Y_t$.
The random variable $Z_t$ can depend on information obtained prior to that decision, i.e., $\mathcal{H}_t  = \{X_1, Z_1, Z_1Y_1, X_2, Z_2, Z_2Y_2, \dots, X_t \}$ and possibly on internal randomization of the learner.
Similarly, $\hat{Y}_t$ must be $\mathcal{F}_t = \sigma \{X_1, Z_1Y_1, X_2, Z_2Y_2, \cdots, X_t, Z_tY_t\}$ measurable.
The goal of the learner is to minimize the expected number of tests applied, while guaranteeing that the misclassification rate is less than a desired threshold $\alpha$, with probability at least $1-\delta$.
We define this constraint as $(\alpha,\delta)$-safety, where our objective is to minimize the expected number of tests required while retaining this $(\alpha,\delta)$-safety.

\begin{defn}
    An algorithm outputting $\{\hat{Y}_t\}$ satisfies $(\alpha,\delta)$-safety if
    \begin{equation*}
        \P\left(\bigcap_{\bar{T}=1}^T\left\{\frac{1}{\bar{T}}\sum_{t=1}^{\bar{T}} \mathds{1}\{\hat{Y}_t \neq Y_t\} \leq \alpha\right\}\right) \geq 1-\delta.
    \end{equation*}
    where the probability is computed with respect to the randomness in $\{X_t\},\{Y_t\}$, and any randomness internal to the algorithm. \looseness=-1
\end{defn}

\subsection{Baseline policy}\label{subsec:baseline}
First, we characterize the baseline testing strategy satisfying $(\alpha,\delta)$-safety in the case where the feature distribution $P$ and optimal discriminator $\theta\opt$ are known a priori to the learner.
Although many decision rules $Z_t$ are possible, we focus on threshold rules of the form below (\Cref{fig:eqn_and_threshold}).  

\begin{align*}
    Z_t &= \mathds{1}\{| \langle X_t, \theta\opt\rangle| \le \tau \}  \\
    \hat{Y}_t &= \begin{cases}
        0 & \text{if } \langle X_t, \theta\opt\rangle < -\tau, \\
        Y_t & \text{if } |\langle X_t, \theta\opt\rangle|\le \tau, \\
        1 & \text{if } \langle X_t, \theta\opt\rangle > \tau.
    \end{cases}  \numberthis \label{fig:eqn_and_threshold}
\end{align*}

\begin{figure}[htbp]
    \centering
    \begin{tikzpicture}[scale=.75]
        \draw[thick, <->] (-5.5,0) -- (5.5,0);
        
        \foreach \x/\label in {-5/{-1}, -2/{-\tau^\star}, 0/{0}, 2/{\tau^\star},5/{+1}} {
            \draw[thick] (\x,-0.1) -- (\x,0.1);
            \node[below, font=\small] at (\x,-0.2) {$\label$};
        }
        
        \draw[thick, forestgreen, <->] 
            (-5,0.7) -- (5,0.7) node[midway, above=5pt, blue] {$\langle X_t, \theta^\star \rangle$};
        
        \draw[thick, blue, <->] (-5,-1) -- (-2,-1) node[midway, below] {Predict 0};
        \draw[thick, red, <->] (-2,-1) -- (2,-1) node[midway, below] {Test};
        \draw[thick, blue, <->] (2,-1) -- (5,-1) node[midway, below] {Predict 1};
        
    \end{tikzpicture}
    \caption{Threshold-based testing policy.}
    \label{fig:threshold}
\end{figure}

When $P$ and $\theta\opt$ are known, a threshold decision rule is optimal when the safety constraint is imposed only in expectation, as we show in the following proposition.

\begin{proposition}\label{prop:optimal_policy_on_expectation}
    Consider a variant of safe learning where the constraint is only required to hold in expectation, at the final time step: 
\begin{equation*}
    \min_{\{Z_t,\hat{Y}_t\}} \E \left[ \sum_{t=1}^T Z_t \right] \text{ s.t.} \quad \E\left[\frac{1}{T}\sum_{t=1}^T \mathds{1}\{\hat{Y}_t \neq Y_t\} \right] \leq \alpha.
\end{equation*}
Then, an optimizing rule for $\hat{Y}_t, Z_t$ is the threshold policy \Cref{fig:eqn_and_threshold}.
\end{proposition}

The proof of \Cref{prop:optimal_policy_on_expectation} follows by relating this to the fractional knapsack problem, which we detail in \Cref{app:baseline_policy}.
We provide additional discussion on how this does not naively yield $(\alpha,\delta)$-safety, but still motivates the use of a threshold policy as a baseline.
As a consequence, we consider competing against the optimal threshold decision rule $\tau\opt$ that is a function of $P$, $\theta\opt$, and $\alpha$, henceforth referred to as the baseline policy. \looseness=-1

To identify the optimal threshold, we define the function $\perr(\theta, P, \tau)$ as the probability of misclassification incurred by the threshold $\tau$, if $\theta$ was the underlying logistic parameter, and where the expectation is taken with respect to $P$:
\begin{align*} 
    &\perr(\theta,P,\tau) \numberthis \label{eq:p_err_defn} \\
    &= \int (1+\exp(|x^\top \theta|))^{-1} \mathds{1}\left\{|x^\top \theta| > \tau\right\} P(dx). 
    \vspace{-1cm}
\end{align*}
The term inside the integral $(1+\exp(|x^\top \theta|))^{-1}$ is the optimal misclassification error for a fixed $x,\theta$ pair.
The term $\mathds{1}\left\{|x^\top \theta| > \tau\right\}$ equals one only if we predict the label $\hat{y}$ without observing the  real label $y$ for context $x$, when using a threshold rule.
Having defined the error probability for a given threshold $\tau$, we can now easily define the optimal threshold.
For any problem parameters $\theta \in \R^d, \alpha' \in[0,1]$, and distribution $\rho \in \Delta(\CX)$, we define the optimal decision threshold $\tau\opt$ as the minimum value of $\tau\in [0,1]$ that satisfies the $\alpha$-fraction misclassification constraint:
\begin{equation} \label{eq:tau_defn}
        \tau\opt(\theta,\rho,\alpha') :=  \min \{\tau : \perr(\theta,\rho,\tau) \le \alpha' \}.
\end{equation}

When considering the in-expectation objective from \Cref{prop:optimal_policy_on_expectation} we conclude that any algorithm requires an expected number of tests $p\opt T$:
\begin{align}\label{eq:tauStarDef}
        \tau\opt :=  \tau\opt(\theta\opt,P,\alpha), \ p\opt := \P\left(x : |x^\top \theta\opt| \le \tau\opt\right), \hspace{-.1cm}
\end{align}
where $\theta\opt, P,$ and $\alpha$ are the true parameters. Here, we have overloaded notation for $\tau\opt$ as both a function, and the evaluation of this function at the true problem parameters.
Note that in practice, $\perr$ must be estimated using $\hat{P}$, our observed samples from $P$, in addition to $\theta\opt$ being unknown.

Before introducing our regret objective, we examine the relationship between the safety parameter $\alpha$, which serves as an input, and the baseline policy testing probability $p\opt$.
When the misclassification rate threshold $\alpha$ approaches zero, the system must minimize error rates, necessitating testing of all cases.
This constraint leads to increased values of $\tau\opt$ and, consequently, higher values of $p\opt$.
Conversely, in the degenerate scenarios where $\alpha$ grows large, policies become indifferent to misclassification errors and conduct vanishing testing, yielding values of $p\opt$ that approach zero.

This lets us define the ``safe regret'' of an algorithm as the number of excess tests it takes over this oracle baseline, while satisfying $(\alpha,\delta)$-safety.
Formally, the regret is $\texttt{Regret}(T) := \E[\;\sum_{t=1}^T (Z_t -p\opt)\;] $.
An algorithm could trivially sample at each time step and satisfy the misclassification criterion; the question is, for a given misclassification rate $\alpha$ and error probability $\delta$, can a learner achieve sublinear safe regret in $T$.
To analyze this quantity, we make the following natural assumptions.
\begin{assum}
\label{assum:nontrivial}
    The optimal baseline tests a nonzero fraction of the time, i.e. $p\opt >0$.
\end{assum}

Other works such as, \citet{orabona2011better}, \citet{sekhari2023selective}, use the notation $T_\varepsilon$ to describe the number of times the Bayes optimal classifier outputs a label with confidence less than a fixed parameter $\varepsilon>0$. Our $p\opt$ is analogous to $T_\varepsilon$:
it serves as a measure to quantify the inherent difficulty of the problem instance (how many patients are close to the decision boundary).
We additionally assume that the density $P$ is smooth, which is reasonable for patient data with continuous-valued features.

\begin{assum} \label{assum:density}
    The density $P$ is upper and lower bounded by constants: $0<m\le P(x)\le M< \infty$, for all $x$ such that $\norm{x}_2 \le 1$.
\end{assum}

This is necessary for ensuring the stability of our estimates of $\tau\opt$ with respect to small perturbations in $\theta$, $\hat{P}$, and $\alpha$.
Using \Cref{assum:density} we derive the following result regarding the minimum eigenvalue of the covariance matrix of the baseline policy. This lemma ensures that $\theta\opt$ can be well estimated from the observed data.
We refer the reader to \Cref{sec:related_work} for a detailed discussion of analogous assumptions and problem formulations in the literature.

\begin{restatable}{lemma}{lemMinEigenvalue}\label{lem:min_eigenvalue}
    There is a constant $\lambda_0^{\min} (\tau\opt,d)> 0$: 
    \begin{align*}
        \lambda_{\min}&\left(\E_P \left[X X^\top\  \middle \vert \  |\dotp{X}{\theta\opt}|\leq \tau\opt \right]\right)\\
        &= \lambda_0 \geq \lambda_0^{\min} (\tau\opt,d)>0.
    \end{align*}
    \vspace{-.7cm}
\end{restatable}

As $\norm{\theta\opt}_2 = 1$, a ball of radius $\tau\opt$ is a subset of the contexts tested by the baseline policy.
The contexts drawn from this ball form a positive definite covariance matrix, which implies that the minimum eigenvalue of the overall covariance matrix is positive.
We defer the proof to \Cref{app:min_eigenval_discussion}.

Importantly, these assumptions are strictly for the \textit{analysis} of our algorithm.
Our algorithm does not require knowledge of any of these parameters $m,M,\lambda_0$, or $p\opt$ as input.
We are able to learn and adapt to them on the fly, they simply requiring them to be strictly positive and finite.

\subsection{Logistic Bandits tools}

Our algorithm leverages existing confidence intervals for $\theta\opt$ 
\citet{faury2020improved}.
We utilize their ellipsoidal confidence set to simplify our analysis, noting that tighter confidence intervals exist \citet{lee2024unified}.
In our setting, the non-linearity of the logistic function over the decision set $(\mathcal{X},\Theta)$ is bounded as $\kappa \le 6$.
Borrowing notation \citet{faury2020improved}, we denote the set of labeled samples $\big((X_t,Y_t)$ pairs$\big)$ collected up to the beginning of round $t$ which are used to estimate $\theta\opt$ by $\CS_\theta^t$, and the nonoverlapping set of samples (only the context, $X_t$) used to estimate the distribution $P$ by $\CS_P^t$.
We denote the cardinalities of these two sets by $N_\theta^t$ and $N_P^t$ respectively.
We define the regularized log-likelihood objective as: 
\begin{equation*}
    \mathcal{L}_t(\theta)= \sum_{s \in \CS_{\theta}^t} \ell(x_s, y_s, \theta)  - \frac{1}{2} \|\theta\|_2^2,
    \vspace{-.3cm}
\end{equation*}
where $\ell(x,y, \theta) = y\log \mu(x^\top \theta) + (1-y)\log(1-\mu(x^\top \theta))$,
and its maximum (regularized) likelihood estimator as $ \hat{\theta}_{t} = \argmax_{\theta \in \R^d}\mathcal{L}_{t}(\theta)$.
We also denote the design matrix as $V_t=\sum_{s \in \CS_{\theta}^t}X_sX_s^\top +\kappa \identity{d}$, and for technical reasons we consider a projection $\theta_t^L$ of $\hat{\theta}_t$ onto the feasible set $\Theta$ defined as follows,

\begin{align}
        \theta_{t}^L &:= \argmin_{\theta \in \Theta} \norm{g_{t}(\theta) - g_{t}(\hat{\theta}_{t})}_{V_{t}^{-1}}         \text{, where }  \notag \\
 g_{t}(\theta) &= \sum_{s \in \CS_{\theta}^t}\mu(\dotp{x_s}{\theta})x_s + \theta \label{eq:projection}.
\end{align}
These allow us to define the confidence ellipsoid $\CC_t$ for $\theta\opt$, which is implicitly a function of a confidence parameter $\delta'$, and its radius $B_{t}(\delta')$:
\begin{align*}
    &\mathcal{C}_{t} :=  \Big\{ \theta \in \Theta, \norm{\theta - \theta_{t}^L}_{V_{t}} \leq B_{t}(\delta') \Big\}, \numberthis \label{eq:conf_ellipse_joint} \\
    &B_{t}(\delta') := 2\kappa\left(1 + \sqrt{\log\left(\frac{1}{\delta'}\right) + 2d\log\left(1+\frac{N_\theta^t}{\kappa d}\right)}\right). \notag
\end{align*}
We omit the dependence of quantities like $B_t$ on the confidence level $\delta'$ when clear from context.
In the end we will designate $\delta' = \delta/7$ to obtain the desired result via a union bound.
These confidence intervals \citet{faury2020improved} satisfy the following anytime guarantees:
\begin{lemma}\label{lem:faury_concentration}[Lemma 12 of \citet{faury2020improved}.]
For any fixed $\delta'$, let $G_\theta$ be the event that the confidence intervals defined in \Cref{eq:conf_ellipse_joint} hold: \looseness=-1
    \begin{equation*}
        \P(G_\theta)= \P\left(t \in \N,\; \theta\opt \in \CC_t \right) \geq 1 -\delta'.
    \end{equation*}
    \vspace{-.7cm}
\end{lemma}
Since the number of samples $N_{\theta}^t$ collected to estimate $\CC_t$ is a random variable in our setting, we condition on its value in \Cref{lem:faury_concentration}. 

Before diving into our algorithm and its analysis, we discuss the role and behavior of key quantities that will arise.
To begin, the number of samples collected $N_\theta^t$ used to build our confidence intervals grows linearly in $t$ satisfying $N_\theta^t \succcurlyeq p\opt t$.
As a consequence, the bound $B_t$ used in $\CC_t$ (which satisfies $B_t \le B_T$) grows extremely slowly in $t$, with $B_t \preccurlyeq \sqrt{d\log(1+\frac{p\opt t}{d})}$.
The other portion of the confidence interval involves upper bounding $\norm{x}_{V_t^{-1}}$. The lower bound on $N_\theta^t $ and \Cref{lem:min_eigenvalue} yield that $\norm{x}_{V_t^{-1}} \preccurlyeq 1/\sqrt{t\lambda_0}$. 
Note that $\lambda_{\min}^t$ is computable from the observed data, obviating knowledge of $\lambda_0$.
This enables us to prove a regret upper bound without using the elliptical potential lemma, as is done in many prior works in Online Logistic Regression \citet{cesa2006prediction} or in Linear Bandits \citet{abbasi2011improved}.

\section{ALGORITHM DESIGN}

The pseudo-code of our algorithm \SCOUT (Safe Contextual Online Understanding with Thresholds) 
is presented in \Cref{alg:alg1}.
\SCOUT tests a patient ($Z_t = 1$) if the inner product between their context $X_t$ and the current estimate $\theta_t^L$ has a magnitude smaller than an estimator $\tau_t$ of the true threshold $\tau\opt$.
To iteratively refine the estimates of $\theta\opt$ and $\tau\opt$, \SCOUT employs a classical sample-splitting trick to avoid dependencies.
The context distribution $P$ is estimated as $\hat{P}_t$, the empirical distribution of contexts observed from odd samples, $\CS_P^t$, enabling estimation of $\tau\opt$.
$\theta\opt$ is estimated as $\theta_t^L$, using labeled data from even samples where a test was performed, $\CS_\theta^t$.

\begin{algorithm}[t!]
    \caption{\SCOUT} \label{alg:alg1}
    \begin{algorithmic}[1] 
    \STATE \textbf{Input:} Number of rounds $T$, target error rate $\alpha$, confidence level $\delta$
    \STATE \textbf{Initialize:} $\CS_{P}^{(1)} = \emptyset$,  $\CS_{\theta}^{(1)} = \emptyset$. Maintain $N_P^t =|\CS_{P}^t|$, $N_\theta^t = |\CS_{\theta}^t|$
    \FOR{$t=1,2,\hdots,T$}
        \STATE Observe context $X_t$
        \IF{$t\le2$}
        \STATE Set $Z_t =1$
        \ELSE
        \STATE Compute $\theta_t^L$ from \eqref{eq:projection} and $\tau_t$ from \eqref{eq:tau_t_defn}
        \STATE Set $Z_t = \mathds{1}\{\abs{\langle \theta_t^L,X_t\rangle} \le \tau_t\}$
        \ENDIF
        \IF{$Z_t=1$}
            \STATE Observe $Y_t$
            \STATE Predict $\hat{Y}_t = Y_t$
        \ELSE
            \STATE Predict $\hat{Y}_t = \mathds{1}\{\langle X_t, \theta^L_t\rangle > 0\}$
        \ENDIF
        \IF{$Z_t=1$ and $t$ is even}
            \STATE Set $\CS_{\theta}^{t+1} = \CS_{\theta}^{t} \cup \{(X_t,Y_t)\}$
        \ENDIF
        \IF{$t$ is odd}
            \STATE Set $\CS_{P}^{t+1} = \CS_{P}^{t} \cup \{X_t\}$
        \ENDIF
    \ENDFOR
    \end{algorithmic}
\end{algorithm}

The testing condition $Z_t := \indctr{\abs{\dotp{X_t}{\theta_t^L}} \le \tau_t}$ is computed as follows (we defer additional discussion to \Cref{sec:theoretical_analysis}). 
Recall that $\theta^L_t$ is the MLE defined in \Cref{eq:projection}, $\hat{P}_t$ is the empirical distribution of the contexts, and $\lambda_{\min}^t := \lambda_{\min}(V_t)$.
\begin{align}
   \zeta_t(\delta') &:= \sqrt{\frac{(d+1)\log \left(1/\eps_Q\right) + \log\left(\frac{\pi^2t^2}{\delta'}\right)}{4t}},\label{eq:zeta_t}\\
    \tau_t &:= \tau\opt\left(\theta^L_t,\hat{P}_t,\alpha_t -\zeta_t -2B_t/\sqrt{ \lambda_{\min}^t} - \eps_Q\right) + \notag \\
    &\quad 3B_t/\sqrt{ \lambda_{\min}^t} + \eps_Q \label{eq:tau_t_defn}
\end{align}

\begin{figure}[htbp]
    \centering
    \begin{tikzpicture}[scale=.5]
        \draw[thick, <->] (-7.5,0) -- (7.5,0);
        
        \foreach \x/\label in {-7/{-1}, 0/{0}, 7/{+1}} {
            \draw[thick] (\x,-0.1) -- (\x,0.1);
            \node[below, font=\small] at (\x,-0.2) {$\label$};
        }
        
        \definecolor{conv1}{RGB}{200,200,200}
        \definecolor{conv2}{RGB}{150,150,150}
        \definecolor{conv3}{RGB}{100,100,100}
        \definecolor{conv4}{RGB}{50,50,50}
        \definecolor{conv5}{RGB}{25,25,25}
        
        \foreach \i/\pos/\col in {1/-6.9/conv1, 2/-5.9/conv2, 3/-4.9/conv3, 4/-3.9/conv4} {
            \draw[thick, color=\col] (\pos,-0.1) -- (\pos,0.1);
            \node[above, font=\scriptsize, color=\col] at (\pos,0.2) {$-\tau_\i$};
        }
        
        \node[above, font=\scriptsize, color=conv4] at (-3.,0.2) {$\cdots$};
        \draw[thick, color=conv5] (-2.3,-0.1) -- (-2.3,0.1);
        \node[above, font=\scriptsize, color=conv5] at (-2.3,0.2) {$-\tau_t$};
        
        \draw[thick, red, line width=2pt] (-2,-0.1) -- (-2,0.1);
        \node[below, font=\small, red] at (-2,-0.2) {$-\tau^\star$};
        
        \foreach \i/\pos/\col in {1/6.9/conv1, 2/5.9/conv2, 3/4.9/conv3, 4/3.9/conv4} {
            \draw[thick, color=\col] (\pos,-0.1) -- (\pos,0.1);
            \node[above, font=\scriptsize, color=\col] at (\pos,0.2) {$+\tau_\i$};
        }

        \node[above, font=\scriptsize, color=conv4] at (3.15,0.2) {$\cdots$};
        \draw[thick, color=conv5] (2.3,-0.1) -- (2.3,0.1);
        \node[above, font=\scriptsize, color=conv5] at (2.3,0.2) {$+\tau_t$};        
        
        \draw[thick, red, line width=2pt] (2,-0.1) -- (2,0.1);
        \node[below, font=\small, red] at (2,-0.2) {$+\tau^\star$};
        
    \end{tikzpicture}
    \caption{Pessimistic choice of $\abs{\tau_t}$.}
    \label{fig:convergence}
\end{figure}

Our testing threshold $\tau_t$ is designed to be systematically pessimistic.
We begin with a plug-in estimator of the optimal threshold as $\tau\opt(\theta_t^L,\hat{P}_t,\alpha)$.
To guarantee safety, we inflate our threshold to account for estimation errors.
First, we reduce $\alpha$ to $\alpha_t = \max(0,\alpha - \sqrt{{\log(2t^2/\delta')}/{2t}})$ (discussed in \Cref{proof:policy_is_feasible}) to guarantee $(\alpha,\delta)$-safety, if the true $\theta\opt$ and $P$ were known. We set $\delta'$ in \Cref{thm:regret_upper_bound} as $\delta'=\delta/7$.
Then, we reduce $\alpha$ further by $\zeta_t$ (implicitly, $\zeta_t(\delta')$) to account for the fact that $P$ is unknown and we only have $\hat{P}_t$.
Most critically, we add buffer terms proportional to $B_t/\sqrt{\lambda_{\min}^t}$, which tracks the fact that $\theta_t^L$ is not equal to $\theta\opt$, but is not too far away.
Finally, $\eps_Q$ is a quantization parameter to ensure that all the estimators are simultaneously accurate, and is taken as $\eps_Q:=\eps_Q(t) = 1/t^2$.
Through our careful theoretical analysis we can avoid searching through the entire confidence ellipsoid $\CC_t$ in \eqref{eq:conf_ellipse_joint}, which is typically accomplished via expensive convex optimization subroutines, and can instead employ a computationally efficient threshold surrogate $\tau_t$.
This threshold $\tau_t$ provably tests whenever the optimal baseline threshold policy tests, and makes a vanishing number of excess tests. \looseness=-1

\section{THEORETICAL ANALYSIS} \label{sec:theoretical_analysis}

We begin by showing that \SCOUT can accurately estimate $\perr$.
The learner does not start with knowledge of $P$ or $\theta\opt$, and by extension $\tau\opt$ but we show that as \SCOUT improves its estimation of each of these, its estimate of $\perr$ improves.
We analyze this with a sequence of lemmas.

First, we show that, with high probability, our estimates $\perr(\theta,\hat{P}_t,\hat{\tau}_t)$ are close to the true error probability $\perr(\theta,P,\tau)$ (\Cref{lem:perr_accurate}).
To control this across \textit{all} $\theta \in \CB(0,1)$ and $\tau\in[0,1]$, we quantize the set of possible $\theta$ and $\tau$ (denoted $\CQ_\theta$, and $\CQ_\tau$ respectively), and use a union bound to ensure that our error estimates hold simultaneously for all quantized values.
We define this good event as $G_{\perr}$ (\Cref{eq:gt_good_defn}), and show that it holds with probability at least $1-\delta'$ in \Cref{lem:gt_goodevent}.
Additionally, we define our quantized estimator of $\tau$ as $\tau\opt_Q$, which is close to $\tau\opt$:
\begin{align}
    &\tau\opt_Q(\theta,\hat{P},\alpha) := \min \{\tau_Q \in \CQ_\tau : \perr(\theta,\hat{P},\tau_Q) \le \alpha \},\notag \\ 
    &0 \stackrel{}{\le} \tau\opt_Q(\theta,\hat{P},\alpha)-\tau\opt(\theta,\hat{P},\alpha) \stackrel{}{\le} \eps_Q. \hspace{-.4cm}\label{eq:tau_opt-tau_opt_Q_sandwich}
\end{align}

Having established the stability of the optimal threshold to changes in $P$  (\Cref{lem:perr_accurate}), we now show that it is also stable under changes in the parameter $\theta$. To state our results, for any $\theta_Q \in \CQ_\theta \cap \CC_t$ we define an estimator $\hat{\tau}$, which is lower bounded by $\tau\opt$ on $G_{\perr}$ and $G_\theta$:
\begin{align}
    \hat{\tau}(\theta_Q, \hat{P}_t, \alpha) 
    &:= \tau\opt_Q\left(\theta_Q,\hat{P}_t,\alpha - \zeta_t - 2B_t/\sqrt{ \lambda_{\min}^t}\right) + \notag \\
    &\quad 2B_t/\sqrt{ \lambda_{\min}^t},\label{eq:tauhat_tauopt}\\
    \hat{\tau}(\theta_Q, \hat{P}_t, \alpha)  &\ge \tau\opt(\theta\opt, P, \alpha) \ \forall \ \theta_Q \in \CQ_\theta \cap \CC_t. \label{eq:pessimism_hat_star}
\end{align} 
In other words, the empirical $\hat{\tau}$ estimator evaluated at the approximate values $\theta_Q$ and $\hat{P}_t$ provides us with an upper bound for the true threshold $\tau\opt$ evaluated at $\theta\opt$ and $P$.
This enables our design of $\tau_t$ used in the algorithm.
The last property we will need for our analysis is that $\tau\opt$ does not vary too quickly with respect to $\alpha$.
We show that for small $\gamma$, $\tau\opt(\theta\opt,P,\alpha-\gamma)$ is not much larger than $\tau\opt$ (\Cref{lem:tauopt_stability_alpha}).
This necessitates bounding the context probability mass in a subsection of an annulus, and may be of independent interest.
For more details on our stability analysis, we refer the reader to \Cref{app:stability}.

\subsection{Defining a good event}\label{sec:good_event}

As is common practice in Multi-Armed Bandit analyses, we define a ``good event'' under which all concentration arguments hold, and condition on this event for the remainder of our analysis.
Recall that $N_{\theta}^{t} = \left| \CS_{\Theta}^t \right| $ denotes the number of samples $(X_s,Y_s)$ collected to estimate $\theta\opt$ up to round $t$, and similarly $N_{P}^{t} = \left| \CS_P^t \right|$ is the number of samples collected used in the context distribution estimation. \looseness=-1

\begin{defn}\label{def:good_event} The good event $G = G_\theta \cap G_{\perr} \cap G_N \cap G_\lambda$ is comprised of the following:
    \begin{enumerate}
        \item $G_\theta$: The confidence sets $\mathcal{C}_t$ (defined in \Cref{lem:faury_concentration}) are valid, in that $\theta\opt \in \mathcal{C}_t$ for all $t$. 
        \item $G_{\perr}$: The estimates of $\perr$ on $\CQ_\theta \times \CQ_\tau$ are $\zeta_t$ accurate for all $t$ (\Cref{lem:perr_accurate}). 
        \item $G_N$: the confidence sets get enough samples. $G_N  = \bigcap_{t=1}^T G_N^{(t)}$, where $G_N^{(t)}$ is the event that 
        $N_{\theta}^{(t)} \ge p\opt t/2  - \sqrt{\frac{\ln(\pi t^2/(3\delta'))}{2}}$.
        
        \item $G_\lambda$: The minimum eigenvalue of the empirical covariance matrix grows linearly in $t$. 
        Concretely, $G_\lambda = \cap_{t=T_0}^T G_\lambda^{(t)}$, where $G_\lambda^{(t)}$ is the event that $\lambda_{\min}^t \geq p\opt t\lambda_{0}/12$.
    \end{enumerate}
\end{defn}

Detailed proofs are deferred to \Cref{app:good_event_proof}.
The first event $G_\theta$ satisfies $\P(G_\theta) \ge 1 - \delta'$ by \Cref{lem:faury_concentration}.
The second event $G_{\perr}$ satisfies $\P(G_{\perr}) \ge 1 - \delta'$ by \Cref{lem:gt_goodevent}.
To prove that $G_N$ holds with high probability, we utilize the fact that on $G_\theta$ and $G_{\perr}$, when the optimal policy tests then our policy tests as well, as proved in \Cref{lem:Z_t_pessimistic}.
Combining this fact with Hoeffding's inequality yields the desired result in \Cref{lem:ntheta_t_event}.
When $G_N$ holds, we have $N_{\theta}^{t} \ge p\opt t /3$ for all $t \ge T_0$ where $T_0$ is a large constant.
For the last event, $\P(G_\lambda) \ge 1 - 2\delta'$, which we show via a covering argument used to bound the minimum eigenvalue of the empirical covariance matrix $V_t$ (\Cref{lemma:min_eigenv_lb}), and $G_N$ to lower bound the number of samples used. Thus, \looseness=-1

\begin{lemma}\label{lem:total_good_event}
The good event $G$ holds with high probability: $\P(G) \ge 1-6\delta'$.
\end{lemma}

\subsection{Safety Analysis}\label{sec:algorithmic_rules}

Our testing rule is designed to be computationally efficient and pessimistic. 
Here, pessimism means that whenever the baseline policy tests, our policy does the same.
To prove the $(\alpha,\delta)$-safety of \SCOUT, we utilize two helper lemmas.
In \Cref{lem:Z_t_pessimistic}, we prove that when the baseline policy tests for $\tau\opt(\theta\opt,P,\alpha_t)$, our policy tests as well.
In \Cref{lem:prediction_imitation}, we prove that when the baseline policy predicts, our policy outputs the same prediction. Combining these yields the desired result. \looseness=-1

\begin{restatable}{lemma}{policyIsFeasible}\label{lem:policy_is_feasible}
    Under $G$, \SCOUT is $(\alpha,\delta')$-safe. 
\end{restatable}

\begin{figure*}[t!]
    \centering
    \begin{subfigure}{\linewidth}
        \centering
        \includegraphics[width=\linewidth]{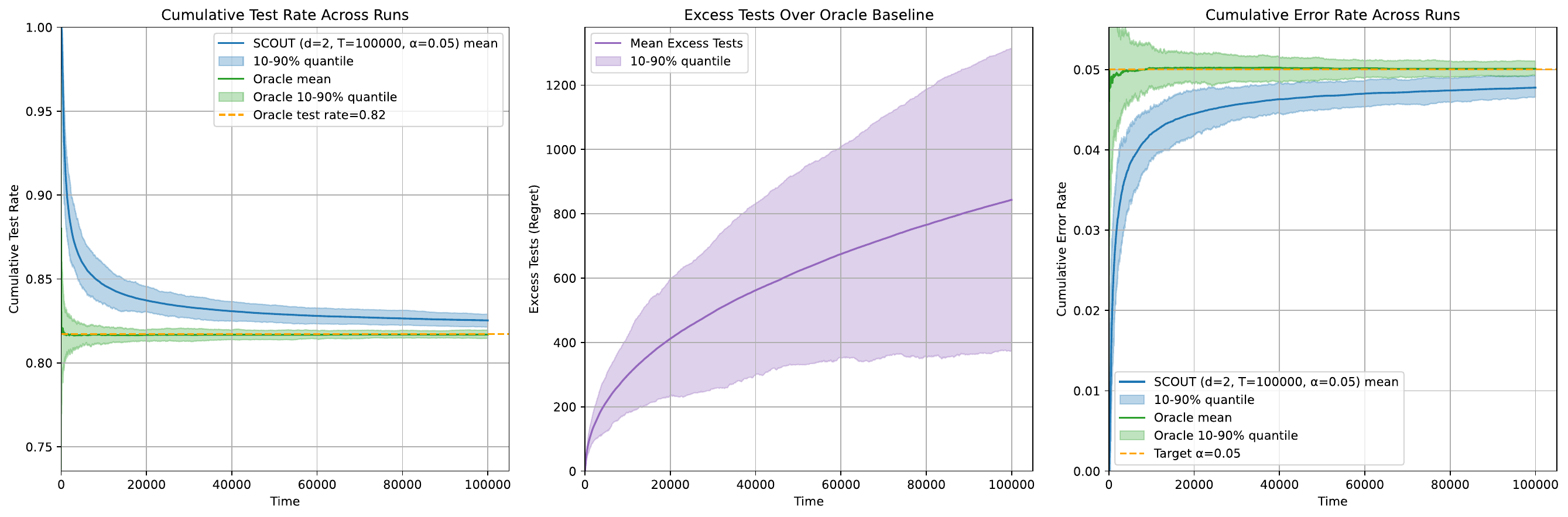}
        \label{fig:subfig1}
    \end{subfigure}
    \hfill
    \begin{subfigure}{\linewidth}
        \centering
        \includegraphics[width=\linewidth]{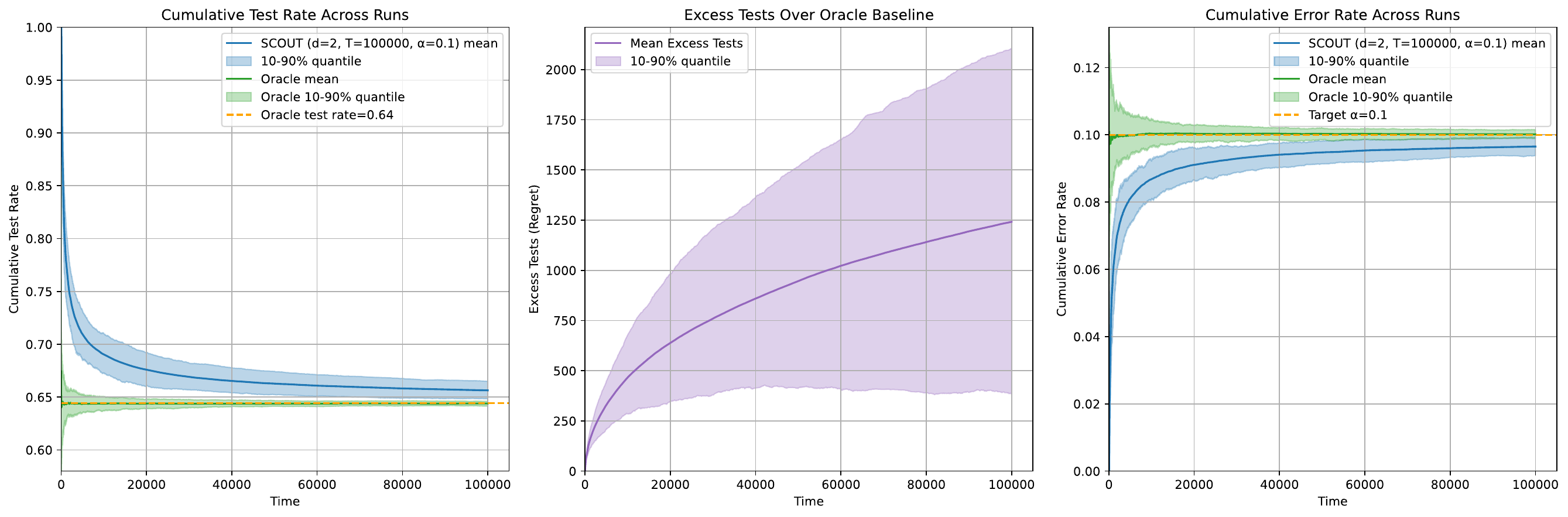}
        \label{fig:subfig2}
    \end{subfigure}
    \hfill
    \begin{subfigure}{\linewidth}
    \centering
    \includegraphics[width=\linewidth]{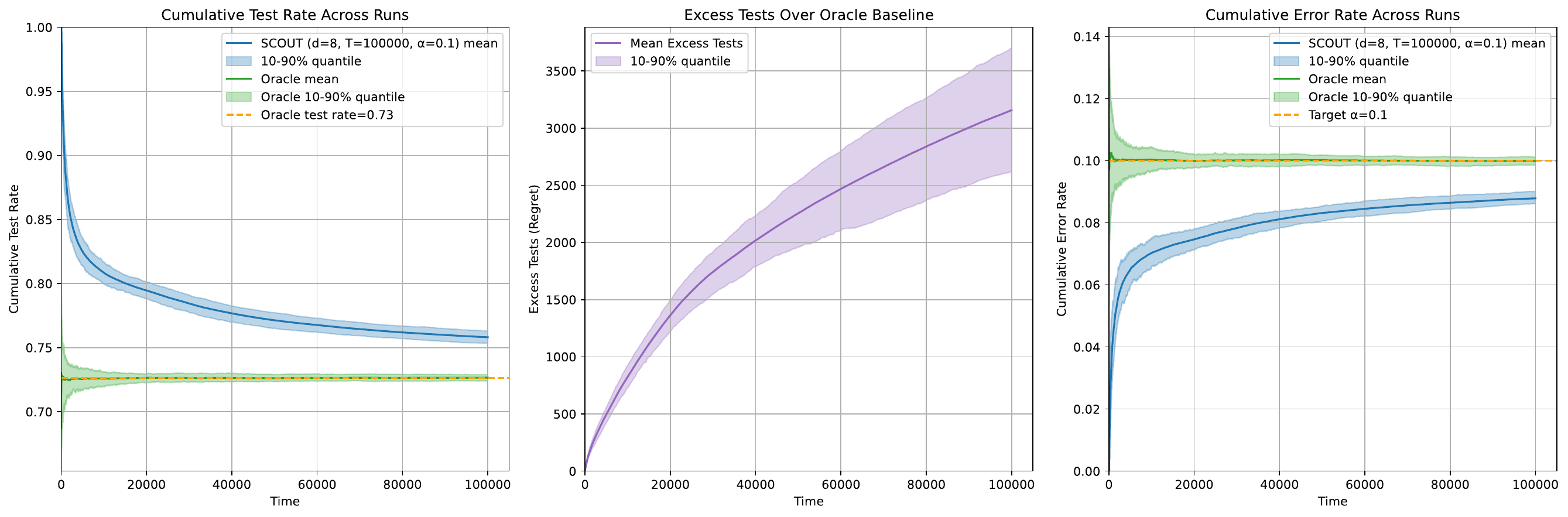}
    \label{fig:1subfig3}
    \end{subfigure}
        \vspace{-1.2cm}
    \caption{Simulation results. 
    Rows 1 and 2 correspond to $d=2$, where the first shows $\alpha=0.05$, and the second $\alpha=0.1$.
    Row 3 shows $d=8,\alpha=0.1$. $x$-axis corresponds to time $t$.
    Left plots show the cumulative test rate (10-90\% quantiles shaded), where blue shows the performance of \SCOUT, orange shows the oracle test rate at $p\opt$, and green shows the empirical test rate under $\tau\opt$.
    The middle plots show the excess number of tests, demonstrating the sublinear regret of \SCOUT.
    The right plots show the misclassification rate of \SCOUT.
    While the optimal baseline policy fluctuates around the desired threshold $\alpha$, often exceeding it, \SCOUT starts far below (very safe) then gradually learns to be more aggressive, approaching misclassification rate $\alpha$ but never exceeding it.}
    \label{fig:main}
\end{figure*}

\subsection{Regret Analysis}

To derive a regret bound, we begin by proving a bound on the instantaneous regret during rounds $t>T_0$ (\Cref{lem:regret_bound}, proof in \Cref{sec:regret_analysis}).
Summing this lemma over $t$ yields the following Theorem, where we set $\delta' = \delta/7$.
\begin{restatable}{thm}{regretUpperBound}
\label{thm:regret_upper_bound}
\SCOUT satisfies $(\alpha,\delta)$-safety and has regret bounded by
    \begin{equation*}
        T_0 + \tilde{C}\frac{M}{m}\sqrt{ \frac{dT\log\left(T/\delta\right)}{p\opt \lambda_0}},
    \end{equation*}
    for an absolute constant $\tilde{C}>0$, which is made explicit in the proof (\Cref{sec:regret_analysis}).
\end{restatable} 
Note that the probability parameter $\delta$ can scale exponentially in $T$ without changing the regret. 
While at first our algorithm may appear to beat the linear dimension dependence expected in linear bandits, this missing factor is hidden in $\lambda_0$.
In \Cref{sec:regret_analysis}, we can apply a lower bound for $\lambda_0$ (see \Cref{lem:min_eigenvalue}) and recover the $\tilde{\CO}(d\sqrt{T})$ regret bound.
For a detailed synopsis of our work and potential future extensions, see \Cref{sec:disc}.

\section{NUMERICAL RESULTS}\label{sec:experiments}
We corroborate our theoretical guarantees with numerical simulations, showing that \SCOUT is able to efficiently compute the testing rule and converge to the optimal error rate.
The code for reproducing our experiments is available online: \url{https://github.com/TavorB/SCOUT}.
We generate simulations varying the dimensionality and the target error rate $\alpha$, highlighting the rapid convergence of our method when $p\opt$ is large.
We discuss several algorithmic modifications in \Cref{app:mod_from_written}, including batched parameter updates and omission of the projection step, which allow the algorithm to run efficiently while retaining the core principles of \SCOUT.
The empirical results, which demonstrate sublinear regret and adherence to the safety constraint across all instances, validate that these practical simplifications do not compromise the algorithm's performance in our simulated environments.

\section{DISCUSSION}\label{sec:disc}

In this work we introduced \SCOUT, the first algorithm that provably balances \textbf{no-regret learning} with a \textbf{high-probability safety guarantee} on the empirical misclassification rate in logistic bandits.  
Our analysis shows that a simple, efficiently-computable testing rule suffices to achieve the order optimal $\widetilde O\!\bigl(\sqrt{dT/\lambda_0}\bigr)$ excess-test rate.
Empirical results confirm that these bounds translate to practice on moderately large horizons.

In medical triage---our motivating use-case---\SCOUT can be viewed as a “test-or-treat’’ policy that automatically calibrates how aggressively to screen as new evidence accrues.  
Because the policy is pessimistic by design, it never tests less than an oracle baseline that knows both the patient distribution and the ground-truth regression coefficients.  
This property is attractive in any high-stakes domain where misclassifications are costly (e.g.\ credit risk, fraud detection, or industrial quality control).

There are many interesting directions of future work.
One simple extension is handling unequal Type-I and Type-II error control. The threshold-selection step can be split to handle false positives and false negatives separately by using two one-sided versions of $\perr$.
Additionally, we can use improved confidence bounds from \citet{lee2024unified} in \Cref{lem:faury_concentration} to remove the $\kappa$ factor in $B_t$ and generalize to larger context and $\Theta$ sets.
Less straightforwardly, we have the setting where the optimal baseline does not need to test, i.e., $p\opt = 0$.  If the optimal policy never tests, can one detect \emph{fast enough} that screening is unnecessary while still retaining the high-probability safety constraint?
Going beyond stochastic contexts, we plan to explore whether the ideas behind \SCOUT can be combined with online calibration tools to handle non-stationary or even adversarial $X_t$.  \looseness=-1

\clearpage
\balance

\bibliography{refs}

\begin{thebibliography}{49}
\providecommand{\natexlab}[1]{#1}
\providecommand{\url}[1]{\texttt{#1}}
\expandafter\ifx\csname urlstyle\endcsname\relax
  \providecommand{\doi}[1]{doi: #1}\else
  \providecommand{\doi}{doi: \begingroup \urlstyle{rm}\Url}\fi

\bibitem[Abbasi-Yadkori et~al.(2011)Abbasi-Yadkori, P{\'a}l, and Szepesv{\'a}ri]{abbasi2011improved}
Yasin Abbasi-Yadkori, D{\'a}vid P{\'a}l, and Csaba Szepesv{\'a}ri.
\newblock Improved algorithms for linear stochastic bandits.
\newblock \emph{Advances in neural information processing systems}, 24, 2011.

\bibitem[Afsar et~al.(2022)Afsar, Crump, and Far]{afsar2022reinforcement}
M~Mehdi Afsar, Trafford Crump, and Behrouz Far.
\newblock Reinforcement learning based recommender systems: A survey.
\newblock \emph{ACM Computing Surveys}, 55\penalty0 (7):\penalty0 1--38, 2022.

\bibitem[Bar-El~Avidan and Bistritz(2025)]{bar2025juggler}
Ella Bar-El~Avidan and Ilai Bistritz.
\newblock Juggler: Multitask learning with task performance constraints.
\newblock In \emph{2025 IEEE 64th Conference on Decision and Control (CDC)}, pages 8009--8014. IEEE, 2025.

\bibitem[Bastani et~al.(2022)Bastani, Drakopoulos, Gupta, Vlachogiannis, Hadjichristodoulou, Lagiou, Magiorkinis, Paraskevis, and Tsiodras]{bastani2022interpretable}
Hamsa Bastani, Kimon Drakopoulos, Vishal Gupta, Jon Vlachogiannis, Christos Hadjichristodoulou, Pagona Lagiou, Gkikas Magiorkinis, Dimitrios Paraskevis, and Sotirios Tsiodras.
\newblock Interpretable operations research for high-stakes decisions: Designing the greek covid-19 testing system.
\newblock \emph{INFORMS Journal on Applied Analytics}, 52\penalty0 (5):\penalty0 398--411, 2022.

\bibitem[Camilleri et~al.(2021)Camilleri, Xiong, Fazel, Jain, and Jamieson]{camilleri2021selective}
Romain Camilleri, Zhihan Xiong, Maryam Fazel, Lalit Jain, and Kevin~G Jamieson.
\newblock Selective sampling for online best-arm identification.
\newblock \emph{Advances in Neural Information Processing Systems}, 34:\penalty0 11071--11082, 2021.

\bibitem[Camilleri et~al.(2022)Camilleri, Wagenmaker, Morgenstern, Jain, and Jamieson]{camilleri2022active}
Romain Camilleri, Andrew Wagenmaker, Jamie~H Morgenstern, Lalit Jain, and Kevin~G Jamieson.
\newblock Active learning with safety constraints.
\newblock \emph{Advances in Neural Information Processing Systems}, 35:\penalty0 33201--33214, 2022.

\bibitem[Cand{\`e}s et~al.(2025)Cand{\`e}s, Ilyas, and Zrnic]{candes2025probably}
Emmanuel~J Cand{\`e}s, Andrew Ilyas, and Tijana Zrnic.
\newblock Probably approximately correct labels.
\newblock \emph{arXiv preprint arXiv:2506.10908}, 2025.

\bibitem[Cesa-Bianchi and Lugosi(2006)]{cesa2006prediction}
Nicolo Cesa-Bianchi and G{\'a}bor Lugosi.
\newblock \emph{Prediction, learning, and games}.
\newblock Cambridge university press, 2006.

\bibitem[Cesa-Bianchi et~al.(2006)Cesa-Bianchi, Gentile, Zaniboni, and Warmuth]{cesa2006worst}
Nicolo Cesa-Bianchi, Claudio Gentile, Luca Zaniboni, and Manfred Warmuth.
\newblock Worst-case analysis of selective sampling for linear classification.
\newblock \emph{Journal of Machine Learning Research}, 7\penalty0 (7), 2006.

\bibitem[Chung and Lu(2006)]{chung2006concentration}
Fan Chung and Linyuan Lu.
\newblock Concentration inequalities and martingale inequalities: a survey.
\newblock \emph{Internet mathematics}, 3\penalty0 (1):\penalty0 79--127, 2006.

\bibitem[Cortes et~al.(2016)Cortes, DeSalvo, and Mohri]{cortes2016learning}
Corinna Cortes, Giulia DeSalvo, and Mehryar Mohri.
\newblock Learning with rejection.
\newblock In \emph{Algorithmic Learning Theory: 27th International Conference, ALT 2016, Bari, Italy, October 19-21, 2016, Proceedings 27}, pages 67--82. Springer, 2016.

\bibitem[Cortes et~al.(2018)Cortes, DeSalvo, Gentile, Mohri, and Yang]{cortes2018online}
Corinna Cortes, Giulia DeSalvo, Claudio Gentile, Mehryar Mohri, and Scott Yang.
\newblock Online learning with abstention.
\newblock In \emph{international conference on machine learning}, pages 1059--1067. PMLR, 2018.

\bibitem[Dara et~al.(2022)Dara, Dhamercherla, Jadav, Babu, and Ahsan]{dara2022machine}
Suresh Dara, Swetha Dhamercherla, Surender~Singh Jadav, CH~Madhu Babu, and Mohamed~Jawed Ahsan.
\newblock Machine learning in drug discovery: a review.
\newblock \emph{Artificial intelligence review}, 55\penalty0 (3):\penalty0 1947--1999, 2022.

\bibitem[Dasgupta et~al.(2005)Dasgupta, Kalai, and Monteleoni]{dasgupta2005analysis}
Sanjoy Dasgupta, Adam~Tauman Kalai, and Claire Monteleoni.
\newblock Analysis of perceptron-based active learning.
\newblock In \emph{International conference on computational learning theory}, pages 249--263. Springer, 2005.

\bibitem[Dekel et~al.(2012)Dekel, Gentile, and Sridharan]{dekel2012selective}
Ofer Dekel, Claudio Gentile, and Karthik Sridharan.
\newblock Selective sampling and active learning from single and multiple teachers.
\newblock \emph{The Journal of Machine Learning Research}, 13\penalty0 (1):\penalty0 2655--2697, 2012.

\bibitem[Diakonikolas et~al.(2021)Diakonikolas, Kane, Kontonis, Tzamos, and Zarifis]{diakonikolas2021efficiently}
Ilias Diakonikolas, Daniel~M Kane, Vasilis Kontonis, Christos Tzamos, and Nikos Zarifis.
\newblock Efficiently learning halfspaces with tsybakov noise.
\newblock In \emph{Proceedings of the 53rd Annual ACM SIGACT Symposium on Theory of Computing}, pages 88--101, 2021.

\bibitem[Diakonikolas et~al.(2024)Diakonikolas, Kontonis, Tzamos, and Zarifis]{diakonikolas2024online}
Ilias Diakonikolas, Vasilis Kontonis, Christos Tzamos, and Nikos Zarifis.
\newblock Online learning of halfspaces with massart noise.
\newblock \emph{arXiv preprint arXiv:2405.12958}, 2024.

\bibitem[Duan et~al.(2023)Duan, Zhao, Qi, Zhou, Wang, and Shi]{duan2023towards}
Yue Duan, Zhen Zhao, Lei Qi, Luping Zhou, Lei Wang, and Yinghuan Shi.
\newblock Towards semi-supervised learning with non-random missing labels.
\newblock In \emph{Proceedings of the IEEE/CVF International Conference on Computer Vision}, pages 16121--16131, 2023.

\bibitem[Egorova et~al.(2023)Egorova, Gil, Segura, Temme, et~al.]{egorova2023computation}
Vera Egorova, Amparo Gil, Javier Segura, NM~Temme, et~al.
\newblock Computation of the regularized incomplete beta function.
\newblock 2023.

\bibitem[Faury et~al.(2020)Faury, Abeille, Calauz{\`e}nes, and Fercoq]{faury2020improved}
Louis Faury, Marc Abeille, Cl{\'e}ment Calauz{\`e}nes, and Olivier Fercoq.
\newblock Improved optimistic algorithms for logistic bandits.
\newblock In \emph{International Conference on Machine Learning}, pages 3052--3060. PMLR, 2020.

\bibitem[Folland(1999)]{folland1999real}
Gerald~B Folland.
\newblock \emph{Real analysis: modern techniques and their applications}.
\newblock John Wiley \& Sons, 1999.

\bibitem[Freund et~al.(1997)Freund, Seung, Shamir, and Tishby]{freund1997selective}
Yoav Freund, H~Sebastian Seung, Eli Shamir, and Naftali Tishby.
\newblock Selective sampling using the query by committee algorithm.
\newblock \emph{Machine learning}, 28:\penalty0 133--168, 1997.

\bibitem[Gangrade et~al.(2021{\natexlab{a}})Gangrade, Kag, Cutkosky, and Saligrama]{gangrade2021online}
Aditya Gangrade, Anil Kag, Ashok Cutkosky, and Venkatesh Saligrama.
\newblock Online selective classification with limited feedback.
\newblock \emph{Advances in Neural Information Processing Systems}, 34:\penalty0 14529--14541, 2021{\natexlab{a}}.

\bibitem[Gangrade et~al.(2021{\natexlab{b}})Gangrade, Kag, and Saligrama]{gangrade2021selective}
Aditya Gangrade, Anil Kag, and Venkatesh Saligrama.
\newblock Selective classification via one-sided prediction.
\newblock In \emph{International Conference on Artificial Intelligence and Statistics}, pages 2179--2187. PMLR, 2021{\natexlab{b}}.

\bibitem[Gangrade et~al.(2024)Gangrade, Chen, and Saligrama]{gangrade2024safe}
Aditya Gangrade, Tianrui Chen, and Venkatesh Saligrama.
\newblock Safe linear bandits over unknown polytopes.
\newblock In \emph{The Thirty Seventh Annual Conference on Learning Theory}, pages 1755--1795. PMLR, 2024.

\bibitem[Giudici(2024)]{giudici2024safe}
Paolo Giudici.
\newblock Safe machine learning.
\newblock \emph{Statistics}, 58\penalty0 (3):\penalty0 473--477, 2024.

\bibitem[Goel et~al.(2023)Goel, Hanneke, Moran, and Shetty]{goel2023adversarial}
Surbhi Goel, Steve Hanneke, Shay Moran, and Abhishek Shetty.
\newblock Adversarial resilience in sequential prediction via abstention.
\newblock \emph{Advances in Neural Information Processing Systems}, 36:\penalty0 8027--8047, 2023.

\bibitem[Gu et~al.(2022)Gu, Yang, Du, Chen, Walter, Wang, and Knoll]{gu2022review}
Shangding Gu, Long Yang, Yali Du, Guang Chen, Florian Walter, Jun Wang, and Alois Knoll.
\newblock A review of safe reinforcement learning: Methods, theory and applications.
\newblock \emph{arXiv preprint arXiv:2205.10330}, 2022.

\bibitem[Hanneke and Yang(2021)]{hanneke2021toward}
Steve Hanneke and Liu Yang.
\newblock Toward a general theory of online selective sampling: Trading off mistakes and queries.
\newblock In \emph{International Conference on Artificial Intelligence and Statistics}, pages 3997--4005. PMLR, 2021.

\bibitem[Horn and Johnson(2012)]{horn2012matrix}
Roger~A Horn and Charles~R Johnson.
\newblock \emph{Matrix analysis}.
\newblock Cambridge university press, 2012.

\bibitem[Jorgensen(2014)]{jorgensen2014volumes}
Michael Jorgensen.
\newblock Volumes of n-dimensional spheres and ellipsoids, 2014.

\bibitem[Kazerouni et~al.(2017)Kazerouni, Ghavamzadeh, Abbasi~Yadkori, and Van~Roy]{kazerouni2017conservative}
Abbas Kazerouni, Mohammad Ghavamzadeh, Yasin Abbasi~Yadkori, and Benjamin Van~Roy.
\newblock Conservative contextual linear bandits.
\newblock \emph{Advances in Neural Information Processing Systems}, 30, 2017.

\bibitem[Lattimore and Szepesv{\'a}ri(2020)]{lattimore2020bandit}
Tor Lattimore and Csaba Szepesv{\'a}ri.
\newblock \emph{Bandit algorithms}.
\newblock Cambridge University Press, 2020.

\bibitem[Lee et~al.(2025)Lee, Yun, and Jun]{lee2024unified}
Junghyun Lee, Se-Young Yun, and Kwang-Sung Jun.
\newblock A unified confidence sequence for generalized linear models, with applications to bandits.
\newblock \emph{Advances in Neural Information Processing Systems}, 37:\penalty0 124640--124685, 2025.

\bibitem[Li(2010)]{li2010concise}
Shengqiao Li.
\newblock Concise formulas for the area and volume of a hyperspherical cap.
\newblock \emph{Asian Journal of Mathematics \& Statistics}, 4\penalty0 (1):\penalty0 66--70, 2010.

\bibitem[Orabona et~al.(2011)Orabona, Cesa-Bianchi, et~al.]{orabona2011better}
Francesco Orabona, Nicolo Cesa-Bianchi, et~al.
\newblock Better algorithms for selective sampling.
\newblock In \emph{Proceedings of the 28th international conference on machine learning: Bellevue, Washington, USA}, pages 433--440. Omnipress, 2011.

\bibitem[Pacchiano et~al.(2021)Pacchiano, Ghavamzadeh, Bartlett, and Jiang]{pacchiano2021stochastic}
Aldo Pacchiano, Mohammad Ghavamzadeh, Peter Bartlett, and Heinrich Jiang.
\newblock Stochastic bandits with linear constraints.
\newblock In \emph{International conference on artificial intelligence and statistics}, pages 2827--2835. PMLR, 2021.

\bibitem[Pinelis and Ruppert(2022)]{pinelis2022machine}
Michael Pinelis and David Ruppert.
\newblock Machine learning portfolio allocation.
\newblock \emph{The Journal of Finance and Data Science}, 8:\penalty0 35--54, 2022.

\bibitem[Sekhari et~al.(2023)Sekhari, Sridharan, Sun, and Wu]{sekhari2023selective}
Ayush Sekhari, Karthik Sridharan, Wen Sun, and Runzhe Wu.
\newblock Selective sampling and imitation learning via online regression.
\newblock \emph{Advances in Neural Information Processing Systems}, 36:\penalty0 67213--67268, 2023.

\bibitem[Settles(2009)]{settles2009active}
Burr Settles.
\newblock Active learning literature survey.
\newblock 2009.

\bibitem[Seung et~al.(1992)Seung, Opper, and Sompolinsky]{seung1992query}
H~Sebastian Seung, Manfred Opper, and Haim Sompolinsky.
\newblock Query by committee.
\newblock In \emph{Proceedings of the fifth annual workshop on Computational learning theory}, pages 287--294, 1992.

\bibitem[Slivkins(2013)]{slivkins2013dynamic}
Aleksandrs Slivkins.
\newblock Dynamic ad allocation: Bandits with budgets.
\newblock \emph{arXiv preprint arXiv:1306.0155}, 2013.

\bibitem[Sutton et~al.(1999)Sutton, Barto, et~al.]{sutton1999reinforcement}
Richard~S Sutton, Andrew~G Barto, et~al.
\newblock Reinforcement learning.
\newblock \emph{Journal of Cognitive Neuroscience}, 11\penalty0 (1):\penalty0 126--134, 1999.

\bibitem[Tsybakov(2004)]{tsybakov2004optimal}
Alexander~B Tsybakov.
\newblock Optimal aggregation of classifiers in statistical learning.
\newblock \emph{The Annals of Statistics}, 32\penalty0 (1):\penalty0 135--166, 2004.

\bibitem[Vamathevan et~al.(2019)Vamathevan, Clark, Czodrowski, Dunham, Ferran, Lee, Li, Madabhushi, Shah, Spitzer, et~al.]{vamathevan2019applications}
Jessica Vamathevan, Dominic Clark, Paul Czodrowski, Ian Dunham, Edgardo Ferran, George Lee, Bin Li, Anant Madabhushi, Parantu Shah, Michaela Spitzer, et~al.
\newblock Applications of machine learning in drug discovery and development.
\newblock \emph{Nature reviews Drug discovery}, 18\penalty0 (6):\penalty0 463--477, 2019.

\bibitem[Vershynin(2018)]{vershynin2018high}
Roman Vershynin.
\newblock \emph{High-dimensional probability: An introduction with applications in data science}, volume~47.
\newblock Cambridge university press, 2018.

\bibitem[Wainwright(2019)]{wainwright2019high}
Martin~J Wainwright.
\newblock \emph{High-dimensional statistics: A non-asymptotic viewpoint}, volume~48.
\newblock Cambridge university press, 2019.

\bibitem[Yao et~al.(2021)Yao, Brunskill, Pan, Murphy, and Doshi-Velez]{yao2021power}
Jiayu Yao, Emma Brunskill, Weiwei Pan, Susan Murphy, and Finale Doshi-Velez.
\newblock Power constrained bandits.
\newblock In \emph{Machine Learning for Healthcare Conference}, pages 209--259. PMLR, 2021.

\bibitem[Zhu and Van~Roy(2023)]{zhu2023scalable}
Zheqing Zhu and Benjamin Van~Roy.
\newblock Scalable neural contextual bandit for recommender systems.
\newblock In \emph{Proceedings of the 32nd ACM International Conference on Information and Knowledge Management}, pages 3636--3646, 2023.

\end{thebibliography}
\bibliographystyle{plainnat}

\clearpage
\appendix
\thispagestyle{empty}

\onecolumn

\hsize\textwidth
\linewidth\hsize
{\centering{\Large\bfseries Appendix \par}}
\vskip 0.2in

\addcontentsline{toc}{section}{Supplementary Material}


\appendix

\section{Baseline policy}\label{app:baseline_policy}
Here we provide some discussion and proofs regarding the optimal baseline we compare to.

\subsection{Proof of ~\Cref{prop:optimal_policy_on_expectation}}

\begin{proof}

When the value of the parameter $\theta\opt$ and the collection of the contexts $\{X_t\}_{t=1}^{T}$ are known, we can equivalently write the problem as follows.
Let $p_t = \mu(X_t^\top\theta\opt)$, the labels $Y_t\sim Ber(p_t)$ independently across $t$.

To compute the expected error, that is $\E(E_t) := \E(\indctr{\hat{Y}_t \neq Y_t})$, we only need to examine the case where we do not test.
When we do test, we observe the true label and incur zero error.
For $Z_t = 0$ then, the expected error is
\begin{enumerate}
    \item If $\hat{Y}_t = 1$ then $\E(\indctr{\hat{Y}_t \neq Y_t} \mid \hat{Y}_t = 1) = 1 - p_t$.
    \item Else if $\hat{Y}_t = 0$ then $\E(\indctr{\hat{Y}_t \neq Y_t} \mid \hat{Y}_t = 0) = p_t$.
\end{enumerate}

The optimal policy then is to output the prediction with the smallest error.
The expected error then is equal to 
$$\E(\indctr{\hat{Y}_t \neq Y_t}) := \min\{ 1 - p_t , p_t \}.$$

We denote $\P(Z_t = 0) = \eta_t$.
The optimal policy choice is reduced to the following optimization problem.

\begin{equation}
\min_{\{\eta_t\}} \sum_{t=1}^T 1-\eta_t \quad \text{s.t.} \quad \frac{1}{T}\sum_{t=1}^T \min\{1-p_t, p_t\} \eta_t \le \alpha, \quad 0\leq\eta_t\leq 1.
\end{equation}

Or equivalently can be written as:
\begin{equation}
\max_{\{\eta_t\}} \sum_{t=1}^T \eta_t \quad \text{s.t.} \quad \frac{1}{T}\sum_{t=1}^T \min\{1-p_t, p_t\} \eta_t \le \alpha, \quad 0\leq\eta_t\leq 1.
\end{equation}

The solution of this Linear Program is the solution of the \textit{Fractional Knapsack} problem with budget $\alpha$.
This problem can be optimally solved with a greedy strategy, sorting the coefficients $\min\{1-p_t, p_t\}$ in non-increasing order and assigning $\eta = 1$ to the lowest "error" contexts until we do not violate the budget constraint $\alpha$.
This strategy is clearly a threshold strategy that depends on $a$.
    
\end{proof}

\subsubsection{Conversion to $(\alpha,\delta)$ safety}

It is worth mentioning that solving the problem by satisfying the constraint in expectation does not provide any guarantees when we require the constraint to hold with high probability. 
Even if we apply the Markov's inequality to convert the constraint in expectation to a high probability one, we derive a very loose bound (need to target error rate $\alpha^2$ to obtain a high probability bound of $\alpha$).

$$\P\left( \frac{1}{T}\sum_{t=1}^T \mathds{1}\{\hat{Y}_t \neq Y_t\} \ge \alpha \right) \le \frac{\E\left[\frac{1}{T}\sum_{t=1}^T \mathds{1}\{\hat{Y}_t \neq Y_t\} \right]}{\alpha} \le 1.$$

However, we show that we are still competitive with respect to this fixed baseline policy.

\subsection{Proof of \Cref{lem:min_eigenvalue}}\label{app:min_eigenval_discussion}

We outline the proof as follows; as $\norm{\theta\opt} = 1$, a ball of radius $\tau\opt$ is a subset of the contexts tested by the baseline policy.
The contexts drawn from this ball form a positive definite covariance matrix, which implies that the minimum eigenvalue of the overall covariance matrix is positive.

\lemMinEigenvalue*

\begin{proof}
    By Cauchy-Schwarz, 
    $\abs{X^\top\theta\opt} \le \norm{X},$ as $\|\theta\opt\|=1$.
    As a result all contexts $X \in \CB(0,\tau\opt)$ satisfy $\abs{X^\top\theta\opt} \le \tau\opt$ and thus are tested by the baseline policy.
    We can split the set of contexts to be tested by the baseline policy, $\CT = \{X \in \CB(0,1): \abs{X^\top\theta\opt} \le \tau\opt\}$ into $\CB(0,\tau\opt) \cup \left( \CT \setminus \CB(0,\tau\opt) \right).$

    We begin by showing that the covariance matrix of the contexts tested by the baseline policy under a \textit{uniform} context distribution has a positive minimum eigenvalue.
    Then, leveraging the assumption that $P$ is lower bounded (\Cref{assum:density}), we prove our desired claim.

    Let $V_d(1)$ be the volume of the \textit{d-dimensional} unit ball.
    We begin by showing that the minimum eigenvalue of the uniform distribution on the \textit{d-dimensional} unit ball is positive using standard arguments as in \citet{vershynin2018high} (Version 2, Section 3.3.3).

\begin{lemma}\label{lem:min_eigenv_uniform_distribution}
    The minimum eigenvalue of  $X$ drawn uniformly from the \textit{d-dimensional} ball satisfies:
    \begin{equation*}
        \lambda_{\min}\left(\E_{\Bx \sim \textnormal{Unif}(\CB(0,1))} \left[\Bx \Bx^\top\right]\right) = \frac{1}{d+2}.
    \end{equation*}
\end{lemma}

\begin{proof}
The quantity $\E [\Bx \Bx^\top]$ is the covariance matrix of the uniform over the unit \textit{d-dimensional} ball.
For $\Bx \sim \text{Unif}(\CB(0,1))$, $\E [\Bx \Bx^\top]$ can be written as $a\BI_d$ due to spherical symmetry. 

By a change of variables, we can obtain that $\E[\Bx_i \Bx_j] = -\E[\Bx_i \Bx_j]$ for $i\neq j$, implying that $\E[\Bx_i \Bx_j] = 0$.  To compute the diagonal entries:

\begin{align*}
    \E [x_i^2]  &= \frac{1}{d}\E[\Bx^2] \\
    &= \frac{1}{d} \int_{\norm{\Bx}_2^2 \leq 1} \frac{\Bx^2}{V_d(1)} d\Bx \\
    &= \frac{1}{d V_d(1)} \int_{\mathcal{S}^{d-1}}\int_{0 \leq r \leq 1} r^2 r^{d-1} dr d\sigma(\omega) \\
    &= \frac{S_d(1)}{V_d(1)} \frac{1}{d(d+2)}\\
    &= \frac{1}{d+2}
\end{align*}
where $S_d(1)$ is the surface of the unit sphere and $d\sigma$ any surface measure.
In the last line, we leverage the volume to surface area ratio of $\CB(0,1)$:
$$\frac{V_d(1)}{S_d(1)} = \frac{\frac{\pi^{d/2}}{\Gamma(d/2+1)}}{\frac{d\pi^{d/2}}{\Gamma(d/2+1)}} = \frac{1}{d}.$$
Thus, all eigenvalues of this covariance matrix are equal to $1/(d+2)$.
\end{proof}

    Now, as our density is smooth, we can use that for all $\Bx,\upsilon \in \CB(0,1)$ it holds $(\Bx^\top\upsilon)^2 p(\Bx) \geq (\Bx^\top\upsilon)^2 m$ and so:
    
    \begin{align*}
    \lambda_0 &=
    \lambda_{\min}\left(\E_P \left[X X^\top\  \middle \vert \   |\langle X, \theta\opt\rangle | \leq \tau\opt \right]\right)\\
    &=
        \min_{\|\upsilon\|=1} \upsilon^\top \E_P \left[X X^\top\  \middle \vert \  |\langle X, \theta\opt\rangle | \leq \tau\opt \right]\upsilon\\
        &= \frac{1}{p\opt} \min_{\|\upsilon\|=1} \int_{\mid X^\top\theta\opt\mid\leq \tau\opt} (\Bx^\top\upsilon)^2 p(\Bx) d\Bx \\
        &\overset{(a)}{\ge} \frac{1}{p\opt} \min_{\|\upsilon\|=1} \int_{\CB(0,\tau\opt)} (\Bx^\top\upsilon)^2 m d\Bx \\
        &\overset{(b)}{=} \frac{m (\tau\opt)^{d+2} V_d(1)}{p\opt} \min_{\|\upsilon\|=1} \int_{\CB(0,1)} (\Bu^\top\upsilon)^2  \frac{1}{V_d(1)} d\Bu \\
        &\overset{(c)}{=} \frac{m (\tau\opt)^{d+2} V_d(1)}{p\opt} \frac{S_d(1)}{V_d(1)} \frac{1}{d(d+2)} \\
        &\overset{(d)}{=} \frac{m (\tau\opt)^{d+2} V_d(1)}{p\opt (d+2)}:= \lambda_0^{\min} (\tau\opt,d).
    \end{align*}
    (a) utilizes the fact that $p(x) \ge m$ from \Cref{assum:density}, $\CB(0,\tau\opt) \subseteq \left\{\mid X^\top\theta\opt\mid\leq \tau\opt\right\}$, and  $(\Bx^\top v)^2 \ge 0$ for all $\Bx, v$.
    (b) comes from a change of variables, with $\Bx \mapsto \tau\opt \Bu$, with $d\Bx = \tau\opt d \Bu$.
    (c) utilizes \Cref{lem:min_eigenv_uniform_distribution}, and
    (d) simplifies the volume to surface area ratio.

\end{proof}

\section{Stability of error estimates}\label{app:stability}

To analyze \SCOUT, we first study the stability of $\perr$.
Since the learner does not start with knowledge of $P$ or $\theta\opt$, and by extension $\tau\opt$ we must show that, as time progresses \SCOUT's estimates of the error probabilities are not too far off.

Before analyzing the stability of the $\tau(\cdot)$ function, we present an auxiliary lemma that will be employed throughout the subsequent analysis.

\begin{lemma}\label{lem:perrShift}
For any $x>0$ and any $\theta,P$, it holds that  
\begin{equation*}
    \min\{ \tau \in [0,1]: \perr(\theta,P,\tau - x) \le \alpha\} \le \min\{ \tau \in [0,1]: \perr(\theta,P,\tau) \le \alpha\} + x.
\end{equation*}

\begin{proof}
Let 
\[
g(\tau) := \perr(\theta, P, \tau).
\]
It holds that $g$ is non-increasing on $\R$, $(\perr(\theta, P, \tau) = 1/2,\text{ for } \tau<0, \perr(\theta, P, \tau) = 0, \text{ for } \tau>1)$.
 Define
\[
\tilde{\tau} := \min\{\tau \in [0,1] : g(\tau) \le \alpha\}.
\]
We want to prove
\[
\min\{\tau \in [0,1] : g(\tau - x) \le \alpha\} 
\;\le\; \tilde{\tau} + x.
\]

Let $s := \tilde{\tau} + x$. We consider the following two cases.

First case; $s \le 1$. 
Then $s - x = \tilde{\tau}$. By definition of $\tilde{\tau}$ we have $g(\tilde{\tau}) \le \alpha$.  
Since $g$ is non-increasing, it follows that
\[
g(s - x) = g(\tilde{\tau}) \le \alpha,
\]
so $s$ belongs to the set $\{\tau \in [0,1] : g(\tau - x) \le \alpha\}$.  
Hence
\[
\min\{\tau \in [0,1] : g(\tau - x) \le \alpha\} \le s = \tilde{\tau} + x.
\]

Second case; $s > 1$.
In this case,
\[
\min\{\tau \in [0,1] : g(\tau - x) \le \alpha\} \le 1 < s = \tilde{\tau} + x.
\]

In either case, we conclude that
\[
\min\{\tau \in [0,1] : \perr(\theta, P, \tau - x) \le \alpha\}
\;\le\; \min\{\tau \in [0,1] : \perr(\theta, P, \tau) \le \alpha\} + x.
\]
\end{proof}
    
\end{lemma}

\subsection{Smoothness of $\tau\opt$ with respect to $\hat{P}_t$}

Since $P$ is unknown, \SCOUT estimates it via its empirical counterpart $\hat{P}_t$. In the following Lemma we bound the error between $\perr(\theta,\hat{P}_t,\tau)$ and $\perr(\theta,P,\tau)$.

\begin{lemma} \label{lem:perr_accurate}
    Let $\hat{P}_t$ be the empirical distribution of constructed from $\lceil t/2 \rceil$ i.i.d. samples from $P$.
    Then, for any fixed $\theta$ and $\tau$, with probability at least $1-\delta'$ over the randomness in $\hat{P}_t$:
    \begin{equation*}
        \left| \perr(\theta,\hat{P}_t,\tau)-\perr(\theta,P,\tau)\right| \le \sqrt{\frac{\log\left(\frac{\pi^2t^2}{3\delta'}\right)}{4t}}
    \end{equation*}
\end{lemma}
The proof of this result uses standard concentration bounds (Hoeffding's inequality \citet{wainwright2019high}) using the fact that for any fixed $\theta$ and $\tau$, \eqref{eq:p_err_defn} is the expectation of a $[0,1/2]$ bounded random variable.

\begin{proof}[Proof of \Cref{lem:perr_accurate}]
    First, we collect a context as a sample at every odd round, so at round $t$ it holds that $\left|\mathcal{S}_{P}^{t}\right| = \lceil t/2 \rceil \ge t/2$.
    Indexing these samples as $x_i$, we can write the empirical error $\perr(\theta,\hat{P}_t,\tau)$ as follows:
    \begin{align*}
        \perr(\theta,\hat{P}_T,\tau) - \perr(\theta,P,\tau) &= \int (1+\exp(|x^\top \theta|))^{-1} \mathds{1}\left\{|x^\top \theta| > \tau\right\} \hat{P}_t(dx) -\perr(\theta,P,\tau)\\
        &=  \frac{1}{\lceil t/2 \rceil}\sum_{i=1}^{\lceil t/2 \rceil} \left(\xi_i - \perr(\theta,P,\tau)\right), \numberthis
    \end{align*}
    where we define $\xi_i$ as the $i$-th term in this sum:
    \begin{equation*}
        \xi_i = (1+\exp(|x_i^\top \theta|))^{-1} \mathds{1}\left\{|x_i^\top \theta| > \tau\right\}.
    \end{equation*}
    As $0 \leq (1+\exp(|x_i^\top \theta|))^{-1} \leq \frac{1}{2}$, the summands $\xi_i$ are i.i.d. [0,1/2] random variables with mean $\perr(\theta,P,\tau)$, so we can apply Hoeffding's inequality \citet{wainwright2019high}:
    \begin{align*}
        \P\left(  \left|\frac{1}{\lceil t/2 \rceil}\sum_{i=1}^{\lceil t/2 \rceil} \left(\xi_i - \perr(\theta,P,\tau)\right) \right| \geq \sqrt{\frac{\log(2/\delta'')}{4t}} \right) \leq \delta''.
    \end{align*}
    By taking the union bound over all rounds $t \geq 1$ and setting $\delta'' := \frac{6\delta'}{\pi^2 t^2}$ we derive:
    \begin{align*}
        \P\left( \left|\frac{1}{\lceil t/2 \rceil}\sum_{i=1}^{\lceil t/2 \rceil} \left(\xi_i - \perr(\theta,P,\tau)\right) \right| \leq \sqrt{\frac{\log \left(\frac{\pi^2t^2}{3\delta'}\right)}{4t}} ,\forall t: t\geq 1\right) \geq 1 - \delta'.
    \end{align*}
    Here, we apply the well-known result for the Basel series: $\sum_{t=1}^{\infty}\frac{1}{t^2} = \frac{\pi^2}{6}$.
    
\end{proof}

Since we require this bound to hold over all $\theta \in \Theta$ and $\tau\in [0,1]$ and these sets are uncountable, we utilize an $\epsilon$-net analysis for both $\tau \in [0,1]$ and $\theta \in \Theta$. We detail this quantization analysis strategy in the following section.

\subsubsection{Quantization to enable union bounding} \label{sec:quantization}
We define quantized versions of $\tau$ and $\theta$, to bound the failure probability of our estimators over a countable quantized set.
We take progressively finer and finer quantizations, with our quantization accuracy scaling as $\eps_Q = t^{-2}$ ($t$ suppressed from notation).
We consider an $\eps_Q$ covering of the unit interval for $\tau$ as $\CQ_\tau:= \CN([0,1], \eps_Q)$, denoting the quantized $\tau$ value as $\tau_Q \in \CQ_\tau$ and an $\eps_Q$ cover of the $d$-dimensional unit sphere for $\theta$ as $\CQ_\theta:= \CN(\CS^{d-1},\eps_Q)$, denoting the quantized $\theta$ value as $\theta_Q \in \CQ_\theta$. We can bound the size of these covering sets as $|\CQ_\tau| \le \eps_Q^{-1}$ and $|\CQ_\theta| \le (3/\eps)^d$ \citet{vershynin2018high}.

We are now able to define the ``good'' event $G_{\perr}$ where our error probability estimates are uniformly bounded by $\zeta_t$ on our quantized sets as:
\begin{equation} \label{eq:gt_good_defn}
    G_{\perr} = \left\{\left| \perr(\theta_Q,\hat{P}_t,\tau_Q)-\perr(\theta_Q,P,\tau_Q)\right| \le \zeta_t \ :\ \forall t \in [T], \forall \theta_Q \in \CQ_\theta, \forall \tau_Q \in \CQ_\tau\right\}.
\end{equation}

The following lemma shows that $G_{\perr}$ happens with high probability.

\begin{lemma}
\label{lem:gt_goodevent}
    The good event $G_{\perr}$ satisfies $\P(G_{\perr}) \ge 1-\delta'$.
\end{lemma}

The proof of this result utilizes \Cref{lem:perr_accurate} and the union bound over the quantized sets $\CQ_\theta$ and $\CQ_\tau$.

\begin{proof}[Proof of \Cref{lem:gt_goodevent}]
    To extend ~\Cref{lem:perr_accurate} to hold simultaneously for all $\theta_Q \in \CQ_\theta$ and $\tau_Q \in \CQ_\tau$, we define an $\eps_Q$-net for each, and union bound over their cartesian product.
    By ~\Cref{lem:perr_accurate} we know that for any fixed $\theta, \tau$, and $\delta''>0$:
    \begin{equation*}
        \P \left( \left| \perr(\theta,\hat{P}_t,\tau)-\perr(\theta,P,\tau)\right| \le \sqrt{\frac{\log(\frac{\pi^2t^2}{3\delta''})}{4t}}, \forall t\geq 1 \right) \geq 1 -\delta''.
    \end{equation*}
    Let $\CQ_\theta = \mathcal{N}(\mathcal{S}^{d-1},\eps_{\theta})$ an $\eps_{Q}$-cover of the unit ball $\mathcal{S}^{d-1}$.
    By \textbf{Corollary 4.2.13} of \citet{vershynin2018high} we have that the covering number of $\mathcal{S}^{d-1}$ satisfies for any $\eps_{Q}\in(0,1]$;
    \begin{align*}
        \left(\frac{1}{\eps_{Q}}\right)^d &\leq |\CQ_\theta| \leq \left(\frac{2}{\eps_{Q}}+1\right)^d < \left(\frac{3}{\eps_{Q}}\right)^d.
    \end{align*}
    As $\tau$ lives in $[0,1]$, an $\eps$-net of the unit segment in the real line is $\{\eps, 2\eps,\dots,\lfloor\frac{1}{\eps}\rfloor\eps\}$, and so $|\CQ_\tau| \le \frac{1}{\eps_\tau}$.
    By taking a union bound over all $\tau_Q \in \CQ_\tau$ and all $\theta_Q \in \CQ_\theta$, i.e. taking $\delta'' = \delta' / (|\CQ_\theta| \cdot |\CQ_\tau|)$, we have 
    \begin{equation*}
        \P(G_{\perr}) = \P \left( \left| \perr(\theta_Q,\hat{P}_t,\tau_Q)-\perr(\theta_Q,P,\tau_Q)\right| \le \zeta_t, \forall t\geq 1, \theta_Q\in \CQ_\theta, \tau_Q \in \CQ_\tau\right) \ge 1- \delta'.
    \end{equation*}
    Recall that $\zeta_t$ is defined in \Cref{eq:zeta_t} as 
    \begin{equation*}
    \zeta_t := \sqrt{\frac{(d+1)\log \left(1/\eps_Q\right) + \log\left(\frac{\pi^2t^2}{\delta'}\right)}{4t}}.
    \end{equation*}
    This stems from the union bound with $\delta'' = \delta' / (|\CQ_\theta| \cdot |\CQ_\tau|)$,
    \begin{align*}
        \sqrt{\frac{\log(\frac{\pi^2t^2}{3\delta''})}{4t}}
        &= \sqrt{\frac{\log\left(\frac{\pi^2t^2|\CQ_\theta| \cdot |\CQ_\tau|}{3\delta'}\right)}{4t}}\\
        &\le \sqrt{\frac{\log\left(\frac{\pi^2t^2 \left(3 \eps_Q^{-d-1}\right)}{3\delta'}\right)}{4t}}\\
        &= \sqrt{\frac{(d+1)\log \left(1/\eps_Q\right) + \log\left(\frac{\pi^2t^2}{\delta'}\right)}{4t}}\\
        &= \zeta_t, \numberthis
    \end{align*}
as claimed.
    As discussed, we utilize $\eps_Q = 1/t^2$ to simplify the regret analysis in \Cref{thm:regret_upper_bound}.
    \end{proof}

Having established guarantees on the closeness of the $\perr$ estimators to their true values over our quantized set, we turn our attention to the task of understanding - for a fixed $\theta$ - the closeness of the optimal estimated threshold $\tau\opt_Q(\theta,\hat{P},\alpha)$ over the quantized set defined as 
\begin{align}
    \tau\opt_Q(\theta,\hat{P},\alpha) := \min \{\tau_Q \in \CQ_\tau : \perr(\theta,\hat{P},\tau_Q) \le \alpha \},
\end{align}
and the optimal estimated threshold $  \tau\opt(\theta,\hat{P},\alpha)$ over the entire domain of $\tau$. 
The following ``sandwich" relationship between $\tau\opt$ and $\tau\opt_Q$ holds:
\begin{align}
    \tau\opt(\theta,\hat{P},\alpha) \stackrel{(i)}{\le} \tau\opt_Q(\theta,\hat{P},\alpha) \stackrel{(ii)}{\le} \tau\opt(\theta,\hat{P},\alpha) +\eps_Q.
\end{align}
where $(i)$ holds because $ \tau\opt(\theta,\hat{P},\alpha) = \min \{\tau \in [0,1] : \perr(\theta,\hat{P},\tau) \le \alpha \}$ and $\CQ_\tau \subset [0,1]$ thus showing $ \tau\opt(\theta,\hat{P},\alpha)$ is the result of minimizing the same function $\perr$ over a larger set than in the definition of $\tau\opt_Q(\theta,\hat{P},\alpha)$. Inequality $(ii)$ holds because by definition of the covering set $\CQ_\tau$ the threshold in the cover closest to $\tau\opt(\theta,\hat{P},\alpha)$ from above (say $\tilde\tau \in \CQ_\tau$)  must satisfy $\tau\opt(\theta,\hat{P},\alpha) \leq \tilde \tau \leq \tau\opt_Q(\theta,\hat{P},\alpha)$ and $| \tau\opt(\theta,\hat{P},\alpha) - \tilde{\tau}| \leq \eps_Q$. Since $\perr(\theta, \hat{P}, \tau\opt) \leq \alpha$ and $\perr$ is monotonically decreasing in $\tau$ we see that $\perr(\theta, \hat{P}, \tilde{\tau}) \leq \alpha$ and therefore, due to the definition of $\tau\opt_Q(\theta,\hat{P},\alpha)$ as the minimum threshold in $\CQ_\tau$ satisfying $\perr \leq \alpha$, $\tilde{\tau} = \tau\opt_Q(\theta,\hat{P},\alpha)$. Combining these observations we conclude that $| \tau\opt(\theta,\hat{P},\alpha) - \tau\opt_Q(\theta,\hat{P},\alpha)| \leq \eps_Q$ and therefore the desired result.

\subsection{Stability of $\tau\opt$ with respect to $\theta$}

Having established the stability of the optimal threshold to changes in $P$, we now show that it is also stable under changes in the parameter $\theta$. To state our results, for any $\theta_Q \in \CQ_\theta \cap \CC_t$ we define an estimator $\hat{\tau}$ as (see \Cref{eq:tauhat_tauopt})
\begin{equation}
    \hat{\tau}(\theta_Q, \hat{P}_t, \alpha) 
    := \tau\opt_Q\left(\theta_Q,\hat{P}_t,\alpha - \zeta_t - 2B_t/\sqrt{ \lambda_{\min}^t}\right) + 2B_t/\sqrt{ \lambda_{\min}^t}.
\end{equation}
This section's main result is that  as long as $G_{\perr},G_\theta$ hold then, 
\begin{equation}
\hat{\tau}(\theta_Q, \hat{P}_t, \alpha)  \ge \tau\opt(\theta\opt, P, \alpha) \text{ for all }\theta_Q \in \CQ_\theta \cap \CC_t.
\end{equation}
In other words, the empirical $\hat{\tau}$ estimator evaluated at the estimated $\theta_Q$ and $\hat{P}_t$ provides us with an upper bound for the true threshold $\tau\opt$ evaluated at $\theta\opt$ and $P$. Eventually, for our regret bound, we require the reverse direction: that our estimated threshold $\hat{\tau}$ is not too much larger than $\tau\opt$, so that we do not perform too many excess tests. 
In order to show this we first establish a helper Lemma showing that our estimate $\perr(\theta,\hat{P},\tau)$ is close to $\perr(\theta\opt,\hat{P},\tau)$ when $\theta$ is close to $\theta\opt$, for any distribution $\rho$ and threshold $\tau$.

\begin{lemma}
\label{lem:perr_stability_tau}
    For all $\theta,\theta' \in \Theta$, $ \tau \ge \frac{\|\theta - \theta'\|_{V_t}}{\sqrt{\lambda_{\min}^t}}$, and density $\rho(x)$ on $\CX$:
    \begin{align*}
          \perr(\theta,\rho,\tau) \le \ \perr \left(\theta',\rho,\tau- \frac{\|\theta - \theta'\|_{V_t}}{\sqrt{\lambda_{\min}^t}}\right) + \frac{\|\theta - \theta'\|_{V_t}}{\sqrt{\lambda_{\min}^t}}.
    \end{align*}
\end{lemma}

To prove this we leverage algebraic properties of $\perr(\theta,\rho,\tau)$ and the Hölder inequality, a standard technique in Linear Bandits (see \citet{lattimore2020bandit}, Part V).
\begin{proof}
    Here, we use $x$ as a dummy variable for integration:
    \begin{align*}
        \perr(\theta,\rho,\tau) &= \int (1+\exp(|x^\top \theta|))^{-1} \mathds{1}\left\{|x^\top \theta| > \tau\right\} \rho(dx) \\
        &= \int (1+\exp(|x^\top \theta' + x^\top (\theta - \theta')|))^{-1} \mathds{1}\left\{|x^\top \theta' + x^\top (\theta - \theta')| > \tau\right\} \rho(dx) \\
        &\le \int (1+\exp(|x^\top \theta' | -  |x^\top (\theta - \theta')|))^{-1} \mathds{1}\left\{|x^\top \theta'| > \tau - |x^\top (\theta - \theta')| \right\} \rho(dx) \\
        &\le \int \left((1+\exp(|x^\top \theta' |))^{-1} + |x^\top (\theta - \theta')|\right) \mathds{1}\left\{|x^\top \theta'| > \tau - |x^\top (\theta - \theta')| \right\}\rho(dx) \\
        &\le  \max_{x' \in \CX} \int \left((1+\exp(|x^\top \theta' |))^{-1} + |x'^\top (\theta - \theta')|\right) \mathds{1}\left\{|x^\top \theta'| > \tau - |x'^\top (\theta - \theta')| \right\}\rho(dx) \\
        &= \max_{x' \in \CX}\perr(\theta',\rho,\tau - |x'^\top (\theta - \theta')|) + \int |x^\top (\theta - \theta')| \mathds{1}\left\{|x^\top \theta'| > \tau - |x'^\top (\theta - \theta')| \right\} \rho(dx) \\
        &\le \max_{x' \in \CX} \perr(\theta',\rho,\tau- \|\theta - \theta'\|_{V_t}\| x' \|_{V_t^{-1}}) + \| \theta - \theta' \|_{V_t} \| x' \|_{V_t^{-1}} \P_{\rho}\left(|x^\top \theta'| > \tau - |x^\top (\theta - \theta')| \right) \\
        &\le \max_{x' \in \CX} \perr(\theta',\rho,\tau- \|\theta - \theta'\|_{V_t}\| x' \|_{V_t^{-1}}) + \| \theta - \theta' \|_{V_t} \| x' \|_{V_t^{-1}} \\
        &= \perr(\theta',\rho,\tau- \frac{\|\theta - \theta'\|_{V_t}}{\sqrt{\lambda_{\min}^t}}) + \frac{\|\theta - \theta'\|_{V_t}}{\sqrt{\lambda_{\min}^t}}
    \end{align*}
    The first inequality follows from the triangle inequality, and the second inequality follows from the fact that $1/(1+\exp(z))$ is 1/4-Lipschitz (coarsely upper bounded as 1).
    The third bounds by looking at the worst case context $x'$.
    The fourth inequality utilizes Hölder's inequality, on the worst case context $x'$, and that $\perr$ is monotone in $\tau$.
    The second to last inequality follows from the fact that a probability is always less than or equal to 1.
    Finally, we apply the following bound for any $x' \in \CX$; $\| x' \|_{V_t^{-1}} \le \frac{1}{\sqrt{\lambda_{\min}^t}}$, where we have implicitly used that $\norm{x'} \le 1, \forall x' \in \CX$.

\end{proof}

\Cref{lem:perr_stability_tau} indicates that as our ability to estimate $\theta$ improves, so will our error probability estimates.
Now, conditioning on the good event $G_{\perr}$,  we show that $\tau\opt_Q(\theta_Q,\hat{P}_t,\alpha)$ is close to $\tau\opt$ when $\theta_Q$ is close to $\theta\opt$.

\begin{restatable}{lemma}{tauQUpperNLowerBound}\label{lem:tau_q_upperNlower_bound}
    Conditioning on $G_{\perr}$, for any $\theta_Q \in \CQ_\theta \cap \CC_t, \theta \in \CC_t$ such that $ \frac{\|\theta_Q - \theta\|_{V_t}}{\sqrt{\lambda_{\min}^t}} \le \tau\opt(\theta, P,\alpha - \zeta_t - \frac{\|\theta_Q - \theta\opt\|_{V_t}}{\sqrt{\lambda_{\min}^t}})$ it is true that:
    \begin{align}
    \tau\opt_Q(\theta_Q,\hat{P}_t,\alpha) &\leq \tau\opt\left(\theta,P,\alpha - \zeta_t - \frac{\|\theta_Q - \theta\|_{V_t}}{\sqrt{\lambda_{\min}^t}}\right) + \frac{\|\theta_Q - \theta\|_{V_t}}{\sqrt{\lambda_{\min}^t}} +\eps_Q, \notag\\
    \tau\opt_Q(\theta_Q,\hat{P}_t,\alpha) &\ge \tau\opt\left(\theta,P,\alpha + \zeta_t + \frac{\|\theta_Q - \theta\|_{V_t}}{\sqrt{\lambda_{\min}^t}}\right) - \frac{\|\theta_Q - \theta\|_{V_t}}{\sqrt{\lambda_{\min}^t}} \label{eq:lower_bound_helper_taustar_q_taustar}
    \end{align}
\end{restatable}

The proof of the above lemma relies on \Cref{eq:gt_good_defn} to relate $\tau_Q\opt(\cdot,\hat{P}_t,\cdot)$ to $\tau_Q\opt(\cdot,P,\cdot)$ and \Cref{lem:perr_stability_tau} to connect $\tau_Q\opt(\theta_Q,P,\cdot)$ to $\tau_Q\opt(\theta,P,\cdot)$.

\begin{proof}
    Conditioning on the good event $G_{\perr}$, we have that
    \begin{align*}
        \tau\opt_Q(\theta_Q,\hat{P}_t,\alpha) &= \min \{\tau_Q \in \CQ_\tau : \perr(\theta_Q,\hat{P}_t,\tau_Q) \le \alpha \} \\
        &\overset{(a)}{\le} \min \left\{\tau_Q \in \CQ_\tau : \perr(\theta_Q,P,\tau_Q) \le \alpha - \zeta_t \right\} \\
        &\overset{(b)}{\le} \min \left\{\tau_Q \in \CQ_\tau: \perr(\theta,P,\tau_Q) \le \alpha - \zeta_t - \frac{\|\theta_Q - \theta\|_{V_t}}{\sqrt{\lambda_{\min}^t}}\right\} + \frac{\|\theta_Q - \theta\|_{V_t}}{\sqrt{\lambda_{\min}^t}} \\
        &\le \min \left\{\tau \in [0,1]: \perr(\theta,P,\tau) \le \alpha - \zeta_t - \frac{\|\theta_Q - \theta\|_{V_t}}{\sqrt{\lambda_{\min}^t}}\right\}+ \frac{\|\theta_Q - \theta\|_{V_t}}{\sqrt{\lambda_{\min}^t}} + \eps_Q\\
        &= \tau\opt\left(\theta,P,\alpha - \zeta_t - \frac{\|\theta_Q - \theta\|_{V_t}}{\sqrt{\lambda_{\min}^t}}\right) + \frac{\|\theta_Q - \theta\|_{V_t}}{\sqrt{\lambda_{\min}^t}} +\eps_Q \numberthis
    \end{align*}
    Where inequality (a) follows from conditioning on the good event $G_{\perr}$, and (b) follows from \Cref{lem:perr_stability_tau}.

    The lower bound for $\tau\opt_Q(\theta_Q,\hat{P}_t,\alpha)$ follows analogously:
    \begin{align*}
        \tau\opt_Q(\theta_Q,\hat{P}_t,\alpha) &= \min \{\tau_Q \in \CQ_\tau : \perr(\theta_Q,\hat{P}_t,\tau_Q) \le \alpha \} \\
        &\overset{(a)}{\ge} \min \{\tau_Q \in\CQ_\tau : \perr(\theta_Q,P,\tau_Q) \le \alpha +\zeta_t \} \\
        &= \tau_Q\opt(\theta_Q,P,\alpha+\zeta_t)\\
        &\ge \tau\opt(\theta_Q,P,\alpha+\zeta_t),
    \end{align*}
    where $(a)$ follows by the good event $G_{\perr}$, and the final inequality from the looseness of quantization.
    Now, we will lower bound $\tau\opt(\theta_Q,P,\alpha)$ in terms of $\tau\opt$ using \Cref{lem:perr_stability_tau}.
    \begin{align*}
        \tau\opt(\theta_Q,P,\alpha) &= \min\{ \tau \in [0,1]: \perr(\theta_Q,P,\tau) \le \alpha \} \\
        &\overset{(a)}{\ge} \min \left\{ \tau \in [0,1]: \perr\left(\theta,P,\tau + \frac{\|\theta_Q - \theta\|_{V_t}}{\sqrt{\lambda_{\min}^t}} \right) - \frac{\|\theta_Q - \theta\|_{V_t}}{\sqrt{\lambda_{\min}^t}} \le \alpha \right\} \\
        &\ge \min\left\{ \tau \in [0,1]: \perr(\theta,P,\tau ) \le \alpha + \frac{\|\theta_Q - \theta\|_{V_t}}{\sqrt{\lambda_{\min}^t}} \right\} - \frac{\|\theta_Q - \theta\|_{V_t}}{\sqrt{\lambda_{\min}^t}} \\
        &=\tau\opt\left(\theta,P,\alpha + \frac{\|\theta_Q - \theta\|_{V_t}}{\sqrt{\lambda_{\min}^t}}\right) - \frac{\|\theta_Q - \theta\|_{V_t}}{\sqrt{\lambda_{\min}^t}},
    \end{align*}
     where $(a)$ follows from \Cref{lem:perr_stability_tau}.
     
\end{proof}

    Putting this all together we have that on $G_{\perr},G_\theta$, evaluating at $\theta = \theta\opt$,
    \begin{align*}
    \hat{\tau}(\theta_Q,\hat{P}_t,\alpha)
    &= \tau\opt_Q\left(\theta_Q,\hat{P}_t,\alpha - \zeta_t - 2B_t/\sqrt{\lambda_{\min}^t}\right) + 2B_t/\sqrt{\lambda_{\min}^t}\\
    &\overset{(a)}{\ge} \tau\opt(\theta\opt,P,\alpha).
    \end{align*}
    where 
    (a) leverages \Cref{lem:tau_q_upperNlower_bound}. This uses the fact that when $G_\theta$ and $G_{\perr}$ hold,
\begin{equation*}
     \|\theta_Q - \theta\opt\|_{V_t}  \le 2B_t.
\end{equation*}

\subsection{Smoothness of $\tau\opt$ with respect to $\alpha$}

The last property we will need for our analysis is that $\tau\opt$ does not vary too quickly with respect to $\alpha$.
We show that for small $\gamma$, $\tau\opt(\theta\opt,P,\alpha+\gamma)$ is not much smaller than $\tau\opt$.
Note that while $\perr$ is continuous with respect to $\tau$ when evaluated at $P$ the true distribution, it is discontinuous when evaluated at $\hat{P}$ because this is an empirical distribution.

However, by \Cref{assum:density}, the true distribution of contexts is upper and lower bounded by constants and so $\perr$, which integrates the distribution, will change at an upper and lower bounded rate. We leverage these properties to prove the following stability result.

\begin{restatable}{lemma}{lemTauoptStabilityAlpha}
\label{lem:tauopt_stability_alpha}
    Under \Cref{assum:nontrivial,assum:density},
    \begin{equation}
        \tau\opt(\theta\opt,P,\alpha-\gamma) \le  \tau\opt(\theta\opt,P,\alpha) + \frac{(1+e)\gamma}{ m \cdot V_d(1)\useconstant{CL}(\tau\opt)},
    \end{equation}
     for $0<\gamma < \min \{\frac{\tau\opt\cdot m \cdot V_d(1) \cdot \useconstant{CL}(\tau\opt)}{2(1+e)},\frac{m \cdot V_d(1) \cdot f(\frac{1+\tau\opt}{2}) \cdot (1-\tau\opt)}{2(1+e)} \}$,
     where $f(\cdot)$ is the PDF of $Z \sim Beta(\frac{1}{2},\frac{d+1}{2})$. 
\end{restatable}

The proof proceeds as follows.
First, we study the stability of $\tau\opt$ when the contexts follow the uniform distribution on the unit ball, characterizing the mass of contexts satisfying $|\abs{\dotp{X,\theta\opt}}\le \tau\opt$ (\Cref{lem:spherical_segment}).
Then we use \Cref{assum:density} to derive bounds for the unknown distribution $P$ (\Cref{lem:spherical_seg_prob_bound}).
Finally, we leverage these upper and lower bounds to derive the stability of $\tau\opt$ with respect to $\alpha$ (\Cref{lem:tauopt_stability_alpha}).

\begin{lemma}\label{lem:spherical_segment}
    For any $0\le \tau \le 1$ intersection of $\{\Bx : \norm{\Bx} \le 1\}$ with $\{\Bx : \abs{\dotp{\Bx}{\theta\opt}}\le \tau\}$ is a spherical segment (see \Cref{fig:spherical_seg}) with volume equal to 
    $$
    \int_{\norm{\Bx} \le 1}\indctr{\abs{\dotp{\Bx}{\theta\opt}} \le \tau} d\Bx = V_d \cdot I_{\tau^2}\left(\frac{1}{2},\frac{d+1}{2}\right),
    $$
    where $V_d$ is the volume of the \textit{d-dimensional} unit ball, and $I_x(a,b) = \frac{\int_{0}^{x}t^{a-1}(1-t)^{b-1}dt}{B(a,b)}$ is the regularized Beta function ~\citet{egorova2023computation}, that is the cumulative distribution function of the Beta distribution.
\end{lemma}

\begin{figure}[h]
    \centering
    \begin{subfigure}[b]{0.4\textwidth}
        \centering
        \includegraphics[width=\textwidth]{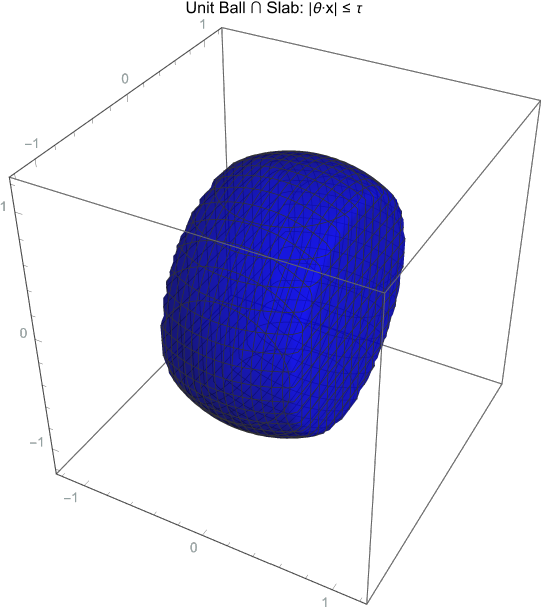}
    \end{subfigure}
    \hfill
    \begin{subfigure}[b]{0.4\textwidth}
        \centering
        \includegraphics[width=\textwidth]{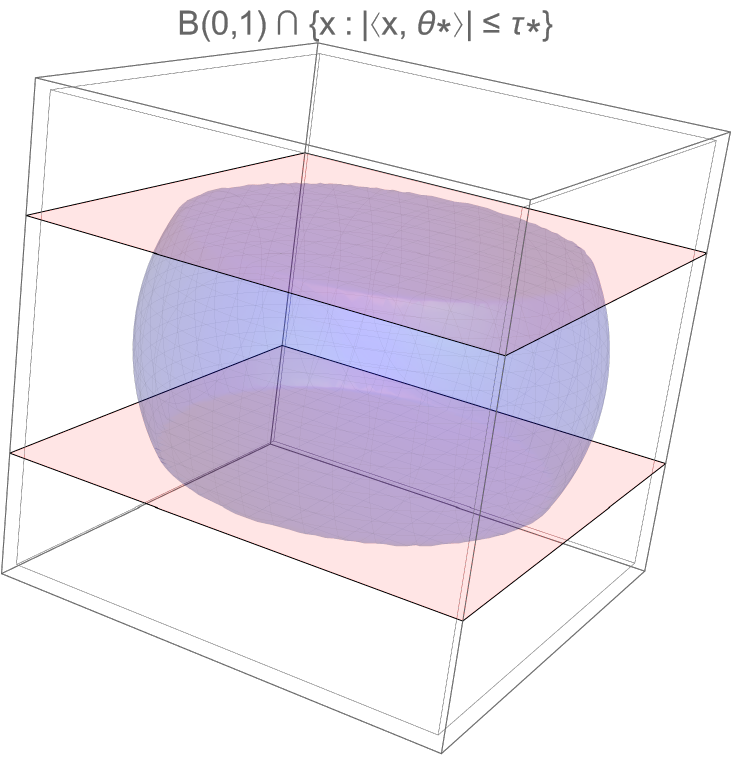}
    \end{subfigure}
    \caption{$\CB(0,1) \cap \{x:\abs{\dotp{x}{\theta\opt}}\le \tau\opt\}$}
    \label{fig:spherical_seg}
\end{figure}

Before proving \Cref{lem:spherical_segment} we will first prove an auxiliary lemma that allows us to work with a more convenient vector in the surface of the unit ball instead of $\theta\opt$.
For more details about orthogonal transformations we refer the reader to \citet{horn2012matrix}.

\begin{lemma}\label{lem:orthogonalTransformation}
Let $\theta,\theta' \in \mathbb{R}^d$ be vectors on the unit sphere, i.e.\ $\|\theta\|=\|\theta'\|=1$. 
Then there exists an orthogonal matrix $S \in \mathbb{R}^{d \times d}$ such that 
\[
S \theta' = \theta.
\]
\end{lemma}

\begin{proof}
If $\theta'=\theta$, the claim holds with $S=I$.  

Otherwise, set
\[
u := \frac{\theta'-\theta}{\|\theta'-\theta\|},
\qquad
S := I - 2uu^T.
\]
The matrix $S$ is called a \emph{Householder reflection}. It satisfies $S^\top S = I$, so it is orthogonal.  

We compute
\[
u^T \theta' 
= \frac{(\theta'-\theta)^T \theta'}{\|\theta'-\theta\|}
= \frac{1 - \theta^T \theta'}{\|\theta'-\theta\|}
= \frac{\|\theta'-\theta\|}{2},
\]
since $\|\theta\|=\|\theta'\|=1$ implies 
\[
\|\theta'-\theta\|^2 = \|\theta'\|^2 + \|\theta\|^2 - 2 \theta'^T \theta = 2(1-\theta'^T \theta).
\]

Hence
\[
S \theta' = \theta' - 2u(u^T \theta') 
= \theta' - u \|\theta'-\theta\|
= \theta' - (\theta'-\theta) 
= \theta.
\]
Thus $S$ is an orthogonal matrix such that $S \theta' = \theta$.
\end{proof}

We will apply now this lemma for $\theta' = \theta\opt$ and $\theta = (0,0,\dots,1)$ to compute the area of integration at \Cref{lem:spherical_segment}.

\begin{proof}[Proof of \Cref{lem:spherical_segment}]
    A similar proof, but for spherical caps, can be found in \citet{li2010concise}.
    We follow similar steps to the didactic work of \citet{jorgensen2014volumes}.

    For $\theta\opt \in \R^d$ with $\|\theta\opt\|=1$,
    we have to integrate over all $\Bx \in \R^d$ such that 
    \begin{equation}\label{eq:limitsInt}
        \{\Bx^\top \Bx \le 1\} \cap \{\abs{\Bx^\top\theta\opt} \le \tau\}.
    \end{equation}

    We apply \Cref{lem:orthogonalTransformation} for $\theta' = \theta\opt$ and $\theta = (0,0,\dots,1)^\top$.
    Then, let $S$ be the orthogonal matrix such that 
    $$S\theta\opt = (0,0,\dots,1)^\top.$$
    We can use then \Cref{eq:limitsInt} to change the limits of integration;
    \begin{align*}
        \{\Bx^\top \Bx \le 1\} \cap \{\abs{\Bx^\top\theta\opt} \le \tau\}
        &= \{\Bx^\top S^\top S\Bx \le 1\} \cap \{\abs{\Bx^\top S^\top S\theta\opt} \le \tau \}\\
        &= \{(S\Bx)^\top (S\Bx) \le 1\} \cap \{\abs{(S\Bx)^\top (0,0,\dots,1) } \le \tau\}
    \end{align*}

    Let $\tilde{\Bx} = S\Bx$ then the new integration domain is 
    \begin{align*}
        \{\sum_{i=1}^{d-1} \tilde{x}_i^2 \le 1 - {\tilde{x}_d}^2\} \cap \{\abs{\tilde{x}_d} \le \tau\}.
    \end{align*}

    We define the volume of interest as
    \begin{equation}\label{eq:volumeOfInterest}
        V_I = \int_{\norm{\Bx} \le 1}\indctr{\abs{\dotp{\Bx}{\theta\opt}} \le \tau} d\Bx.
    \end{equation}
    By integrating first with respect to the first $n-1$ dimensions and then to the last one we get
    \begin{equation*}
        V_I = \int_{-\tau}^{\tau} \left( \int_{\{\Bx\in \R^{d-1}:\norm{\Bx}\le \sqrt{1- x_n^2}\}} dx_1 \dots dx_{d-1}\right) dx_d.
    \end{equation*}
    Now, we can use that the volume of a sphere with radius $r$ in $d$ dimensions is equal to \citep{jorgensen2014volumes}
    $$V_d(r) = \frac{r^d\pi^{(d/2)}}{\Gamma(\frac{d}{2}+1)},$$
    and calculate the inner integral as
    \begin{align*}
        \int_{-\tau}^{\tau} \left( \int_{\{\Bx\in \R^{d-1}:\norm{\Bx}\le 1- x_n^2\}} dx_1 \dots dx_{d-1}\right) dx_d &= \frac{\pi^{\frac{d-1}{2}}}{\Gamma(\frac{d-1}{2}+1)}\int_{-\tau}^{\tau} (1 - x_d^2)^{\frac{d-1}{2}} dx_d.
    \end{align*}
    We use the fact that the function $1-x^2$ is even and the previous expression becomes
    \begin{equation*}
        V_I = 2\frac{\pi^{\frac{d-1}{2}}}{\Gamma(\frac{d-1}{2}+1)}\int_{0}^{\tau} (1 - x_d^2)^{\frac{d-1}{2}} dx_d.
    \end{equation*}
    We now make the change of variables, $x_d := \sqrt{t}$ and $dx_d = \frac{1}{2}t^{-\frac{1}{2}}dt$.
    The new limits of integration are; when $x_d = 0$ then $t=0$ and when $x_d =\tau$, $t=\tau^2$.
    \begin{align*}
        V_I &= \frac{\pi^{\frac{d-1}{2}}}{\Gamma(\frac{d-1}{2}+1)}\int_{0}^{\tau^2} (1 - t)^{\frac{d-1}{2}} \cdot t^{-\frac{1}{2}}dt \\
        &= \frac{\pi^{\frac{d-1}{2}}\Gamma(1/2)\Gamma(d/2+1)}{\Gamma(\frac{d-1}{2}+1)\Gamma(1/2)\Gamma(d/2+1)}\int_{0}^{\tau^2} (1 - t)^{\frac{d-1}{2}} \cdot t^{-\frac{1}{2}}dt \\
        &= \frac{\pi^{\frac{d}{2}}}{\Gamma(d/2+1)} \cdot \frac{\Gamma(d/2+1)}{\Gamma(1/2)\Gamma(\frac{d-1}{2}+1)} \cdot \int_{0}^{\tau^2} (1 - t)^{\frac{d-1}{2}} \cdot t^{-\frac{1}{2}}dt,
    \end{align*}
    where we used that $\Gamma(1/2) = \sqrt{\pi}$.
    We further use the definition of the Beta function $B(\alpha,\beta) = \frac{\Gamma(\alpha)\Gamma(\beta)}{\Gamma(\alpha+\beta)}$ and that $V_d(1) = \frac{\pi^{\frac{d}{2}}}{\Gamma(d/2+1)}$ (see \citet{jorgensen2014volumes}).
    \begin{align*}
        V_I &= V_d(1) \cdot \frac{\int_{0}^{\tau^2} (1 - t)^{\frac{d-1}{2}} \cdot t^{-\frac{1}{2}}dt}{B(\frac{1}{2},\frac{d+1}{2})} \\
        &= V_d(1) I_{\tau^2}\left(\frac{1}{2},\frac{d+1}{2}\right).
    \end{align*}
\end{proof}

We are interested in studying the stability of the previous quantity when we evaluate at $\tau - \lambda$, for $0<\lambda < \tau$ instead of at $\tau$.
This is the difference between the CDF of the Beta distribution evaluated at $(\tau -\lambda)^2$ and at $\tau^2$, i.e. $I_{\tau^2}\left(\frac{1}{2},\frac{d+1}{2}\right) - I_{(\tau-\lambda)^2}\left(\frac{1}{2},\frac{d+1}{2}\right)$.

We will show that for the given parameters $\alpha, \beta$ for $Z \sim Beta(\alpha, \beta)$, the CDF $F(z) = \P(Z \le z)$ is a concave function.
Then, we will bound the difference $F(1-\tau^2) - F(1-(\tau +\lambda)^2)$ by using standard arguments for increasing, concave functions that lie in $[0,1]$.
These can be summarized in the following lemmata.

\begin{lemma} \label{lem:beta_cdf}
    For $Z \sim Beta(\frac{1}{2},\frac{d+1}{2})$, $d \ge 1$, the CDF of $Z$ is non-decreasing and concave over its support.
\end{lemma}

\begin{proof}
    Let $F(z) = \P(Z \le z)$. Then, $F'(z) =: f(z) > 0$ for all $z>0$, as $f$ is a density, and so $F$ is non-decreasing.
    We calculate the derivative of the density function by differentiating its logarithm.

    \begin{align*}
        f(z) & = z^{-\frac{1}{2}}(1-z)^{\frac{d-1}{2}}, \\
        \log\left(f(z)\right) &= -\frac{1}{2} \log z +  \frac{d-1}{2} \log(1-z), \\
        \log\left(f(z)\right)' &= -\frac{1}{2z} - \frac{d-1}{2(1-z)}<0.
    \end{align*}
    Then, for all $0<z<1$, $f'(z) = \log\left(f(z)\right)' f(z) < 0$, and so $F$ is concave.
    \Cref{fig:cdfBeta} illustrates the CDF across various values of parameter $d>1$.

    \begin{figure}
        \centering
        \includegraphics[width=0.5\linewidth]{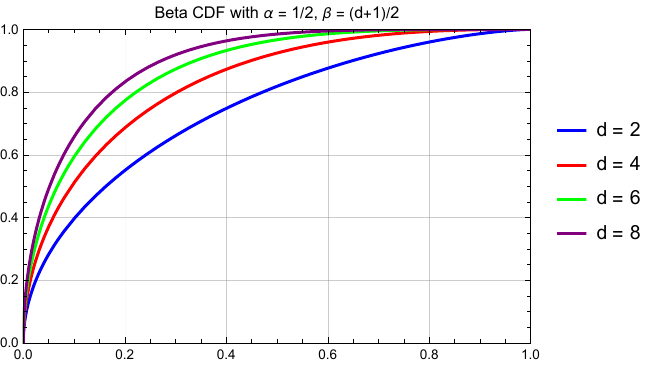}
        \caption{The CDF of $Beta(\frac{d+1}{2},\frac{1}{2})$ for various values of $d$.}
        \label{fig:cdfBeta}
    \end{figure}
\end{proof}

To continue in our analysis, we will need to show that the $\tau$ for which we are evaluating stability is bounded away from one. Concretely, we wish to evaluate at $\tau < (1+\tau\opt)/2 < 1$ for stability purposes.
\begin{lemma}\label{lem:tauLessOne}
    Under \Cref{assum:density} $\tau\opt < 1$.
\end{lemma}
\begin{proof}
$\|\theta\opt\| = 1$, and $\|X\|_2 \le 1$ a.s., and so $\abs{\dotp{X}{\theta\opt}} \le 1$.

Recall that $\perr$ is defined as,
    \begin{equation*}
         \perr(\theta,P,\tau) = \int (1+\exp(|x^\top \theta|))^{-1} \mathds{1}\left\{|x^\top \theta| > \tau\right\} P(dx).
    \end{equation*}

We will show that $\perr(\cdot)$ is continuous in $\tau$, that is for any $\tau_0 \in [0,1]$
\begin{align*}
    \lim_{\tau \rightarrow \tau_0} \perr(\theta, P, \tau) &=  \perr(\theta, P, \tau_0).
\end{align*}

We will apply \Cref{lem:orthogonalTransformation} to compute the integral 

$$\int (1+\exp(|x^\top \theta|))^{-1} \mathds{1}\left\{|x^\top \theta| > \tau\right\} P(dx),$$

for $\theta, [0,0,\dots,1]^\top$.
Let $S$ the orthogonal matrix such that $\theta = S\cdot[0,0,\dots,1]^\top$.

For any $x$ let $u = S\cdot x$, and $u_i$ its i-th coordinate, we can write $x^\top \theta$ as 

\begin{align*}
    x^\top \theta &= x^\top S^\top S \theta \\
    &= (Sx)^\top [0,0,\dots,1]^\top \\
    &= u_d.
\end{align*}

The inequality $|x^\top \theta| > \tau$ can be written as

\begin{align*}
    |x^\top \theta| &> \tau \\
    |x^\top S^\top S \theta| &> \tau \\
    |u_d| > \tau.
\end{align*}

By the change of variable $u \mapsto Sx$, we have that

\begin{align*}
    \norm{Sx}_2 = \norm{x}_2 \le 1 &\iff \norm{u}_2 \le 1 \\
    dx &= | \det S^\top | du = du.
\end{align*}

Then, we have that 

\begin{align*}
    \int_{\norm{x}_2 \le 1} (1+\exp(|x^\top \theta|))^{-1} \mathds{1}\left\{|x^\top \theta| > \tau\right\} P(dx) &= \int_{\norm{u}_2 \le 1} (1+\exp(|u_d|))^{-1} \mathds{1}\left\{|u_d| > \tau\right\} P(S^\top u)du
\end{align*}

Now, to prove continuity we fix a sequence $\tau_n \rightarrow \tau$ for an arbitrary value of $\tau$. We must prove now 

\begin{equation*}
    \lim_{n \rightarrow \infty} \int_{\norm{u}_2 \le 1} (1+\exp(|u_d|))^{-1} \mathds{1}\left\{|u_d| > \tau_n\right\} P(S^\top u)du = \int_{\norm{u}_2 \le 1} (1+\exp(|u_d|))^{-1} \mathds{1}\left\{|u_d| > \tau\right\} P(S^\top u)du.
\end{equation*}

As $\tau_n \rightarrow \tau$ we know that for every $\eps >0$ there exists $N(\eps) \in \mathbb{N}$ such that for all $n \ge N(\eps)$ it holds that $\abs{\tau_n -\tau} < \eps$.
We will use the dominated convergence theorem (Theorem 2.24 \citet{folland1999real}). Let 
$$g_n(u) :=  (1+\exp(u_d))^{-1} \mathds{1}\left\{|u_d| > \tau_n\right\} P(S^\top u),$$  $$g(u):= (1+\exp(u_d))^{-1} \mathds{1}\left\{|u_d| > \tau\right\} P(S^\top u).$$
We will prove first that $g_n(u) \rightarrow g(u)$ almost everywhere.
Equivalently we can prove that $\mathds{1}\left\{|u_d| > \tau_n\right\} \rightarrow \mathds{1}\left\{|u_d| > \tau\right\}$ almost everywhere.
We will consider three cases for the range of values of $u_d$.

Consider three cases for the fixed real number \(|u_d|\).

\textbf{Case 1:} \(|u_d|>\tau\).  
Let \(\varepsilon=\tfrac12(\abs{u_d} - \tau)>0\). For all \(n\ge N(\eps)\) such that \(|\tau_n-\tau|<\varepsilon\) we have
\[
\tau_n \le \tau+\varepsilon < \tau + \tfrac12(\abs{u_d} - \tau) = \tfrac12(\tau+|u_d|)<|u_d|,
\]
so \(|u_d|>\tau_n\) and therefore \(g_n(u)=1\). Hence \(g_n(u) = 1=g(u), \forall n\ge N(\eps)\).

\textbf{Case 2:} \(|u_d|<\tau\).  
Let \(\varepsilon=\tfrac12(\tau-|u_d|)>0\). For all sufficiently large \(n\ge N(\eps)\) with \(|\tau_n-\tau|<\varepsilon\) we get
\[
\tau_n \ge \tau-\varepsilon > \tau-\tfrac12(\tau-|u_d|) = \tfrac12(\tau+|u_d|) > |u_d|,
\]
so \(|u_d|\le \tau_n\) and \(g_n(u)=0\). Hence \(g_n(u)= 0=g(u)\).

\textbf{Case 3:} \(|u_d| = \tau\).  
For the third case an alternating sequence would not converge but it does not matter as the set $\{u \in \CB(0,1) : \abs{u_d} = \tau\}$ has measure zero under P.

As a result now we proved that $g_n(u) \rightarrow g(u)$ almost everywhere.
Moreover, $0\le g_n(u) \le \frac{1}{2}$ for all $n \in \mathbb{N}$ for every $u$.
By applying the dominated convergence theorem, we get that $\perr(\cdot)$ is continuous at $\tau$;

\begin{align*}
     \lim_{n \rightarrow \infty}\int_{\norm{u}_2 \le 1} g_n(u) du &= \int_{\norm{u}_2 \le 1} g(u) du \\
     \lim_{n \rightarrow \infty} \int_{\norm{u}_2 \le 1} (1+\exp(|u_d|))^{-1} \mathds{1}\left\{|u_d| > \tau_n\right\} P(S^\top u)du &= \int_{\norm{u}_2 \le 1} (1+\exp(|u_d|))^{-1} \mathds{1}\left\{|u_d| > \tau\right\} P(S^\top u)du.
\end{align*}

Now, we will show that $\perr(\cdot)$ is strictly decreasing in $\tau$.
Let $P_d(S^\top u)$ the marginal distribution at the d-th coordinate.
Then for $\tau_1 < \tau_2$

\begin{align*}
    \perr(\theta, P,\tau_1) - \perr(\theta, P,\tau_2) &= \int_{-1}^{1} (1+\exp(|u_d|))^{-1} \left( \mathds{1}\left\{|u_d| > \tau_1\right\} - \mathds{1}\left\{|u_d| > \tau_2\right\} \right)P_d(S^\top u)du_d \\
    &= \int_{-\tau_2}^{-\tau_1} (1+\exp(|u_d|))^{-1} P_d(S^\top u)du_d + \int_{\tau_1}^{\tau_2} (1+\exp(|u_d|))^{-1} P_d(S^\top u)du_d\\
    &>0
\end{align*}
where we have strict inequality as $P_d(S^\top u) > 0$ by \Cref{assum:density}.

This concludes the proof that $\perr(\cdot)$ is strictly decreasing as a function of $\tau$.

Since $x^\top \theta \leq 1$ for all $x$, as $\norm{x}_2\le1, \norm{\theta}_2$=1. It follows that $\text{indicator}(|x^\top \theta| > \tau)  = 0$ and therefore that $\perr(\theta\opt,P,1) = 0$.

Finally, since $\perr(\theta\opt,P,1) = 0 < \alpha$, and $\perr(\theta\opt,P,\tau)$ is a strictly monotone (decreasing) and continuous function of $\tau$, we get that $\tau\opt<1$.

\end{proof}

Now, analyzing the Beta CDF by using concavity, monotonicity, and the fact that $F(0) = 0$ and $F(1) = 1$ (\Cref{lem:beta_cdf}) we will derive upper and lower bounds for the difference $F(\tau^2) - F\left((\tau-\lambda)^2\right)$.
As in our algorithm we design a sequence of threshold converging to the real one, one can imagine $\lambda$ as part of a sequence $\{\lambda_t\}$ that converges to zero.

\newconstant{CL}
\newconstant{CR}

\begin{lemma}\label{lem:betaCDF_stability}
    Under \Cref{assum:nontrivial}, for all $0< \lambda < \frac{\tau\opt}{2} <\tau < \frac{1+\tau\opt}{2}<1$ there exist functions $\useconstant{CL}(\tau\opt),\useconstant{CR}(\tau\opt,\lambda)$ such that it holds that; 
    \begin{align*}
        \useconstant{CL}(\tau\opt) \cdot \lambda &\le F\left(\tau^2\right) - F\left((\tau-\lambda)^2\right) \le \useconstant{CR}(\tau\opt,\lambda) \cdot \lambda,
    \end{align*}
    where $F(\cdot)$ denotes the CDF of the random variable $Z \sim Beta(\alpha, \beta)$ and $f$ its density.
    $\useconstant{CL}(\tau\opt),\useconstant{CR}(\tau\opt,\lambda)$ are defined as follows;
    \begin{align*}
        \useconstant{CL}(\tau\opt) &:= \frac{\tau\opt}{2}\left( 1 -F\left(\frac{(1+\tau\opt)^2}{4}\right) \right),\\
        \useconstant{CR}(\tau\opt,\lambda) &:= 2\frac{1}{(\frac{\tau\opt}{2}-\lambda)^2}.
    \end{align*}

\end{lemma}

\begin{proof}
    We apply the mean value theorem in the intervals $[0,(\tau-\lambda)^2], [(\tau-\lambda)^2, \tau^2], [\tau^2,1]$.

    By applying the mean value theorem to these intervals there exists $\xi_1 \in (0,(\tau-\lambda)^2), \xi_2 \in ((\tau-\lambda)^2, \tau^2), \xi_3 \in (\tau^2,1)$ such that 
    \begin{align*}
        \frac{F\left((\tau-\lambda)^2\right)-F(0)}{(\tau-\lambda)^2} &= F'(\xi_1)= f(\xi_1),\\
        \frac{F\left(\tau^2\right) - F\left((\tau-\lambda)^2\right)}{\tau^2-(\tau-\lambda)^2} &= F'(\xi_2)=f(\xi_2),\\
        \frac{F(1) - F\left(\tau^2\right)}{1-\tau^2} &= F'(\xi_3)=f(\xi_3).
    \end{align*}

    As $F(0) = 0, F(1) = 1$, $F'(x) = f(x)$ and $f'(x) < 0$ it holds that 
    \begin{equation*}
        f(\xi_1) \ge f(\xi_2) \ge f(\xi_3). 
    \end{equation*}
    We replace the values of $f(\xi_1), f(\xi_2), f(\xi_3)$;
    \begin{align}\label{eq:MVTbounds}
        \frac{F(1) - F\left(\tau^2\right)}{1-\tau^2}&\le \frac{F\left(\tau^2\right) - F\left((\tau-\lambda)^2\right)}{\tau^2-(\tau-\lambda)^2} \le  \frac{F\left((\tau-\lambda)^2\right)-F(0)}{(\tau-\lambda)^2}.
    \end{align}

    Using that $0< \lambda<\frac{\tau\opt}{2} < \tau$ and $0 <F(\cdot) \le 1$ we can upper bound $ \frac{F\left((\tau-\lambda)^2\right)-F(0)}{(\tau-\lambda)^2}$ as follows
    \begin{equation}\label{eq:FLB}
        \frac{F\left((\tau-\lambda)^2\right)-F(0)}{(\tau-\lambda)^2} \le \frac{1}{(\frac{\tau\opt}{2}-\lambda)^2}.
    \end{equation}

    As $F(\cdot)$ is increasing, $F(1) =1$ and $F(\tau^2) \leq F((1+\tau\opt)^2/4)$ since by assumption $\tau < \frac{1+\tau\opt}{2}$, we also have that
    \begin{equation}\label{eq:FUB}
          \frac{F(1) - F\left(\tau^2\right)}{1-\tau^2} \ge 1 -F\left(\frac{(1+\tau\opt)^2}{4}\right) .
    \end{equation}

    In order to derive an upper and lower bound for the middle term of \Cref{eq:MVTbounds}, it remains to upper and lower bound its denominator; $\tau^2 -(\tau-\lambda)^2 = 2\lambda\tau -\lambda^2$ as
    \begin{align}\label{eq:tauDiff}
       \lambda \frac{\tau\opt}{2} \overset{(i)}{\le} \tau^2 -(\tau-\lambda)^2 \overset{(ii)}{\le} 2\lambda.
    \end{align}
    For the lower bound (i) of \Cref{eq:tauDiff} we used the inequalities
    \begin{align*}
        2\lambda\tau - \lambda^2 = \lambda (2\tau -\lambda)
        \ge \lambda \tau
        \ge \lambda \frac{\tau\opt}{2},
    \end{align*}
    where the inequalities hold because $\lambda < \frac{\tau\opt}{2}  < \tau$.
    For the upper bound (ii) in \Cref{eq:tauDiff} we used that $\tau<1$. 

    By replacing \Cref{eq:FLB,eq:FUB,eq:tauDiff} into \Cref{eq:MVTbounds} we have that;

    \begin{align*}
        \frac{\tau\opt}{2}\left( 1 -F\left(\frac{(1+\tau\opt)^2}{4}\right) \right) \lambda&\le F\left(\tau^2\right) - F\left((\tau-\lambda)^2\right) \le  2\frac{1}{(\frac{\tau\opt}{2}-\lambda)^2}\lambda.
    \end{align*}

    Defining the functions $\useconstant{CL}(\tau\opt),\useconstant{CR}(\tau\opt,\lambda)$ as
    \begin{align*}
        \useconstant{CL}(\tau\opt) &:= \frac{\tau\opt}{2}\left( 1 -F\left(\frac{(1+\tau\opt)^2}{4}\right) \right),\\
        \useconstant{CR}(\tau\opt,\lambda) &:= 2\frac{1}{(\frac{\tau\opt}{2}-\lambda)^2}.
    \end{align*}
    we obtain the desired result. 
\end{proof}

With these results in place, we are able to upper and lower bound the volume in this spherical segment.

\begin{lemma}\label{lem:spherical_seg_prob_bound}
    Under \Cref{assum:density}, for all $0< \lambda < \frac{\tau\opt}{2} <\tau < \frac{1+\tau\opt}{2}<1$, we have that
    \begin{equation*}
        m \cdot V_d(1) \cdot \useconstant{CL}(\tau\opt) \cdot \lambda \le \P\left(\tau -\lambda < |X^\top \theta\opt|\le \tau \right) \le M \cdot V_d(1) \cdot \useconstant{CR}(\tau\opt,\lambda) \cdot \lambda.
    \end{equation*}
\end{lemma}

\begin{proof}
We first use that
    \begin{equation}\label{eq:probTauMinusLToL}
        \P\left(\tau -\lambda < |X^\top \theta\opt|\le \tau \right) = \int_{\norm{\Bx} \le 1}\left( \indctr{\abs{\dotp{\Bx}{\theta\opt}} \le \tau} - \indctr{\abs{\dotp{\Bx}{\theta\opt}} \le \tau -\lambda } \right)p(\Bx) d\Bx
    \end{equation}

    We can use the smoothness property of our distribution to sandwich \Cref{eq:probTauMinusLToL} as 

    \begin{align*}
        &m \int_{\norm{\Bx} \le 1}\left( \indctr{\abs{\dotp{\Bx}{\theta\opt}} \le \tau} - \indctr{\abs{\dotp{\Bx}{\theta\opt}} \le \tau-\lambda} \right) d\Bx \\
        &\le  \P\left(\tau < |X^\top \theta\opt|\le \tau + \lambda\right) \\
        &\le M \int_{\norm{\Bx} \le 1}\left( \indctr{\abs{\dotp{\Bx}{\theta\opt}} \le \tau} - \indctr{\abs{\dotp{\Bx}{\theta\opt}} \le \tau-\lambda} \right) d\Bx.
    \end{align*}

    Now, let $Z\sim Beta(\frac{1}{2},\frac{d+1}{2})$ and $F(\cdot)$ its CDF function, then, \Cref{lem:spherical_segment} allows us to write the integral as

    \begin{equation*}
        \int_{\norm{\Bx} \le 1}\left( \indctr{\abs{\dotp{\Bx}{\theta\opt}} \le \tau} - \indctr{\abs{\dotp{\Bx}{\theta\opt}} \le \tau-\lambda} \right) d\Bx = V_d(1) \left( F\left((\tau)^2\right) - F\left((\tau-\lambda)^2\right) \right),
    \end{equation*}

    and the previous equation becomes

    \begin{align*}
        &m V_d(1) \left( F\left((\tau)^2\right) - F\left((\tau-\lambda)^2\right) \right) \\
        &\le  \P\left(\tau < |X^\top \theta\opt|\le \tau + \lambda\right) \\
        &\le M V_d(1) \left( F\left((\tau)^2\right) - F\left((\tau-\lambda)^2\right) \right).
    \end{align*}

    Finally, we apply \Cref{lem:betaCDF_stability} to lower and upper bound $V_d(1) \left( F\left((\tau)^2\right) - F\left((\tau-\lambda)^2\right) \right)$ and conclude the proof.

    \begin{align*}
        m\cdot V_d(1) \cdot \useconstant{CL}(\tau\opt) \cdot \lambda \le \P\left(\tau-\lambda < |X^\top \theta\opt|\le \tau \right) \le M \cdot V_d(1) \cdot \useconstant{CR}(\tau\opt,\lambda) \cdot \lambda.
    \end{align*}
\end{proof}

Before proving \Cref{lem:tauopt_stability_alpha}, we first prove an auxiliary lemma to derive a range of $\gamma$ for which we can apply \Cref{lem:spherical_seg_prob_bound}, i.e. $\frac{\tau\opt}{2} <\tau < \frac{1+\tau\opt}{2}$.

\begin{lemma}\label{lem:safeTauRange}
     For any $0 < \gamma < \frac{m \cdot V_d(1) \cdot f(\frac{1+\tau\opt}{2})\cdot (1-\tau\opt)}{2(1+e)}$ it holds that\footnote{$f(\cdot)$ is the PDF of the random variable $Z \sim Beta(\alpha, \beta)$.} 
     \begin{equation*}
         \min \left\{\tau \in \left[\frac{\tau\opt}{2},1\right] : \perr(\theta\opt,P,\tau) \le \alpha-\gamma \right\} 
         = \min \left\{\tau \in \left[\frac{\tau\opt}{2},\frac{1+\tau\opt}{2}\right] : \perr(\theta\opt,P,\tau) \le \alpha-\gamma \right\}.
     \end{equation*}
\end{lemma}

\begin{proof}
    To prove this, we show that for these values of $\gamma$ there exists a $\tau(\gamma)\in [\tau\opt,\frac{1+\tau\opt}{2}] \subset [\frac{\tau\opt}{2},\frac{1+\tau\opt}{2}]$ such that $\perr\left(\theta\opt, P,  \tau(\gamma) \right) = \alpha -\gamma$. Thus, 
    \begin{align*}
        \gamma &\overset{(a)}{=} \perr(\theta\opt,P,\tau\opt) - \perr\left(\theta\opt, P,  \tau(\gamma)\right)\\
        &\overset{(b)}{\le} \perr(\theta\opt,P,\tau\opt) - \perr\left(\theta\opt,P,\frac{1+\tau\opt}{2}\right).
    \end{align*}
    (a) uses that $\perr(\theta\opt,P,\tau)$ is strictly decreasing and continuous with respect to its third argument $\tau$ (see the proof of \Cref{lem:tauLessOne}), thus $\perr(\theta\opt,P,\tau\opt) = \alpha$, and in (b) the monotonicity of $\perr(\theta\opt,P,\tau)$.
    It now remains to find a lower bound for $$\perr(\theta\opt,P,\tau\opt) - \perr\left(\theta\opt,P,\frac{1+\tau\opt}{2}\right).$$
    
    We remind the reader that by definition of $\perr(\cdot)$
    \begin{equation*} 
    \perr(\theta,P,\tau) = \int (1+\exp(|x^\top \theta|))^{-1} \mathds{1}\left\{|x^\top \theta| > \tau\right\} P(dx).
    \end{equation*}

    \begin{tikzpicture}
    \begin{axis}[
        axis lines = middle,
        xlabel = {$\tau$},
        ylabel = {$\perr(\theta\opt,P,\tau)$},
        xlabel style = {at={(ticklabel* cs:1.05)}, anchor=west},
        ylabel style = {at={(ticklabel* cs:1.05)}, anchor=south},
        xmin = -4, xmax = 8.2,
        ymin = -0.2, ymax = 1.02,
        width = 10cm,
        height = 8cm,
        xticklabels = {},
        yticklabels = {},
    ]
    
    \addplot[
        domain=0.2:6,
        samples=100,
        color=red,
        thick,
    ] {exp(-x/2)};
    
    \addplot[only marks, mark=*, mark size=2pt, color=black] coordinates {(2,0.36787944117144233)};
    \addplot[only marks, mark=*, mark size=2pt, color=black] coordinates {(4,0.1353352832366127)};
    
    \node[above right] at (axis cs:2,0.36787944117144233) {$({\tau\opt},\alpha)$};
    \node[above right] at (axis cs:4,0.1353352832366127) {$\left(\frac{1+{\tau\opt}}{2},\perr(\frac{1+\tau\opt}{2})\right)$};

    \draw[thin, black, dashed] (axis cs:2,0.368) -- (axis cs:2,0);
    \draw[thin, black, dashed] (axis cs:2,0.368) -- (axis cs:-0.5,0.368);
    \draw[thin, black, dashed] (axis cs:4,0.135) -- (axis cs:4,0);
    \draw[thin, black, dashed] (axis cs:4,0.135) -- (axis cs:-0.5,0.135);

    \node[below] at (axis cs:2,-0.05) {$\tau\opt$};
    \node[below] at (axis cs:4,-0.05) {$\frac{1+\tau\opt}{2}$};
    \node[left] at (axis cs:-0.5,0.368) {$\alpha = \perr(\tau\opt)$};
    \node[left] at (axis cs:-0.5,0.25) {$\gamma \in $};
    \node[left] at (axis cs:-0.5,0.135) {$\perr(\frac{1+\tau\opt}{2})$};

    \draw[<->, thick, blue] (axis cs:-0.5,0.135) -- (axis cs:-0.5,0.368);
    
    \end{axis}
    \end{tikzpicture}

    \begin{align*}
        \perr(\theta\opt,P,\tau\opt)-\perr(\theta\opt,P,(1+\tau\opt)/2) &= \int (1+\exp(|x^\top \theta|))^{-1} \left( \mathds{1}\left\{|x^\top \theta| > \tau\opt \right\} - \mathds{1}\left\{|x^\top \theta| > \frac{(1+\tau\opt)}{2} \right\}\right) P(dx)\\
        &\overset{(a)}{\ge} \frac{1}{1+e} \int \mathds{1}\left\{ \tau\opt \le |x^\top \theta| \le \frac{(1+\tau\opt)}{2} \right\} P(dx)\\
        &\overset{(b)}{\ge} \frac{m\cdot V_d(1)}{1+e} \int \mathds{1}\left\{ \tau\opt \le |x^\top \theta| \le \frac{(1+\tau\opt)}{2} \right\} \frac{1}{V_d(1)}dx\\
        &\overset{(c)}{=} \frac{m\cdot V_d(1)}{1+e} \left( F\left( \frac{(1+\tau\opt)}{2} \right) - F(\tau\opt) \right).
    \end{align*}
    
    where (a) comes from $|x^\top \theta| \le 1$, (b) from \Cref{assum:density} and (c) from \Cref{lem:spherical_segment} (recall that $F(\cdot)$ is the CDF of the random variable $Z \sim \text{Beta}(\alpha, \beta)$).
    To derive a lower bound for $F\left( \frac{(1+\tau\opt)}{2} \right) - F(\tau\opt)$ we will use the Mean Value Theorem as in \Cref{lem:betaCDF_stability} applied in $[\tau\opt, \frac{1+\tau\opt}{2}]$ for $F(\cdot)$.
    Then, there exists a $\xi \in (\tau\opt, \frac{1+\tau\opt}{2})$ such that 

    \begin{align*}
        \frac{F\left( \frac{(1+\tau\opt)}{2} \right) - F(\tau\opt)}{\frac{(1+\tau\opt)}{2} -\tau\opt} &= f(\xi) \ge f\left(\frac{1+\tau\opt}{2}\right),
    \end{align*}

    where $f(\cdot)$ is a decreasing function as we proved in \Cref{lem:betaCDF_stability}.

    Combining the above we get 

    \begin{align*}
        \perr(\theta\opt,P,\tau\opt)-\perr(\theta\opt,P,(1+\tau\opt)/2) &\ge \frac{m \cdot V_d(1) \cdot f(\frac{1+\tau\opt}{2})\cdot (1-\tau\opt)}{2(1+e)}.
    \end{align*}

    As a consequence for all $\gamma \in [0, \frac{m \cdot V_d(1) \cdot f(\frac{1+\tau\opt}{2})\cdot (1-\tau\opt)}{2(1+e)}]$ we know that 
    $$\perr(\theta\opt,P,(1+\tau\opt)/2) \le \alpha -\gamma,$$
    and 
    \begin{equation*}
         \min \left\{\tau \in \left[\frac{\tau\opt}{2},1\right] : \perr(\theta\opt,P,\tau) \le \alpha-\gamma \right\} 
         = \min \left\{\tau \in \left[\frac{\tau\opt}{2},\frac{1+\tau\opt}{2}\right] : \perr(\theta\opt,P,\tau) \le \alpha-\gamma \right\}.
     \end{equation*}
\end{proof}

\subsubsection{Proof of \Cref{lem:tauopt_stability_alpha}}\label{proof:tauopt_stability_alpha}

\lemTauoptStabilityAlpha*

\begin{proof}
    For arbitrary $\tau < 1$, we begin by studying the difference between $\perr$ evaluated at thresholds $\tau-\lambda$ and $\tau$.
    By applying \Cref{lem:spherical_seg_prob_bound}, for all $0< \lambda < \frac{\tau\opt}{2} <\tau < \frac{1+\tau\opt}{2}<1$ it is true that;

    \begin{align*}
        \perr(\theta\opt,P,\tau-\lambda) - \perr(\theta\opt,P,\tau)
        &= \int \left( 1+ \exp(\abs{\dotp{x}{\theta\opt}}) \right)^{-1} \indctr{\tau -\lambda < \abs{\dotp{x}{\theta\opt}} < \tau} P(dx)\\
        &\ge \int \frac{1}{1+e} \indctr{\tau -\lambda < \abs{\dotp{x}{\theta\opt}} < \tau} P(dx)\\
        &= \frac{1}{1+e} \P(\tau-\lambda < \abs{\dotp{x}{\theta\opt}} < \tau) \\
        &\ge  \frac{m}{1+e} \cdot V_d(1) \cdot \useconstant{CL}(\tau\opt) \cdot \lambda
        , \label{eq:perr_diff_bound} \numberthis
    \end{align*}

    \begin{align*}
        \tau\opt(\theta\opt,P,\alpha-\gamma) &= \min \{\tau \in [0,1] : \perr(\theta\opt,P,\tau) \le \alpha-\gamma \}\\
        &\overset{(a)}{\le} \min \left\{\tau \in \left[\frac{\tau\opt}{2},\frac{1+\tau\opt}{2}\right] : \perr(\theta\opt,P,\tau) \le \alpha-\gamma \right\}\\
        &\overset{(b)}{\le} \min \left\{\tau \in \left[\frac{\tau\opt}{2},\frac{1+\tau\opt}{2}\right]: \perr(\theta\opt,P,\tau - \lambda) \le \alpha -\gamma + \frac{m}{1+e} \cdot V_d(1) \cdot \useconstant{CL}(\tau\opt) \cdot  \lambda \right\} \\
        &\overset{(c)}{\le} \min \left\{\tau \in \left[\frac{\tau\opt}{2},\frac{1+\tau\opt}{2}\right]: \perr\left(\theta\opt,P,\tau - \frac{(1+e)\gamma}{ m \cdot V_d(1)\useconstant{CL}(\tau\opt)}\right) \le \alpha\right\} \\
        &\overset{(d)}{\le} \min \left\{\tau \in \left[\frac{\tau\opt}{2},\frac{1+\tau\opt}{2}\right]: \perr(\theta\opt,P,\tau) \le \alpha \right\} + \frac{(1+e)\gamma}{ m \cdot V_d(1)\useconstant{CL}(\tau\opt)}\\
        &= \tau\opt(\theta\opt,P,\alpha) + \frac{(1+e)\gamma}{ m \cdot V_d(1)\useconstant{CL}(\tau\opt)}.\\
    \end{align*}

    In (a) we used \Cref{lem:safeTauRange}, in (b) we leveraged the $\perr$ difference bound derived in \Cref{eq:perr_diff_bound},
    (c) follows from setting $\lambda = \frac{(1+e)\gamma}{ m \cdot V_d(1)\useconstant{CL}(\tau\opt)}$, and (d) from \Cref{lem:perrShift} by setting $x:= \frac{(1+e)\gamma}{ m \cdot V_d(1)}$.
    We observe that for $0<\gamma < \min \{\frac{\tau\opt\cdot m \cdot V_d(1) \cdot \useconstant{CL}(\tau\opt)}{2(1+e)},\frac{m \cdot V_d(1) \cdot f(\frac{1+\tau\opt}{2}) \cdot (1-\tau\opt)}{2(1+e)} \}$, we satisfy the condition of \Cref{lem:spherical_seg_prob_bound}. 
\end{proof}

\section{Safety analysis}

We begin by providing a sketch of the results proved in this section.
First, in \Cref{sec:tau_stability_safety} we prove \Cref{lem:stabilityInThetaL2}, which is an analogue of \Cref{lem:perr_stability_tau} but with $\ell_2$ error, to show that shifting from $\theta$ to $\theta_Q \in \CQ_\theta$ doesn't change $\tau$ much.
Then, we have the following two safety lemmas, which compare \SCOUT's performance with the optimal testing policy for confidence $\alpha_t$, i.e. $Z\opt_t = \mathds{1}\{\abs{\dotp{X_t}{\theta\opt}} \le \tau\opt(\theta\opt,P,\alpha_t\}$.
\begin{lemma}\label{lem:Z_t_pessimistic}
    The testing rule $Z_t$ defined in \Cref{alg:alg1} satisfies, conditioned on $G_{\perr}$ and $G_\theta$, ~$Z\opt_t = 1 \implies Z_t = 1$, i.e. $Z_t \ge Z\opt_t$ a.s.
\end{lemma}

This follows by the monotonicity of the threshold $\tau\opt$ with respect to $\alpha$ and by using a ``safer'' error tolerance $\alpha_t$ than $\alpha$.
We defer the proof to \Cref{proof:Z_t_pessimistic}. 

Another property of our testing rule is that when $G_{\perr}$ holds it makes no more errors than the baseline policy.
As formalized in the following lemma, \SCOUT's predictions are identical to those of the oracle policy when it does not test, ensuring its $(\alpha,\delta)$-safety.

\begin{lemma}\label{lem:prediction_imitation}
Let $\hat{Y}_t$ be the prediction of our policy, where $Y\opt_t$ is the prediction of the oracle baseline policy.
When $G_{\perr}$ and $G_\theta$ holds, and  $Z_t = 0$ (which implies that $Z\opt_t = 0$) then $\hat{Y}_t = \hat{Y}\opt_t$. 

\end{lemma}

To show the previous lemma, we use the fact that, on the good event, when we do not test, all the inner products $\dotp{X_t}{\theta\opt}$ have the same sign.
We defer the proof to \Cref{proof:prediction_imitation}.

More formally, we define the Bernoulli random variable $\xi_t = \indctr{\hat{Y}_t \neq Y_t}$, that denotes whether the algorithm made a mistake at round $t$, and $\xi\opt_t = \indctr{\hat{Y}\opt_t \neq Y_t}$  respectively for the baseline policy.
When the algorithm tests (i.e. $Z_t = 1$) then we observe the label and it holds that $\xi_t = 0$.
Conditioning on the good event $G$, the random variables $\xi_t$ and $\xi\opt_t$ satisfy $\xi_t \le \xi\opt_t$ (formalized in \Cref{proof:policy_is_feasible}).
This implies a total error probability bound, stated in the following lemma.

\subsection{$\tau$ stability lemma} \label{sec:tau_stability_safety}

The safety analysis requires the application of \Cref{lem:tau_q_upperNlower_bound} for $\theta := \theta^L_t$.
However, it is not guaranteed that $\theta^L_t \in \CQ_\theta$.
To surpass this technical detail, we use the stability of $\perr$ in $\theta$, similar to \Cref{lem:perr_stability_tau}, but expressing the result in the $\ell_2$ distance, the metric with respect to which the covering is defined.

\begin{lemma}\label{lem:stabilityInThetaL2}
    For all $\theta,\theta' \in \Theta$, $ \tau \ge \|\theta - \theta'\|_{2}$, and density $\rho(x)$ on $\CX$:
    \begin{align*}
          \perr(\theta,\rho,\tau) \le \ \perr \left(\theta',\rho,\tau- \|\theta - \theta'\|_{2}\right) + \|\theta - \theta'\|_{2}.
    \end{align*}
\end{lemma}

\begin{proof}
    Here, we use $x$ as a dummy variable for integration:
    \begin{align*}
        \perr(\theta,\rho,\tau) &= \int (1+\exp(|x^\top \theta|))^{-1} \mathds{1}\left\{|x^\top \theta| > \tau\right\} \rho(dx) \\
        &= \int (1+\exp(|x^\top \theta' + x^\top (\theta - \theta')|))^{-1} \mathds{1}\left\{|x^\top \theta' + x^\top (\theta - \theta')| > \tau\right\} \rho(dx) \\
        &\le \int (1+\exp(|x^\top \theta' | -  |x^\top (\theta - \theta')|))^{-1} \mathds{1}\left\{|x^\top \theta'| > \tau - |x^\top (\theta - \theta')| \right\} \rho(dx) \\
        &\le \int \left((1+\exp(|x^\top \theta' |))^{-1} + |x^\top (\theta - \theta')|\right) \mathds{1}\left\{|x^\top \theta'| > \tau - |x^\top (\theta - \theta')| \right\}\rho(dx) \\
        &\le  \max_{x' \in \CX} \int \left((1+\exp(|x^\top \theta' |))^{-1} + |x'^\top (\theta - \theta')|\right) \mathds{1}\left\{|x^\top \theta'| > \tau - |x'^\top (\theta - \theta')| \right\}\rho(dx) \\
        &= \max_{x' \in \CX}\perr(\theta',\rho,\tau - |x'^\top (\theta - \theta')|) + \int |x^\top (\theta - \theta')| \mathds{1}\left\{|x^\top \theta'| > \tau - |x'^\top (\theta - \theta')| \right\} \rho(dx) \\
        &\le \max_{x' \in \CX} \perr(\theta',\rho,\tau- \|\theta - \theta'\|_{2}\| x' \|_{2}) + \| \theta - \theta' \|_{2} \| x' \|_{2} \P_{\rho}\left(|x^\top \theta'| > \tau - |x^\top (\theta - \theta')| \right) \\
        &\le \max_{x' \in \CX} \perr(\theta',\rho,\tau- \| \theta - \theta' \|_{2} \| x' \|_{2}) + \| \theta - \theta' \|_{2} \| x' \|_{2} \\
        &= \perr(\theta',\rho,\tau- \| \theta - \theta' \|_{2}) + \| \theta - \theta' \|_{2}
    \end{align*}
    The details of this proof are identical to those of \Cref{lem:perr_stability_tau}.
    We also make use that our contexts lie in the unit ball, i.e. $\norm{x}_2 \le 1$.

\end{proof}

Using the previous lemma we derive a similar expression to that of \Cref{lem:tau_q_upperNlower_bound}; 

\begin{align}\label{eq:tauStabilityL2}
    \tau\opt_Q(\theta_Q,\hat{P}_t,\alpha) &\ge \tau\opt_Q\left(\theta,\hat{P}_t,\alpha +\| \theta - \theta_Q \|_{2}\right) - \| \theta - \theta_Q \|_{2}. 
\end{align}

\subsection{Proof of \Cref{lem:Z_t_pessimistic}}\label{proof:Z_t_pessimistic}

\begin{proof}

Let $\tilde{\theta}^L_t \in Q_\theta$ such that $\norm{\tilde{\theta}^L_t  - \theta^L_t}_2 \le \eps_Q$, as $\theta^L_t$ lies in the interior of $\CC_t$.

Leveraging \Cref{lem:stabilityInThetaL2}, we relate $\theta^L_t$ to $\tilde{\theta}^L_t$ as (using the definition of \Cref{eq:tau_t_defn}), on the good events $G_{\perr}$ and $G_\theta$:

\begin{align*}
    \tau_t &= \tau\opt\left(\theta^L_t,\hat{P}_t,\alpha_t -\zeta_t -2B_t/\sqrt{ \lambda_{\min}^t} - \eps_Q\right) + 3B_t/\sqrt{ \lambda_{\min}^t} + \eps_Q\\
    &= \hat{\tau}\left(\theta^L_t,\hat{P}_t,\alpha_t - \eps_Q\right) + B_t/\sqrt{ \lambda_{\min}^t} + \eps_Q\\
    &\ge \hat{\tau}\left(\tilde{\theta}^L_t,\hat{P}_t,\alpha_t\right) + B_t/\sqrt{ \lambda_{\min}^t} \\
    &\ge \tau\opt\left(\theta\opt,P,\alpha_t\right) + B_t/\sqrt{ \lambda_{\min}^t} \\
    \numberthis
\end{align*}
Here, we used the monotonicity of $\tau\opt$ with respect to $\alpha$, in addition to \Cref{lem:tau_q_upperNlower_bound}.
Then, we upper bound the inner product:
\begin{align*}
    \abs{\dotp{X_t}{\theta^L_t}}
    \le \abs{\dotp{X_t}{\theta\opt}} + \|\theta^L_t - \theta\opt\|_{V_t} \|X_t\|_{V_t^{-1}}
    \le \abs{\dotp{X_t}{\theta\opt}} + B_t / \sqrt{\lambda_{\min}^t}
\end{align*}
By Holder.
Combining these together yields that, on $G_{\perr}$ and $G_\theta$,
\begin{equation}
    |\langle X_t, \theta\opt \rangle | \le \tau\opt\left(\theta\opt,P,\alpha_t\right) \quad \implies \quad \abs{\dotp{X_t}{\theta^L_t}} \le \tau_t .
\end{equation}
i.e. $Z\opt_t = 1 \implies Z_t=1$
\end{proof}

\subsection{Proof of \Cref{lem:prediction_imitation}}\label{proof:prediction_imitation}

\begin{proof}    
On $G_{\perr}$ and $G_\theta$, we have that $Z_t = 0$ implies that $\langle \theta, X_t \rangle$ has the same sign for all $\theta \in \CC_t$.
This is because, $Z_t=0$ only when:
\begin{equation*}
|\langle \theta^L_t , X_t \rangle | \ge \tau_t = \tau\opt\left(\theta^L_t,\hat{P}_t,\alpha_t - \zeta_t -2B_t / \sqrt{\lambda_{\min}^t} - \eps_Q\right) + 3B_t / \sqrt{\lambda_{\min}^t}+\eps_Q.
\end{equation*}
As before, we know that
\begin{align*}
    \tau_t &\ge \tau\opt\left(\theta\opt,P,\alpha_t\right) + B_t/\sqrt{\lambda_{\min}^t}
\end{align*}

We also have that for all $\theta \in \CC_t$:
\begin{align*}
    |\langle \theta^L_t , X_t \rangle - \langle \theta , X_t \rangle| \le B_t/\sqrt{\lambda_{\min}^t}.
\end{align*}
Thus, if $|\langle \theta^L_t , X_t \rangle| \ge \tau_t$, and assuming without loss of generality that $\langle \theta^L_t , X_t \rangle>0$, then for all $\theta \in \CC_t$:
\begin{align*}
    0 &\le \langle \theta^L_t , X_t \rangle - \tau_t\\
    &\le \left(\langle \theta , X_t \rangle +B_t/\sqrt{\lambda_{\min}^t}\right) - \left(\tau\opt\left(\theta\opt,P,\alpha_t\right) + B_t/\sqrt{\lambda_{\min}^t} \right)\\
    &= \langle \theta , X_t \rangle - \tau\opt\left(\theta\opt,P,\alpha\right) \numberthis
\end{align*}
i.e. $\langle \theta , X_t \rangle \ge \tau\opt\left(\theta\opt,P,\alpha_t\right) > \tau\opt\left(\theta\opt,P,\alpha\right) > 0$ for all $\theta \in \CC_t$ on $G_{\perr}$ and $G_\theta$ (as $\alpha_t < \alpha$).
\end{proof}

\subsection{$(\alpha,\delta)$ safety (proof of \Cref{lem:policy_is_feasible})}\label{proof:policy_is_feasible}

To prove this lemma, we define the Bernoulli random variable $\xi_t = \indctr{\hat{Y}_t \neq Y_t}$, that denotes whether the algorithm made a mistake at round $t$, and $\xi\opt_t = \indctr{\hat{Y}\opt_t \neq Y_t}$ 
respectively for the baseline policy.
When the algorithm tests (i.e. $Z_t = 1$) then we observe the label and it holds that $\xi_t = 0$.
Conditioning on the good event $G$, we show that the random variables $\xi_t$ and $\xi\opt_t$ satisfy $\xi_t \le \xi\opt_t$.
This implies a total error probability bound.

\policyIsFeasible*

\begin{proof}
    We analyze the four possible outcomes of the binary random variables $(Z\opt_t,Z_t)$, under the good events $G_\theta$ and $G_{\perr}$.
    Recall that $\xi_t$ is whether our algorithm makes a mistake at time $t$, and $\xi\opt_t$ is whether the optimal baseline which tests at threshold $\tau\opt$ makes an error at time $t$.
    
    \textbf{Case 1: $(Z\opt_t,Z_t) = (1,1)$.} 
    In this case, both our policy and the oracle baseline observe the true label and $\xi_t = \xi\opt_t=0$, i.e. neither method makes an error.

    \textbf{Case 2: $(Z\opt_t,Z_t) = (1,0)$.} Under the good event $G$, by \Cref{lem:Z_t_pessimistic} this cannot occur.

    \textbf{Case 3: $(Z\opt_t,Z_t) = (0,1)$.}
    When, $Z\opt_t = 0$ and $Z_t = 1$, our policy tests and observes the true label while the optimal baseline predicts $\hat{Y}\opt_t$, in which case $0 = \xi_t \le \xi\opt_t$ a.s.

    \textbf{Case 4: $(Z\opt_t,Z_t) = (0,0)$.}
    When, $Z\opt_t = 0$ and $Z_t = 0$, from \Cref{lem:prediction_imitation} it holds that $\hat{Y}_t = \hat{Y}\opt_t$ a.s., and so $\xi_t = \xi\opt_t$ a.s. 
    
    Combining these 4 cases together, we have shown that $\xi_t \le \xi\opt_t$ a.s.
    Now, $ \xi\opt_t$ are independent binary random variables with $\E( \xi\opt_t) \le \alpha_t$, since the sequence $\alpha_t$ is decreasing.
    Then at any time $\bar{T} \le T$:
    \begin{align*}
    \P\left( \frac{1}{\bar{T}} \sum_{t=1}^{\bar{T}} \xi_t \ge \alpha \ \middle \vert \ G\right)
        & \le \P\left( \frac{1}{\bar{T}} \sum_{t=1}^{\bar{T}} \xi\opt_t \ge \alpha \ \middle \vert \ G \right)\\
        &\le \P\left( \frac{1}{\bar{T}} \sum_{t=1}^{\bar{T}} (\xi\opt_t - \E \xi\opt_t) \ge \alpha - \alpha_{\bar{T}} \ \middle \vert \ G \right)\\
        &\le \exp(-2\bar{T}(\alpha - \alpha_{\bar{T}})^2).
    \end{align*}
    Recall that
        \begin{equation*}
        \alpha_t = \alpha - \sqrt{\frac{\log(2t^2/\delta')}{2t}},
    \end{equation*}
    Thus:
    \begin{align*}
        \P\left( \bigcup_{\bar{T}=1}^T \left\{\frac{1}{\bar{T}} \sum_{t=1}^{\bar{T}} \xi_t \ge \alpha \right\} \ \middle \vert \ G\right)
        &\le \sum_{\bar{T}=1}^T \P\left( \frac{1}{\bar{T}} \sum_{t=1}^{\bar{T}} \xi_t \ge \alpha \ \middle \vert \ G\right)\\
        &\le \sum_{\bar{T}=1}^T \exp(-2\bar{T}(\alpha - \alpha_{\bar{T}})^2)\\
        &\le \sum_{t=1}^T \frac{\delta'}{2t^2}\\
        &\le \delta' \numberthis
    \end{align*}
     
\end{proof}

\section{Regret analysis} \label{sec:regret_analysis}

We begin by bounding the instantaneous regret at time $t>T_0$.

\begin{lemma} \label{lem:regret_bound}
    For every round $t > T_0$, conditioned on the good event $G$, the regret is bounded as:
    \begin{equation*}
        \E[Z_t-Z\opt_t | G] \le 
        M \cdot V_d(1) \cdot \useconstant{CR}(\tau\opt) \left( \frac{12\left( \zeta_t + 8 B_t/\sqrt{ p\opt t \lambda_0} \right)}{m\cdot V_d(1) \cdot \useconstant{CL}(\tau\opt)} +2\eps_Q + 28B_t/\sqrt{ p\opt t \lambda_0} \right). 
    \end{equation*}
\end{lemma}

\begin{proof}[Proof of \Cref{lem:regret_bound}.]
    For $t \le T_0$ we can bound each term of the regret by 1, i.e. $\E[Z_t - Z] \leq 1$.
    For $t > T_0$ this requires analyzing $\E[Z_t - Z]$, essentially upper bounding how often we test in excess of the optimal baseline.
    We test whenever $c_t = \abs{\dotp{X_t}{\theta_t^L}} -  \tau_t \le 0$.
    Thus, we need to lower bound $c_t$ to show that we do not perform too many excess tests.

    \begin{align*}
        c_t &= \abs{\dotp{X_t}{\theta_t^L}} -  \tau_t \\
        &= \abs{\dotp{X_t}{\theta_t^L}} - \tau\opt\left(\theta_t^L,\hat{P}_t,\alpha_t - \zeta_t - 2B_t/\sqrt{ \lambda_{\min}^t} -\eps_Q\right) - 3B_t/\sqrt{ \lambda_{\min}^t} -\eps_Q\\
        &\overset{(a)}{\ge} \abs{\dotp{X_t}{\theta_t^L}} - \tau\opt_Q\left(\theta_Q,\hat{P}_t,\alpha_t - 2\zeta_t - 4B_t/\sqrt{ \lambda_{\min}^t}\right) - 5B_t/\sqrt{ \lambda_{\min}^t} -\eps_Q\\
        &\overset{(b)}{\ge} \abs{\dotp{X_t}{\theta\opt}} -\tau\opt\left(\theta\opt,P,\alpha - 3\zeta_t - 6B_t/\sqrt{ \lambda_{\min}^t}\right) - 7B_t/\sqrt{ \lambda_{\min}^t}-2\eps_Q\\
        &\overset{(c)}{\ge} \abs{\dotp{X_t}{\theta\opt}} -\tau\opt\left(\theta\opt,P,\alpha\right)- \frac{3(1+e)\left( \zeta_t + 2 B_t/\sqrt{ \lambda_{\min}^t} \right)}{m\cdot V_d(1) \cdot \useconstant{CL}(\tau\opt)} -2\eps_Q - 7B_t/\sqrt{ \lambda_{\min}^t}\\
    \end{align*}
    a) comes from \Cref{lem:tau_q_upperNlower_bound,lem:stabilityInThetaL2} to analyze a quantized version of $\theta_t^L$. Concretely, we utilize $\theta_Q$ as the projection of $\theta_t^L$ onto $\CC_t \cap \Theta_Q$.
    (b) applies \Cref{lem:tau_q_upperNlower_bound} in the reverse direction, to get $\tau\opt$ evaluated at $\theta\opt$.
    We also use the fact that $\alpha_t \ge \alpha - \zeta_t$. Additionally, $\abs{\dotp{X_t}{\theta_t^L}} \ge \abs{\dotp{X_t}{\theta\opt}} - B_t/\sqrt{ \lambda_{\min}^t}$ on $G_{\perr}, G_\theta$.
    Then, in (c), we apply \Cref{lem:tauopt_stability_alpha}, where the condition is met for sufficiently large $T_0$ under $G$.

    \begin{align*}
        \E R_t &=\E[Z_t-Z | G]\\
        &= \P\left(\left\{c_t \le 0\right\} \cap\left\{|\langle X_t, \theta\opt \rangle| \ge \tau\opt\right\}| G\right)\\
        &\overset{a}{\le} \P\left(\tau\opt \le |\langle X_t, \theta\opt \rangle| \le \tau\opt +  \frac{3(1+e)\left( \zeta_t + 2 B_t/\sqrt{ \lambda_{\min}^t} \right)}{m\cdot V_d(1) \cdot \useconstant{CL}(\tau\opt)} +2\eps_Q + 7B_t/\sqrt{ \lambda_{\min}^t} \quad \Big| G\right)\\
        &\overset{b}{\le}  M \cdot V_d(1) \cdot \useconstant{CR}(\tau\opt) \left( \frac{12\left( \zeta_t + 2 B_t/\sqrt{ \lambda_{\min}^t} \right)}{m\cdot V_d(1) \cdot \useconstant{CL}(\tau\opt)} +2\eps_Q + 7B_t/\sqrt{ \lambda_{\min}^t} \right)\\
        &\overset{c}{\le} M \cdot V_d(1) \cdot \useconstant{CR}(\tau\opt) \left( \frac{12\left( \zeta_t + 8 B_t/\sqrt{ p\opt t \lambda_0} \right)}{m\cdot V_d(1) \cdot \useconstant{CL}(\tau\opt)} +2\eps_Q + 28B_t/\sqrt{ p\opt t \lambda_0} \right) \label{eq:regret_bound}\\
    \end{align*}
        a) follows by the upper bounding of the thresholding condition, and b) follows from \Cref{lem:spherical_seg_prob_bound}, and c) from $G$ that $\lambda_{\min}^t \ge p\opt t \lambda_0 / 12$.

    An important technical detail in applying \Cref{lem:spherical_seg_prob_bound} is that the upper and lower bounds of our spherical segment are sufficiently close to $\tau\opt$.
    When we apply this lemma, the perturbation is a constant multiple of $\zeta_t+ B_t/\sqrt{ p\opt t \lambda_0}$ which are of order $\CO(1/\sqrt{t})$ under $G$.
    Thus, for sufficiently large constant $T_0$, for all $t \ge T_0$, we are able to apply \Cref{lem:spherical_seg_prob_bound}.

\end{proof}

With this instantaneous regret, we are now able to sum across all time steps to compute our total regret.
We are then also able to prove the $(\alpha,\delta)$ safety of \SCOUT.

\regretUpperBound*

\begin{proof}[Proof of \Cref{thm:regret_upper_bound}.]

    We first show that \SCOUT satisfies $(\alpha,\delta)$ safety.
    Define $A$ as the event where \SCOUT is $(\alpha,\delta)$-safe.
    \begin{align*}
        \P(\bar{A}) &= \P(\bar{A} | G) \P(G) + \P(\bar{A} | \bar G) \P(\bar G)\\
        &\le \P(\bar{A} | G) +  \P(\bar G)\\
        & \le \delta' + 6\delta'\\
        &= \delta
    \end{align*}
    Here we used the law of total probability, and leveraged from \Cref{lem:gt_goodevent} that the good event happens with probability at least $1-6\delta'$, and from \Cref{lem:policy_is_feasible} that conditioned on $G$, \SCOUT is $(\alpha,\delta')$-safe.
    In the last line we plugged in that $\delta' = \delta/7$.

    Analyzing the number of excess tests, we use \Cref{lem:regret_bound} and condition on $G$, to find that with probability at least $1-\delta$:
    \begin{align*}
        \texttt{Regret}(T) &\leq T_0 + \sum_{t=T_0}^{T}\E R_t \\
        &= T_0 + 12\frac{M}{m}\frac{\useconstant{CR}(\tau\opt)}{\useconstant{CL}(\tau\opt)} \sum_{t=T_0}^{T}\left( \zeta_t + 8 B_t/\sqrt{ p\opt t \lambda_0} \right)\\
        &+ 2 M \cdot V_d(1) \cdot \useconstant{CR}(\tau\opt)\sum_{t=T_0}^{T}\eps_Q + 28 M \cdot V_d(1) \cdot \useconstant{CR}(\tau\opt)\sum_{t=T_0}^{T}B_t/\sqrt{ p\opt t \lambda_0} 
    \end{align*}

    Both $\eps_Q = 1/t^2$ and the $\zeta_t$ (\Cref{eq:zeta_t}) terms are dominated by the term: $\sum_{t=T_0}^{T}B_t/\sqrt{ p\opt t \lambda_0} $.
    Finally, for $B_t$ (from \Cref{eq:conf_ellipse_joint}), we can use from $G$ that we get enough samples, i.e. $N_\theta^t$ grows linearly in $t$.
    \begin{align*}
    B_t &= 2\kappa\left(1 + \sqrt{\log\left(\frac{1}{\delta}\right) + 2d\log\left(1+\frac{N_\theta^t}{\kappa d}\right)}\right)\\
    &\le 13\sqrt{ 2d\log\left(N_\theta^t/\delta\right)}\\
    \sum_{t=T_0}^T B_t (p\opt t \lambda_0 / 12)^{-1/2}
    &\le B_T \sum_{t=T_0}^T (p\opt t \lambda_0 / 12)^{-1/2}\\
    &\le 13\sqrt{ 2d\log\left(T/\delta\right)} \sum_{t=T_0}^T (p\opt t \lambda_0 / 12)^{-1/2}\\
    &\le 52\sqrt{ \frac{dT\log\left(T/\delta\right)}{p\opt \lambda_0}}
    \end{align*}

    Combining this all together we have that:
    \begin{align*}
        \texttt{Regret}(T) &\leq T_0 + \sum_{t=T_0}^{T}\E R_t \\
        &= T_0 + 12\frac{M}{m}\frac{\useconstant{CR}(\tau\opt)}{\useconstant{CL}(\tau\opt)} \sum_{t=T_0}^{T}\left( \zeta_t + 8 B_T/\sqrt{ p\opt t \lambda_0} \right)\\
        &+ M \cdot V_d(1) \cdot \useconstant{CR}(\tau\opt)\sum_{t=T_0}^{T}\frac{1}{t^2} + 28 M \cdot V_d(1) \cdot \useconstant{CR}(\tau\opt)B_T\sum_{t=T_0}^{T}1/\sqrt{ p\opt t \lambda_0}\\
        &\preceq T_0 + 4992 \frac{M}{m}\frac{\useconstant{CR}(\tau\opt)}{\useconstant{CL}(\tau\opt)}\sqrt{ \frac{dT\log\left(T/\delta\right)}{p\opt \lambda_0}} \numberthis
    \end{align*}

    We can further bound the regret by using the lower bound for $\lambda_0$ from \Cref{lem:min_eigenvalue}, 
    $$\lambda_0 \ge \frac{m (\tau\opt)^3 V_d(1)}{p\opt (d+2)}.$$
    Using that,we derive the following asymptotic lower bound 
    \begin{equation*}
        \texttt{Regret}(T) = \CO \left(d\sqrt{\frac{T\log\left(T/\delta\right)}{(\tau\opt)^{d+2}}}\right) 
    \end{equation*}
    We note that our dependence in the number of dimensions is of order $\tilde{\CO}(d\sqrt{T})$, same as in linear and logistic bandits (see \citet{lattimore2020bandit}).
    Then, we observe that the edge cases when $\tau\opt = 0$, that is equivalent to $p\opt=0$ characterize the problem's difficulty. 
    As we have already mentioned in the main text, for $\tau\opt \rightarrow 0$ implies that $p\opt = 0$, and we cannot collect enough samples to form our estimators.
\end{proof}

\section{Good event proof}\label{app:good_event_proof}

\subsection{Theta estimation set gets enough samples}
\begin{lemma} \label{lem:ntheta_t_event}
    On $G_\theta$ and $G_{\perr}$, $N_\theta^t \ge p\opt t/2  - \sqrt{\frac{\ln(\pi t^2/(3\delta'))}{2}}$ with probability at least $1-\delta'$.
\end{lemma}

\begin{proof}[Proof of \Cref{lem:ntheta_t_event}]
In \Cref{lem:Z_t_pessimistic} we proved that, with high probability, our policy tests whenever the optimal one does, when $G_\theta$ and $G_{\perr}$ hold.
This implies that $N_\theta^t \ge N_{OPT}^t$.

As we show, just considering the even time steps, the optimal baseline policy will collect at least $N_{OPT}^t \ge p\opt t/2  - \sqrt{\frac{\ln(\pi t^2/(3\delta'))}{2}}$ samples with high probability up to time $t$.
Using $Z\opt_t$ as whether the optimal thresholding rule would test at time $t$, we have that, on $G_{\perr}$ and $G_\theta$,
\begin{align*}
    N_{OPT}^t \ge \sum_{t=1}^{T//2} Z\opt_{2t}.
\end{align*}
This implies that:
\begin{align*}
    \P\left(N_{OPT}^T \le p\opt \lfloor T /2 \rfloor - \nu_T \right)
    &\le \P\left(\sum_{t=1}^{T//2} \left( Z\opt_{2t} - p\opt\right) \le - \nu_T \right)\\
    &\le \exp \left( -2 \nu_T^2 / \lfloor T/2 \rfloor  \right)\\
    &\le \frac{\delta' \pi^2}{6 t^2}
\end{align*}
by careful construction of $\nu_T$.

Since $\delta'$ is a constant (we simply require that $\delta' = \Omega(T^2 e^{-T})$), then, for some $T_0$, we have that for all $t\ge T_0$ with probability at least $1-\delta'$;
\begin{equation}\label{eq:Nopt_simplification}
    N_{\theta}^t \ge N_{OPT}^t \ge p\opt t / 3.
\end{equation}
\end{proof}

To show that $\P(G_{\lambda}) \geq 1 -\delta$ we will use a covering argument to derive a lower bound for the minimum covariance matrix.
Then, we will use \Cref{lem:ntheta_t_event} as a lower bound on the number of samples collected to construct the empirical covariance matrix.
Finally, we will union bound these two events to complete the proof.

\subsection{$\lambda_{\min}^t$ grows linearly in $t$}
\begin{lemma}\label{lemma:min_eigenv_lb}
Let $\delta \in (0,1)$. Consider a random $d\times d$ dimensional matrix valued process $\{A_t\}_{t=0}^{\infty}$ adapted to a filtration $\mathcal{F}_t = \sigma(A_k \mid k \leq t)$, where each $A_t \in \mathbb{R}^{d \times d}$ is symmetric ($A_t = A_t^\top$), positive semi-definite, satisfies $\|A_t\|_{\text{op}} \leq 1$ almost surely and such that there is a constant $\lambda_{0} > 0$ satisfying
\begin{equation*}
\P\left(    \lambda_{\mathrm{min}}( \mathbb{E}[ A_t | \mathcal{F}_{t-1} ] ) \geq \lambda_{0}~\forall t \in \mathbb{N} \right) \geq 1-\tilde{\delta}.
\end{equation*}
 Let $\lambda_{\min}^t := \lambda_{\min}\left(\sum_{s=0}^t A_s\right)$. Then, for $\eps > 0$, the following holds:
\[
\P\left\{
\lambda_{\min}^t \geq t(\lambda_{0} - 2\eps) - \sqrt{\frac{t}{2}\left(d\log\left(\frac{2}{\eps} + 1\right) + \log\left(\frac{4t^2}{\delta'}\right)\right)}~
\forall t \in \mathbb{N}  \right \} 
\geq 1 - \delta'.
\]
\end{lemma}

\begin{proof}[Proof of \Cref{lemma:min_eigenv_lb}]
    Let the random variable $Z_t^\upsilon := \upsilon^\top A_t \upsilon - \E[ \upsilon^\top A_t \upsilon \mid \mathcal{F}_{t-1}]$, such that $\upsilon \in \mathcal{S}^{d-1}$.
    Notice that $Z_t^\upsilon$ is a martingale difference sequence as;
    \begin{enumerate}
        \item 
        \begin{align*}
            \E[\abs{Z_t^\upsilon}] &\leq \E[\abs{\upsilon^\top A_t \upsilon}] + \E\abs{\E[ \upsilon^\top A_t \upsilon \mid \mathcal{F}_{t-1}]} \\
            &\leq \E[\upsilon^\top A_t \upsilon] + \E\E[ \upsilon^\top A_t \upsilon \mid \mathcal{F}_{t-1}] \\
            &\leq 1 + 1 = 2 < \infty.
        \end{align*}
        \item 
        \begin{align*}
            \E[Z_t^\upsilon \mid \mathcal{F}_{t-1}] &= \E[ \upsilon^\top A_t \upsilon \mid \mathcal{F}_{t-1}] - \E[ \upsilon^\top A_t \upsilon \mid \mathcal{F}_{t-1}] = 0.
        \end{align*}
    \end{enumerate}

    By the Azuma-Hoeﬀding Inequality \citep{chung2006concentration}, as $Z_t^\upsilon \in [0,1]$ a.s., for a fixed $t \in [T]$ we have, $c\geq0$;
    \begin{align*}
        \P\left\{ \sum_{s=0}^{t} (\upsilon^\top A_s\upsilon - \E[\upsilon^\top A_s\upsilon \mid \mathcal{F}_{s-1}]) \leq -c \right\} \leq \exp\left(-\frac{2c^2}{t}\right).
        \end{align*}
Setting the error probability to $\delta_t$,
        \begin{align*}
        \P\left\{ \sum_{s=0}^{t} (\upsilon^\top A_s\upsilon - \E[\upsilon^\top A_s\upsilon \mid \mathcal{F}_{s-1}]) \leq -\sqrt{\frac{\log(\frac{1}{\delta_t})t}{2}} \right\} \leq \delta_t.
    \end{align*}
    Thus, substituting $\delta_t = \frac{\tilde{\delta}}{2t^2}$ and using the union bound we get,
    \begin{equation*}
         \P\left\{ \sum_{s=0}^{t} (\upsilon^\top A_s\upsilon - \E[\upsilon^\top A_s\upsilon \mid \mathcal{F}_{s-1}]) \leq -\sqrt{\frac{\log(\frac{2t^2}{\tilde{\delta}})t}{2}} ~~\forall t \in \mathbb{N}\right\} \leq \sum_{t=1}^\infty \delta_t \leq \tilde{\delta}.
    \end{equation*}

    Let $\mathcal{N}(\mathcal{S}^{d-1},\eps)$ an $\eps$-cover of $\mathcal{S}^{d-1}$.
    By \textbf{Corollary 4.2.13} at \citet{vershynin2018high} we have that the covering numbers of $\mathcal{S}^{d-1}$ satisfy for any $\eps>0$;
    \begin{align*}
        \left(\frac{1}{\eps}\right)^d &\leq \mathcal{N}(\mathcal{S}^{d-1},\eps) \leq \left(\frac{2}{\eps}+1\right)^d.
    \end{align*}

        For convenience, we define ~$\nu(t,\tilde{\delta}) := \sqrt{\frac{[d\log({2/\eps + 1})+\log(\frac{2t^2}{\tilde{\delta}})]t}{2}}$.
    By taking the union bound over all $\upsilon_i \in \mathcal{N}(\mathcal{S}^{d-1},\eps)$ we have
    \begin{equation}\label{equation::support_cover_bound}
        \P\left\{\exists \upsilon_i \in  \mathcal{N}(\mathcal{S}^{d-1},\eps): \sum_{s=0}^{t} (\upsilon_i^\top A_s\upsilon_i - \E[\upsilon_i^\top A_s\upsilon_i \mid \mathcal{F}_{s-1}]) \leq -\nu(t,\tilde{\delta})~~ \forall t \in \mathbb{N}\right\} \leq \tilde{\delta}  \\   
    \end{equation}
    
    Let $\upsilon_t^\star := \argmin_{\upsilon \in \mathcal{S}^{d-1}}\upsilon^\top \sum_{s=0}^{t}A_s\upsilon$, 
    then there exists an $\upsilon_{i_t} \in \mathcal{N}(\mathcal{S}^{d-1},\eps)$ such that
    $\norm{\upsilon_{i_t} - \upsilon_t^\star}_2 \leq \eps$.
    We are going to bound $\abs{{\upsilon_t^{\star}}^\top \sum_{s=0}^{t}A_s \upsilon_t^\star - \upsilon_{i_t}^\top\sum_{s=0}^{t}A_s \upsilon_{i_t}}$ by a function of $\eps$.
    \begin{align}
        \abs{{\upsilon_t^{\star}}^\top \sum_{s=0}^{t}A_s \upsilon_t^\star - \upsilon_{i_t}^\top\sum_{s=0}^{t}A_s \upsilon_{i_t}} &= \abs{{\upsilon_t^{\star}}^\top \sum_{s=0}^{t}A_s \upsilon_t^\star 
        - {\upsilon_t^{\star}}^\top \sum_{s=0}^{t}A_s\upsilon_{i_t}
        + {\upsilon_t^{\star}}^\top \sum_{s=0}^{t}A_s\upsilon_{i_t}
        -\upsilon_{i_t}^\top\sum_{s=0}^{t}A_s \upsilon_{i_t}} \notag \\
        &= \abs{{\upsilon_t^{\star}}^\top \sum_{s=0}^{t}A_s (\upsilon_t^\star-\upsilon_{i_t}) + (\upsilon_t^\star - \upsilon_{i_t})^\top \sum_{s=0}^{t}A_s \upsilon_{i_t}
        } \notag\\
        &= \abs{(\upsilon_t^\star - \upsilon_{i_t})^\top \sum_{s=0}^{t}A_s (\upsilon_{i_t}+\upsilon_t^{\star})}\notag\\
        &\leq \norm{\upsilon_t^\star - \upsilon_{i_t}}_2 \norm{\sum_{s=0}^{t}A_s (\upsilon_{i_t}+\upsilon_t^{\star})}_2 \notag \\
        &\leq \eps \sum_{s=0}^{t}\norm{A_s}_{op}(\norm{\upsilon_{i_t}}_2 + \norm{\upsilon_t^{\star}}_2)\notag \\
        &= 2t\eps. \label{equation::cover_bound_optimal}
    \end{align}
    Using inequality~\ref{equation::support_cover_bound} we have
    \begin{align*}
         \P \left\{ \sum_{s=0}^{t} \upsilon_{i_t}^\top A_s\upsilon_{i_t} \geq \sum_{s=0}^{t}\E[\upsilon_{i_t}^\top A_s\upsilon_{i_t} \mid \mathcal{F}_{s-1}] - \nu(t,\tilde{\delta})~~\forall t \in \mathbb{N}\right\} \geq 1-\tilde{\delta}.
\end{align*}
            where $i_t$ is a point in the cover $\mathcal{N}(\mathcal{S}^{d-1},\eps)$ such that $\norm{\upsilon_{i_t} - \upsilon_t^\star}_2 \leq \eps$. Equation~\ref{equation::cover_bound_optimal} can be used to relate $\sum_{s=0}^{t} \upsilon_{i_t}^\top A_s\upsilon_{i_t}$ and $\lambda_{\mathrm{min}}^t$,
        \begin{align*}
         \P \left\{ \underbrace{\sum_{s=0}^{t} {\upsilon_t^{\star}}^\top A_s{\upsilon_t^{\star}}}_{\lambda_{\mathrm{min}}^t} + 2t\eps \geq \sum_{s=0}^{t}\E[\upsilon_{i_t}^\top A_s\upsilon_{i_t} \mid \mathcal{F}_{s-1}] - \nu(t,\tilde{\delta})~~\forall t \in \mathbb{N}\right\} \geq 1-\tilde{\delta}.
    \end{align*}
Using the fact that $\E[\upsilon_{i_t}^\top A_s\upsilon_{i_t} \mid \mathcal{F}_{s-1}] \geq \lambda_{\mathrm{min}}( \E[ A_s \mid \mathcal{F}_{s-1}] ) $ we conclude that,
   \begin{align*}
         \P \left\{ \lambda_{\mathrm{min}}^t  + 2t\eps \geq \sum_{s=0}^{t}  \lambda_{\mathrm{min}}( \E[ A_s \mid \mathcal{F}_{s-1}] ) - \nu(t,\tilde{\delta})~~\forall t \in \mathbb{N}\right\} \geq 1-\tilde{\delta}.
    \end{align*}
Finally, the assumption that  $\P\left( \lambda_{\mathrm{min}}( \mathbb{E}[ A_t | \mathcal{F}_{t-1} ] ) \geq \lambda_{0}~\forall t \in \mathbb{N} \right) \geq 1-\tilde{\delta}$ and a union bound allows us to conclude that,
\begin{align*}
  &\P \left\{ \lambda_{\min}^t \geq t(\lambda_{0}-2\eps) - \nu(t,\tilde{\delta})~~\forall t \in \mathbb{N}\right\} \\
 &\geq  \P \left\{ \lambda_{\mathrm{min}}^t  + 2t\eps \geq \sum_{s=0}^{t}  \lambda_{\mathrm{min}}( \E[ A_s \mid \mathcal{F}_{s-1}] ) - \nu(t,\tilde{\delta}) \cap \lambda_{\mathrm{min}}( \mathbb{E}[ A_t | \mathcal{F}_{t-1} ] ) \geq \lambda_{0}~~\forall t \in \mathbb{N}\right\} \\
    &\geq  1-2\tilde{\delta}.
\end{align*}
This finalizes the result for $\delta' = 2\tilde{\delta}$.

\end{proof}

We will apply this lemma for $A_t = X_tX_t^\top$.
We use the fact that $\lambda_{\min}\left( \kappa\boldsymbol{I}_d + \sum_{s \in \CS_\Theta^t} X_sX_s^\top\right) > \lambda_{\min}\left(\sum_{s \in \CS_\Theta^t} X_sX_s^\top\right)$.
It is true that $\norm{ X_tX_t^\top}_{op} = \norm{X_t}_2 \le 1$.
We will make again the same observation, by choosing the covering parameter as $\eps = \frac{\lambda_{0}}{5}$, then we have that for all $t\ge T_0$
\begin{equation}\label{eq:lambda_min_simplification}
   \lambda_{\min}^t \geq N_\theta^t \cdot \frac{\lambda_{0}}{4}.
\end{equation}

In \Cref{lem:ntheta_t_event} we proved that with probability at least $1-\delta'$, it holds that $N_\theta^t \ge \frac{p\opt t}{3}$.
By taking the union bound over the two events, we have that with probability at least $1-2\delta'$

\begin{equation*}
    \lambda_{\min}^t \ge p\opt t \cdot \frac{\lambda_0}{12}. 
\end{equation*}

\subsection{Combining all together}
\begin{proof}[Proof of \Cref{lem:total_good_event}]
   By using the product rule we have that 
   \begin{align*}
       \P( G_\theta \cap G_N \cap G_{\perr}) &= \P(G_N \mid G_\theta \cap G_{\perr})\P(G_\theta \cap G_{\perr})\\
   \end{align*}
    As $\P(G_\theta) \geq 1 -\delta$ from \Cref{lem:faury_concentration} and $\P(G_{\perr}) \geq 1 -\delta$ from \Cref{lem:gt_goodevent}, by using the union bound we have $\P(G_\theta \cap G_{\perr}) \geq 1-2\delta$.
    By using also \Cref{lem:ntheta_t_event} we have 
    \begin{align*}
        \P(G_N \mid G_\theta \cap G_{\perr})\P(G_\theta \cap G_{\perr}) &\ge (1-2\delta')^2 \\
        &\ge 1 -4\delta'.
   \end{align*}
   As $\P(G_\lambda) \ge 1 -2\delta'$ by \Cref{lemma:min_eigenv_lb}, by taking the union bound again we have that 
   \begin{equation*}
       \P(G_\theta \cap G_{\perr} \cap G_N \cap G_\lambda) \geq 1- 6\delta'.
   \end{equation*}
\end{proof}

\section{Modifications from written algorithm} \label{app:mod_from_written}
For our numerical simulations, we implemented a version of \SCOUT with a few minor modifications from \Cref{alg:alg1} to enable it to a) run computationally faster, and b) converge more quickly statistically.
These changes are common in practical applications of online learning algorithms to balance theoretical rigor with empirical performance.
Full details are available online: \url{https://github.com/TavorB/SCOUT}.

\textbf{Batched Parameter Updates:} as written, \SCOUT updates the parameter estimate and the testing threshold at every time step $t$.
    In a setting with a large time horizon $T$, re-running the estimation procedures on ever-growing datasets at each step is computationally wasteful, as these will not change too much iteration to iteration.
    Instead, our implementation updates these estimates only periodically.
    Concretely, the estimates for $\theta$ and $\tau$ are cached and reused for a block of subsequent time steps.
    The frequency of these updates is decreased as the simulation progresses, reflecting the gradual convergence of the parameters.

    \vspace{.1cm}
    \textbf{Simplified Testing Condition:} the testing condition of \SCOUT is given by $\langle X_{t},\theta_{t}^{L}\rangle| \le \tau_t $.
    This incorporates several uncertainty terms derived from our theoretical analysis.
    While crucial for the regret bounds, computing these quantities at every step is not necessary in practice, and the same performance can be obtained by simply collapsing these terms into a) the $\tau$ estimate, and b) a bound on $B_t \|X_t\|_{V_t^{-1}}$ (note that in practice this second term may not be known, as it will depend on $\lambda_0$, which \SCOUT will learn and adapt to).
    The testing decision becomes $Z_t = 1$ if $|\langle X_t,  \theta^L_t \rangle|$ is less than the sum of these two terms.

    \vspace{.1cm}
\textbf{Omission of the Projection Step:} Our theoretical analysis utilizes two estimators.
    First, the regularized maximum likelihood estimator $\hat{\theta}_{t} = \argmax_{\theta \in \mathbb{R}^d} \mathcal{L}_{t}(\theta)$, where $\mathcal{L}_t(\theta)$ is the regularized log-likelihood.
    Second, for analysis purposes, a projection of this estimator, $\theta_t^L$, is defined in \Cref{eq:projection}.
    This projection is in practice unneeded, and so we simply utilize $\hat{\theta}_t$ as our $\theta$ estimate.

In addition, we empirically reduce the leading constants for confidence related quantities, for example in the $B_t$ bound.
These were modified once up front, retaining the empirical safety guarantees of our algorithm, and dramatically reducing time to convergence.

\end{document}